\documentclass{article}

\usepackage{colm2024_arxiv}
\usepackage[para]{threeparttable}
\usepackage{markdown}
\usepackage{pythonhighlight}
\usepackage{microtype}
\usepackage[hyphens]{xurl}
\usepackage[breaklinks=true]{hyperref}
\usepackage{booktabs}
\definecolor{darkblue}{rgb}{0, 0, 0.5}
\hypersetup{colorlinks=true, citecolor=darkblue, linkcolor=darkblue, urlcolor=darkblue}

\usepackage{multirow}
\usepackage{CJKutf8}
\usepackage{tabularx}
\usepackage{listings}
\usepackage{xcolor} 
\usepackage{times}
\usepackage{amssymb}
\usepackage{authblk}
\usepackage[T1]{fontenc}
\usepackage{graphicx}
\usepackage[many]{tcolorbox}
\usepackage{bm}
\usepackage{CJK}
\usepackage{arabtex}
\usepackage{utf8}
\setcode{utf8}
\usepackage{fancyvrb}
\usepackage{blindtext}
\usepackage{multicol}
\usepackage{tabularray}
\usepackage{makecell}

\usepackage{subcaption}
\usepackage{relsize}
\usepackage{amsthm}
\newtheorem{thm}{Theorem}
\newtheorem{lemma}{Lemma}
\newtheorem{definition}{Definition}
\usepackage{tablefootnote}

\definecolor{verylightgray}{gray}{0.95} 

\usepackage{fourier} 
\usepackage{array}
\usepackage[export]{adjustbox}
\usepackage{import}
\usepackage{float}
\usepackage{amsmath}
\usepackage{siunitx}
\usepackage{framed}
\usepackage{enumitem}
\usepackage{caption}
\usepackage{footnote}

\DeclareUnicodeCharacter{03A3}{\ensuremath{\Sigma}}

\usepackage{comment}
\usepackage{pifont}
\usepackage{tocloft}

\newcommand{\xmark}{\ding{55}}%

\setlist[2]{noitemsep} 
\setitemize{noitemsep} 
\setenumerate{noitemsep} 

\makeatletter
\renewcommand\AB@affilsepx{, \protect\Affilfont}
\makeatother
\title{RWKV-7 "Goose" with Expressive Dynamic State Evolution}

\author[1,2,\thanks{Equal first authorship. Others listed alphabetically.}] {\textbf{Bo Peng}}
\author[3,*] {\textbf{Ruichong Zhang}} 
\author[2,4,*] {\textbf{Daniel Goldstein}}

\author[2,5] {\authorcr \textbf{Eric Alcaide}}
\author[6] {\textbf{Xingjian Du}}
\author[7] {\textbf{Haowen Hou}}
\author[8] {\textbf{Jiaju Lin}}
\author[9] {\textbf{Jiaxing Liu}}
\author[4,10] {\textbf{Janna Lu}}
\author[11] {\textbf{William Merrill}}
\author[2,12] {\textbf{Guangyu Song}}
\author[13] {\textbf{Kaifeng Tan}}
\author[2] {\textbf{Saiteja Utpala}}
\author[2,4] {\textbf{Nathan Wilce}}
\author[14] {\textbf{Johan S. Wind}}
\author[15] {\textbf{Tianyi Wu}}
\author[2,16] {\textbf{Daniel Wuttke}}
\author[2] {\textbf{Christian Zhou-Zheng}}

\affil[1]{RWKV Project (under Linux Foundation AI \& Data)}
\affil[2]{EleutherAI}
\affil[3]{Tsinghua University}
\affil[4]{Recursal AI}
\affil[5]{Dalle Molle Institute for Artificial Intelligence USI-SUPSI}
\affil[6]{University of Rochester}
\affil[7]{Guangdong Laboratory of Artificial Intelligence and Digital Economy (SZ)}
\affil[8]{Pennsylvania State University}
\affil[9]{Zhejiang University}
\affil[10]{George Mason University}
\affil[11]{New York University}
\affil[12]{Tano Labs}
\affil[13]{Shenzhen University}
\affil[14]{University of Oslo}
\affil[15]{Beijing Normal University}
\affil[16]{Denigma}

\begin{document}

\begin{center}
    \maketitle
\end{center}

\begin{abstract}
We present RWKV-7 "Goose", a new sequence modeling architecture with constant memory usage and constant inference time per token. Despite being trained on dramatically fewer tokens than other top models, our 2.9 billion parameter language model achieves a new 3B SoTA on multilingual tasks and matches the current 3B SoTA on English language downstream performance. RWKV-7 introduces a newly generalized formulation of the delta rule with vector-valued gating and in-context learning rates, as well as a relaxed value replacement rule. We show that RWKV-7 can perform state tracking and recognize all regular languages, while retaining parallelizability of training. This exceeds the capabilities of Transformers under standard complexity conjectures, which are limited to $\mathsf{TC}^0$. To demonstrate RWKV-7's language modeling capability, we also present an extended open source 3.1 trillion token multilingual corpus, and train four RWKV-7 models ranging from 0.19 billion to 2.9 billion parameters on this dataset.

To foster openness, reproduction, and adoption, we release our models\footnote{Model weights at \url{https://huggingface.co/RWKV}} and dataset component listing\footnote{Dataset components listed at \url{https://huggingface.co/RWKV}} on Hugging Face, and our training and inference code\footnote{Source code at: \url{https://github.com/RWKV/RWKV-LM}} on GitHub; all under the Apache 2.0 License.
\end{abstract}
 
\pagebreak
\tableofcontents
\clearpage
 
\section{Introduction} \label{sec:introduction}
\begin{table}
\centering
    \begin{adjustbox}{max width=1.0\linewidth}
\begin{threeparttable}[htb]
    \centering
        \begin{tabular}{lllcccc} 
 \toprule
\bf Name & \bf State Evolution \hfill \bf Scalars & \bf LS & \bf FD & \bf DD & \bf GE \\
\midrule
RWKV-4 & $\bm{s}_t = e^{-w} \odot \bm{s}_{t-1} + e^{k_t} \odot v_t;$  & \xmark & \checkmark & \xmark & \xmark \\
& $\bm{s'}_t = e^{-w} \odot \bm{s}'_{t-1} + e^{k_t}$\\
RetNet & $\bm{S}_t = w \bm{S}_{t-1} + v_t^T k_t $ \hfill $w$  & \checkmark & \xmark & \xmark & \xmark \\
RWKV-5 & $\bm{S}_t = \bm{S}_{t-1} \mathrm{diag} (w) + v_t^T k_t $ & \checkmark & \checkmark & \xmark & \xmark \\
Mamba & $\bm{S}_t = \bm{S}_{t-1} \odot \exp(-(w_t^T 1) \odot \exp (\bm{A})) + (w_t \odot v_t)^T k_t $ & \checkmark & \checkmark & \checkmark & \xmark \\
RWKV-6 \& GLA & $\bm{S}_t = \bm{S}_{t-1} \mathrm{diag} (w_t) + v_t^T k_t $  & \checkmark & \checkmark & \checkmark & \xmark \\
HGRN-2 & $\bm{S}_t = \bm{S}_{t-1} \mathrm{diag} (w_t) + v_t^T (1-w_t) $  & \checkmark & \checkmark & \checkmark & \xmark \\
Mamba-2 & $\bm{S}_t = w_t \bm{S}_{t-1}  + v_t^T k_t $ \hfill $w_t$ & \checkmark & \xmark & \checkmark & \xmark \\
TTT $^a$ & $\bm{S}_t = \bm{S}_{t-1} - a_t \nabla l(\bm{S}_{t-1}, k_t, v_t)$ \hfill $a$ & \checkmark & \xmark & \xmark & \checkmark \\
Longhorn & $\bm{S}_t = \bm{S}_{t-1} \odot (\bm{I} - a_t^T k_t^2) + (a_t \bm{x}_t)^T k_t $  & \checkmark & \checkmark & \checkmark & \xmark \\
Gated DeltaNet & $\bm{S}_t = w_t \bm{S}_{t-1} (\bm{I} - a_t k_t^T k_t)  + a_t v_t^T k_t $
\quad \hfill $w_t, a_t$& \checkmark & \xmark & \checkmark & \checkmark \\
Titans $^a$ & $\bm{M}_t = (1 - \alpha_t) \bm{M}_{t-1} + \bm{S}_t$ \hfill $w_t, a_t$ & \checkmark & \xmark & \checkmark & \checkmark \\
 & $\bm{S}_t = w_t\bm{S}_{t-1} - a_t \nabla l(\bm{M}_{t-1}, k_t, v_t)$ & & & \\
\midrule
Generalized $\Delta$ Rule & $\bm{S}_t = \bm{S}_{t-1} (\mathrm{diag}(w_t) + z_t^T b_t) + v_t^T k_t$ & 
\textbf{\checkmark} & \textbf{\checkmark} & \textbf{\checkmark} & \textbf{\checkmark} \\
\textbf{RWKV-7 (ours)} & $\bm{S}_t = \bm{S}_{t-1} (\mathrm{diag}(w_t) - \hat{\kappa}_t^T (a_t \odot \hat{\kappa}_t))  + v_t^T k_t $  & \textbf{\checkmark} & \textbf{\checkmark} & \textbf{\checkmark} & \textbf{\checkmark} \\
\bottomrule\\
\end{tabular}
\vspace{-1em}
\caption{
Recent RNN architectures used for language modeling. 
\\ 
\footnotesize{
\textbf{LS} (Large State): matrix-valued states, or state size at least 4 times larger than the model dimension. \\ 
\textbf{FD} (Flexible Decay): the dimension of the decay term $w$ or $w_t$ is not smaller than the model dimension. \\
\textbf{DD} (Dynamic Dependence): the decay term $w_t$ is a function over the input $x_t$. \\
\textbf{GE} (Generalized Eigenvalue): evolution matrix admits eigenvalues outside of the interval $[0,1]$.\\
$^a$ Shown with mini batch size 1 for simplicity.
}
} \label{tab:rnn_state_evolution}
\end{threeparttable}
\end{adjustbox}
\end{table}

Autoregressive Transformers \citep{vaswani2023attention} have recently dominated sequence modeling tasks, enjoying excellent in-context processing and highly parallelizable training due to their use of softmax attention. However, softmax attention incurs quadratic computational complexity and memory usage with respect to sequence length due to its linearly expanding key-value cache. For short sequences, much of this cost can be covered by modern GPU parallelism techniques, but Transformer inference becomes increasingly costly as sequence lengths grow.

This limitation has inspired significant research into the design of recurrent neural network (RNN) architectures with compressive states that afford linear computational complexity and constant memory usage, while still allowing highly parallel training. Two of the most commonly proposed alternatives that satisfy these requirements are linear attention variant models \citep{linearTrans2020inputDepend, sun2023retentive, rwkv6_colm, yang2023gated} and State Space Models \citep{gu2023mamba}. These architectures have grown more sophisticated, with many recent proposals incorporating some form of the delta rule, as embodied by parallelized DeltaNet \citep{schlag_deltarule, yang2024_deltanet}. Such models have achieved impressive downstream performance results: since RWKV-4 \citep{peng2023rwkv}, RNN models have shown increasing potential to rival Transformers when given equivalent model size and training compute, while dramatically reducing inference costs.

We present a new architecture, RWKV-7 "Goose", which generalizes the delta rule for use in sequence modeling. First, we add a vector-valued state gating mechanism, enhancing expressivity and providing implicit positional encoding. Second, we expand the in-context learning rate from a scalar to become vector-valued, allowing the model to selectively replace state data on a channel-wise basis. Third, we decouple the keys at which the delta rule removes from and adds to the state. Finally, we place these innovations within a modified RWKV-6 architecture, inheriting important features such as token-shift, bonus, and a $\mathrm{ReLU}^2$ feedforward network. We also introduce an expanded 3.1 trillion token RWKV World v3 corpus designed for enhanced English, code, and multilingual task performance. We use this architecture and corpus to train new state-of-the-art open-source language models, upgraded from preexisting RWKV-5/RWKV-6 checkpoints.

\pagebreak
Our main contributions are as follows:

\begin{itemize}
    \item The \textbf{RWKV-7 "Goose" architecture}, which dramatically improves downstream benchmark performance over RWKV-6 and demonstrates state-of-the-art multilingual performance at 3B scale and near SoTA English language performance, despite being trained on many fewer tokens than the top models in its class.
    
    \item The \textbf{RWKV World v3 public dataset}, comprised of 3.1 trillion tokens of publicly available multilingual data.
    
    \item Public release of four \textbf{pre-trained RWKV-7 World v3 language models}, ranging from 0.19 to 2.9 billion parameters trained on 1.6 to 5.6 trillion tokens.

    \item Public release of three \textbf{pre-trained RWKV-7 Pile language models}, using the GPT-NeoX tokenizer \citep{black-etal-2022-gpt}, ranging from 0.17 to 1.47 billion parameters, useful for comparative study with other architectures.    

    \item \textbf{Proofs} that the generalized delta rule employed in RWKV-7 can \textbf{solve problems outside of $\mathsf{TC}^0$} under the widely held complexity conjecture that $\mathsf{TC}^0 \ne \mathsf{NC}^1$. This includes solving an $S_5$ state tracking problem known to be in $\mathsf{NC}^1$ using only a single layer, and recognizing all regular languages using only a constant number of layers.

    \item A \textbf{method for upgrading the RWKV architecture without pre-training from scratch}, producing increasingly competitive trained models at reduced computational expense.

\end{itemize}

Larger datasets and RWKV-7 models are under active preparation and construction and will be released under the Apache 2 license whenever practical.

\section{Background} \label{sec:background}

Linear attention's major advantage over softmax attention is that it can be formulated as a RNN with constant running time per token and constant memory usage \citep{katharopoulos2020lineartransformers}, while softmax attention takes $O(N)$ time per token and $O(N)$ memory with regard to sequence length. Despite this dramatic efficiency improvement, linear attention has its own significant drawbacks \citep{schlag_deltarule, han2024bridgingdividereconsideringsoftmax, fan2025breakinglowrankdilemmalinear}.

One such issue is that linear attention numerically adds to the fixed-size state at every time-step: older state contents are never removed, only reduced by becoming a smaller proportion of the numerically increasing state. Due to limitations on the state size, eventually such a system must mix values together and muddy the outputs retrieved for a given key \citep{schlag_deltarule, yang2024gatedlinearattentiontransformers}. Modern linear attention architectures like RWKV-6 \citep{rwkv6_colm}, RetNet \citep{sun2023retentive}, Gated Linear Attention \citep{yang2023gated}, and Mamba 2 \citep{dao2024transformersssmsgeneralizedmodels} use per time-step decay to remove some portion of such older values from the state in a data-dependent manner. However, decay is a blunt tool that cannot remove only the values stored at specific keys. 

\paragraph{Delta Rule.} DeltaNet \citep{schlag_deltarule} sidesteps the problem of numerically increasing state by partially replacing the value stored at the current key with the same amount of a new value, allowing the model to both take away old memories and add new ones on a per-key basis. It reformulates the state update as an explicit online learning problem where the goal is to retrieve the correct value as output for a given key as input. DeltaNet was the first to apply the foundational Error Correcting Delta Rule \citep{widrow_delta_rule} to key-value compressive states, akin to those stored in the RNN formulation of linear attention. This update rule is equivalent to a single step of stochastic gradient descent, training the state $\bm{S_t}$ at test time to output the desired values $v_t$ for the keys $k_t$ as inputs using loss $\mathcal{L}=\frac{1}{2}\|(\bm{S}_tk_t-v_t)\|^2$ and gradient $\frac{\partial \mathcal{L}}{\partial \bm{S}} = \bm{S} k^\top k - v^\top k$, leading to a recurrent update formula of $\bm{S}_t = \bm{S}_{t-1} (\bm{I} - a k_t^T k_t) + a v_t^T k_t $, where $a$ is a scalar learning rate. The ideas behind this internal state update can be traced back to fast weights \citep{fastweights}.

There has been significant recent interest in improvements to DeltaNet, in order to bring its efficiency and downstream performance in line with Transformers while still capturing the speed and memory benefits of Linear Attention. Parallelizing DeltaNet \citep{yang2024_deltanet} showed that DeltaNet used diagonal plus low-rank (DPLR) state evolution like S4 \citep{gu2022efficiently}, and could be parallelized across the time dimension, creating a path to efficiently train such models. Our work further extends that parallelization to cover the generalized delta rule formulation introduced herein, as well as the specific formula of RWKV-7.

\paragraph{Concurrent Work.} Concurrent work with our own has focused on architectural improvements beyond DeltaNet while still using the delta rule or variations thereof. Longhorn \citep{liu2024longhornstatespacemodels} employs an update rule that approximates a closed-form solution to a globally optimal update objective, applied on an otherwise unchanged Mamba architecture. Gated Delta Networks \citep{yang2024gated} applies gating to the DeltaNet state, essentially multiplying the transition matrix by a data-dependent scalar per head. This combines the DeltaNet update rule with the scalar decay found in some modern RNNs like RetNet and Mamba-2. The delta rule gradient descent formula with dynamic weight decay $w_t$ and learning rate $a_t$ becomes
$S_t = S_{t-1} \left(\operatorname{diag}(w_t) - k_t^\top k_t\operatorname{diag}(a_t)\right) + v_t^\top k_t\operatorname{diag}(a_t)$.

TTT (Test-Time Training) \citep{sun2024ttt} and Titans \citep{behrouz2024titans} also both apply scalar decay, but eschew per-step gradient descent update rules in favor of a batched multi-timestep approach. Titans also adds momentum to the otherwise classical SGD update applied to the state.

Another concurrent work with our own, Unlocking State-Tracking in Linear RNNs Through Negative Eigenvalues \citep{grazzi2024unlockingstatetrackinglinearrnns}, has demonstrated the potential for increased expressiveness that comes from allowing the state transition matrix to contain negative eigenvalues. We show a result significantly beyond this, proving that RWKV-7 and our generalized delta rule can recognize all regular languages using only a small constant number of layers.\footnote{For now, we choose to allow only part of the range of possible negative eigenvalues in our pre-trained large language models due to experimentally observed training instabilities.}

\section{Architecture}\label{sec:rwkv_arch}

Unlike the other work described above, RWKV-7 generalizes the delta update rule into an extended formula $S_t = S_{t-1}(\mathrm{diag}(w_t)+z_t^Tb_t)+v_t^Tk_t$ to increase expressivity (see Table \ref{tab:rnn_state_evolution}). This is still a diagonal plus rank one update rule, which admits efficient forms of parallelization \citep{yang2024_deltanet}. Here, $w_t$ is a more expressive data-dependent vector-valued decay, unlike the scalar decays featured in the other works previously described. Our use of $z_t$ and $b_t$ in this extended formula permits a flexible approach to state update, while retaining the important reduction in non-SRAM memory bandwidth usage that comes from using small data-dependent vectors instead of large matrices. One example of this flexibility is the ability to use a different removal key than replacement key. This extended delta rule is flexible enough that there may be other useful formulations that fit within it beyond the formulas we have chosen for calculating $z_t$ and $b_t$ in RWKV-7.

The original delta rule allowed a fixed scalar fraction of a value to be replaced from the state via its in-context learning rate parameter. RWKV-7's extended delta rule instead replaces data-dependent vector-valued amounts of the state, allowing each key channel in the state to vary independently. We parameterize $z_t = -\hat\kappa_t$ and $b_t = \hat\kappa_t \odot a_t$ where $a_t$ has elements in $(0,1)$. This keeps the update stable (see \hyperref[sec:eigenvalue-proof]{Appendix \ref*{sec:eigenvalue-proof}}), while maintaining increased expressivity. We demonstrate the improved performance of these design choices via ablations in \hyperref[sec:ablation-minipile]{Appendix \ref*{sec:ablation-minipile}}.

Previous research \citep{illusionstate_2024_merril} pointed out that Transformers and RNNs with a diagonal transition matrix could only represent functions in $\mathsf{TC}^0$. RWKV-7, however, has a non-diagonal and input-dependent transition matrix, allowing it to represent more complex functions than its predecessors. In fact, we demonstrate that RWKV-7 possesses expressive power surpassing that of $\mathsf{TC}^0$ under standard complexity conjectures and can recognize all regular languages. One new component of this power, not present in the original delta rule, is the ability to represent the "copy" state transition (Lemma \ref{lem:factor_transition}). This is a key element in our proof that RWKV-7 can recognize all regular languages with a constant number of layers. See \hyperref[sec:expressivity-proof]{Appendix \ref*{sec:expressivity-proof}} for proof and details.

We replace the main RWKV-6 \citep{rwkv6_colm} diagonal transition matrix with our extended delta rule and make several other changes to the RWKV architecture, observing significant modeling improvements. These include updates to the channel mixing module and the token shift module. We remove the data dependency of token-shift and the receptance gating of channel mixing, both of which contribute to faster training and inference. We increase the use of low-rank projections to generate more of our intermediate calculations, striking a balance between the total number of model parameters, training and inference speed, and downstream performance.

Figure \ref{fig:rwkv7-overall} presents the overall architecture of RWKV-7. Please refer to \hyperref[sec:arch_details]{Appendix \ref*{sec:arch_details}} for more details.

\begin{figure*}[ht!]
    \centering
    \includegraphics[width=1.0\linewidth]{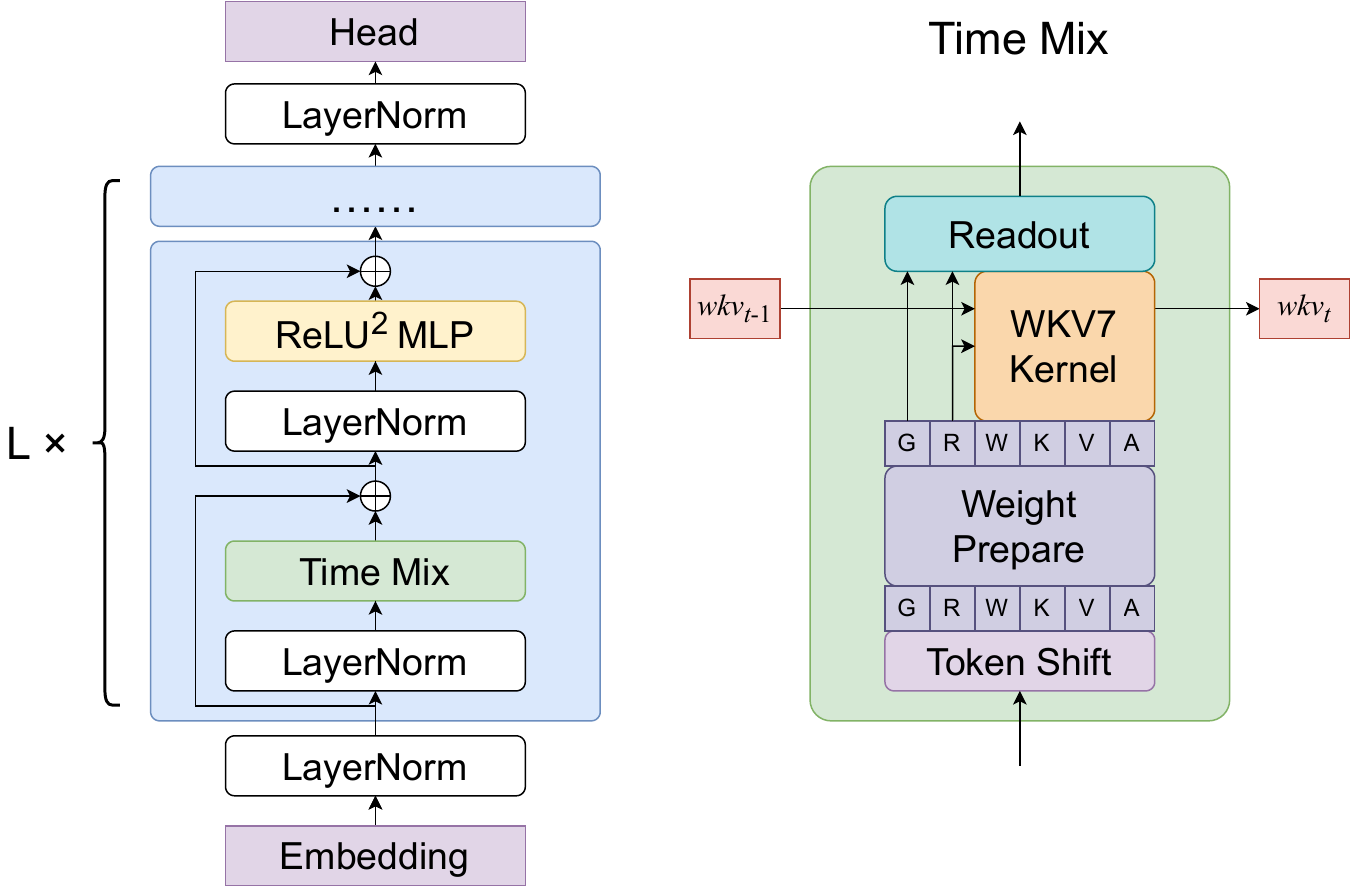}
    \caption{RWKV-7's overall architecture.}
    \label{fig:rwkv7-overall}
\end{figure*}

\begin{figure*}[ht!]
    \centering
    \includegraphics[width=\linewidth]{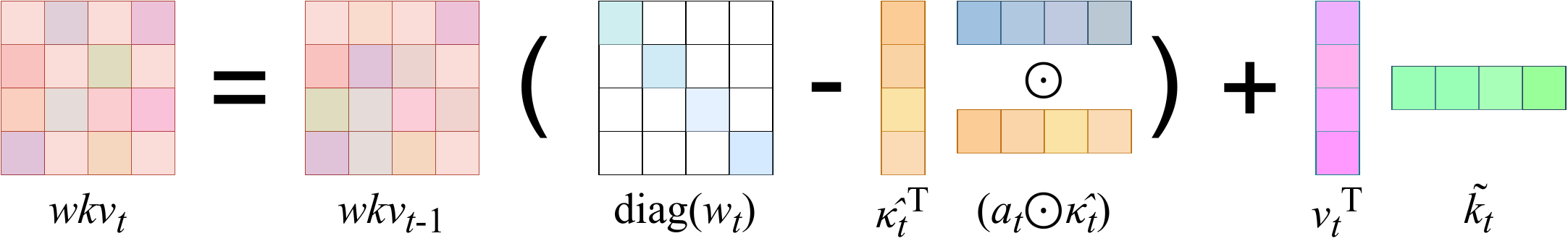}
    \caption{A simple illustration of the update mechanism of a single head of RWKV-7's state. Note that the actual state size is $64 \times 64$ per head, not $4 \times 4$.}
    \label{fig:rwkv7-upd}
\end{figure*}

\section{Method}\label{sec:method}
In this section, we use $D$ to denote the model dimension. Bold capital letters represent trainable matrices, and vectors without a subscript $t$ are trainable parameters. The first subscript denotes sequence position and second subscript denotes layer index, where necessary. We use the convention that all vectors are row vectors unless explicitly transposed, so all matrices operate on the right side, therefore $a^Tb$ is an outer product and $ab^T$ is an inner one. We use the square subscript to denote a placeholder for variable names and use the $\prod$ sign for cumulative matrix multiplication. 
See \hyperref[sec:pseudocode]{Appendix \ref*{sec:pseudocode}} for a pseudocode implementation of these formulas.

\subsection{Time Mixing}

\paragraph{Weight Preparation}
Along the lines of \citep{rwkv6_colm}, we introduce the following notation templates for common operators in the model, using the square subscript to denote a variable:
\begin{align}
\mathrm{lerp} (a, b, x) &= a + (b-a) \odot x, \\
\mathrm{loramlp}_\square(f, x, \text{bias}) &= f(x\bm{A}_\square)\bm{B}_\square + (\lambda_\square \text{ if bias else } 0), 
\end{align}
Unless explicitly stated, all vectors appearing in this section are dimension $D$.

We extend the use of low-rank MLP (a 2-layer MLP with small hidden dimension compared to input and output), abbreviated as $\mathrm{loramlp}$, to implement data dependency using minimal parameters. 

The replacement key $\tilde{k}$, value $v$, decay $w$, removal key $\kappa$, in-context learning rate $a$, receptance $r$, and rwkv gate $g$ parameters are computed as follows (outputs annotated with $\triangleright$):
\begin{align}
x^{\square}_{t} &= \mathrm{lerp}(x_t, x_{t-1}, \mu_{\square}) \quad \square\in \{r,k,v,d,a, g\}, &\text{token shifted inputs} \\
a_t &= \mathrm{sigmoid}(\mathrm{loramlp}_a(\mathrm{Identity},  x^a_t, \text{bias=True})), &\triangleright\text{in-context learning rate}\\
k_{t} &= x^{k}_t \bm{W}_{k}, &\text{key precursor}\\
\kappa_t &= k_t \odot \xi, &\triangleright\text{removal key}\\
\tilde{k}_t &= k_t \odot \mathrm{lerp}(1, a_t, \alpha), &\triangleright\text{replacement key}\\
\nu_{t} &= \mathrm{sigmoid}(\mathrm{loramlp}_\nu(\mathrm{Identity},  x^v_t, \text{bias=True})), &\text{value residual gate}\\
v'_{t,l} &= x^{v}_t \bm{W}_{v}, &\text{value precursor}\\
v_{t} &= \left\{\begin{aligned}
    & v'_{t,0}, & \text{layer } l=0 \\
    &\mathrm{lerp}(v'_{t,0}, v'_{t,l}, \nu_t), & \text{layer } l \ge 1
\end{aligned}\right.
,&\triangleright\text{value}\\
d_t &= \mathrm{loramlp}_d(\tanh, x^d_t, \text{bias=True}), &\text{decay precursor}\\
w_t &= \exp(-e^{-0.5} \mathrm{sigmoid}(d_t)), &\triangleright\text{decay}\\
r_{t} &= x^{r}_t \bm{W}_{r}, &\triangleright\text{receptance}\\
g_t &= \mathrm{loramlp}_g(\mathrm{sigmoid}, x^g_t, \text{bias=False}) &\triangleright\text{rwkv gate}
\end{align}
$\xi$ is a learned parameter representing the removal key multiplier, which transforms the original key into a version to be removed from the state. In practice, $\xi$ lies in a range of approximately $[-5.3,9.4]$. 

$\alpha$ is a learned parameter representing the replacement rate booster, which adjusts the amount added back to the state after the transition matrix is applied.

Unlike $r, k$ and $v$ which are the main carriers of information, $g, d, \nu$ and $a$ act like gates which control the amount of information allowed to pass.

For comprehensive statistics of $\xi$, $\alpha$ and biases of $d_t$ observed in the released RWKV-7 model, including extremum values, mean measurements, and distribution trends, see \hyperref[sec:param_stat]{Appendix \ref*{sec:param_stat}}.

For the computation of $x^{\square}_{t}$, we removed data dependency of linear interpolation from RWKV-6 to improve training speed.

We adapted the idea of Value Residual Learning \citet{zhou2024valueresiduallearningalleviating} for the computation of $v_t$, which has shown to improve the final language modeling loss. 
$\nu_t$ represents the value residual mix, which interpolates between the layer zero and current layer value precursors: $v_{t,0}$ and $v_{t,l}$.

We also updated the formula for computation of $w_t$, restricting all entries in $(\exp(-e^{-0.5}), 1)$ in favor of a smaller condition number for $\mathrm{diag} (w_t)$, which maintains better training stability, and was beneficial to accuracy of the backward pass.

The $\tilde{k}_t$ in the formula can be regarded as a "normalized key", a design to ensure that the state of $\bm{wkv}$ contains columns of $O(1)$ size. Normally, we expect $\tilde{k}_t = k_t \odot (1-w_t)$, as employed in RWKV-6c (see Appendix \ref{sec:arch_details}), so that $\bm{wkv}_t$ rows are linear interpolations between $\bm{wkv}_{t-1}$ and $v_t ^T k_t$ controlled by $w_t$. However, to further enhance expressivity, we decide to decouple $w_t$ and $a_t$. We further decouple $a_t$ from the amount actually added to the state, allowing the replacement rate booster $\alpha$ to interpolate the amount added between the normal in-context learning rate and 1.0. Importantly, all of these modifications operate on a per-channel basis. The numerical range of RWKV-7's $\bm{wkv}$ entries are generally stable as in RWKV-6c, unlike RWKV-6, where entries of the states can accumulate to thousands (see Appendix~\ref{sec:state_inspections} for a state visualization).



\paragraph{The Weighted Key Value State Evolution}
After weight preparation, we reshape $(r, w, \tilde{k}, v, \kappa, a)_t$, splitting them to $h$ heads, with each head sized $D/h$. We always assume that $h$ is a factor of $D$ and heads are equally split. All operations in this section are shown per-head.

Before mixing in the time dimension, $\kappa_t$ is normalized per head:
\begin{align}
    \label{eq:kappa_norm}
    \hat{\kappa}_t &= \kappa_t / \lVert\kappa_t\rVert_2
\end{align}
The $\bm{wkv}$ (Weighted Key Value) is a multi-headed matrix-valued state of fast weights that undergoes dynamic evolution. The evolution of $\bm{wkv}$ is crucial for encoding context information by learning at test time to map keys to values. 
We start by defining the WKV time mixing as the recurrence relation
\begin{align}
    \label{eq:wkv7}
    \bm{wkv}_0 &= \bm{0}, \\
    \bm{wkv}_t &= \bm{wkv}_{t-1} \left(\mathrm{diag}(w_{t}) - \hat{\kappa}^T_t (a_t \odot \hat{\kappa}_t)\right) + v_t^T \cdot \tilde{k}_t 
\end{align}
Compared to RWKV-5 and RWKV-6, the $\bm{wkv}$ in this paper is transposed to ensure consistency with RWKV-7's code.


The $\bm{wkv}_{t}$ attention calculation can alternatively be written in a parallel manner: 
\begin{align}
\bm{wkv}_{t} &= \sum_{i=1}^{t} \left( v_i^T \tilde{k}_i \prod_{j=i+1}^{t} \left(\mathrm{diag}(w_{j}) - \hat{\kappa}^T_j (a_j \odot\hat{\kappa}_j)\right) \right) \in \mathbb{R}^{(D/h) \times (D/h)}
\end{align}
The recurrent transition design has parallels with \citet{schlag_deltarule}, but crucially the transition matrix
\begin{equation}
G_t = \mathrm{diag}(w_t) - \hat{\kappa}^T_t (a_t \odot \hat{\kappa}_t) = \left(\bm{I} - \hat{\kappa}_t^T (\frac{a_t}{w_t} \odot \hat{\kappa}_t)\right) \mathrm{diag}(w_t)  \approx  \left(\bm{I} - 2 \hat{\kappa}_t^T \hat{\kappa}_t\right) \mathrm{diag}(w_t)
\end{equation}
is no longer a Householder matrix but a scaled approximation of it, as $\hat{\kappa}_t \neq \tfrac{a_t}{w_t}\hat{\kappa}_t $. This mimics a Householder matrix but with expanded dynamics, while still having all eigenvalues in a stable range of $[-1, 1]$ and allows the network to decay information in all subspaces if necessary. It contrasts with the case of a Householder-like matrix with learning rate $(I - a v^Tv), \: a \in [0,1]$, as used in \citet{schlag_deltarule, yang2024_deltanet} where all eigenvalues are one except for the last one corresponding to $1-a$. Given these properties, we refer to $w_t$ as "in-context weight decay" and to $a_t$ as "in-context learning rate" (ICLR). The RWKV-7 transition matrix, therefore, allows for both dynamic state evolution and approximation to a forget gate at the same time. See \hyperref[sec:eigenvalue-proof]{Appendix \ref*{sec:eigenvalue-proof}} for the details on the eigenvalue of the transition matrix, and when the transition matrix is guaranteed to be stable.

The original delta rule in \citet{schlag_deltarule} allows partial or full removal of pre-existing values from the state at each time-step, with the amount removed being equal to the scalar $a$. Our formulation extends this ability by making $a$ a vector, allowing for different removal amount per state column.

\paragraph{WKV Bonus and Output}

All operations in this section are shown per-head unless otherwise specified.

Receptance, which acts like the query found in transformers, is applied to the WKV state, and the result is normalized. An added bonus, the amount of which is weighted by $\rho$, allows the model to place extra attention on the current shifted input token without requiring it to store that token in the state.
\begin{align}
u_t &= \left (r_t \cdot (\rho \odot \tilde{k}_t)^T \right) v_{t} & \text{bonus}\label{subeq:addbonus}\\
p_t &= \mathrm{LayerNorm}(r_t \bm{wkv}_t^T) + u_t & \triangleright\text{attention result}
\end{align}
Finally, the heads are recombined via reshaping so that $p_t \in \mathbb{R}^{D}$, gated, and transformed into the output as follows:
\begin{align}
o_t &= (g_t \odot p_t) \bm{W}_o \in \mathbb{R}^{D}
\end{align}
\subsection{MLP}
The MLP module of RWKV-7 is no longer identical to the Channel Mixing module of previous RWKV-4,5,6 architectures \citep{rwkv6_colm}. We remove the gating matrix $\bm{W}_r$, making it a two-layer MLP. In compensation for the removed gating parameters to satisfy the equi-parameter condition, we set the hidden dimension to be 4 times the size of model dimension.
\begin{align}
\label{eq:channel-mix7}
k'_t &= \mathrm{lerp} (x'_t, x'_{t-1}, \mu'_{k}) \bm{W}_{k'}\in \mathbb{R}^{4D}\\
o'_t &= \mathrm{ReLU}(k'_t)^2 \bm{W}_{v'}\in \mathbb{R}^{D}
\end{align}


\section{RWKV World v3 Dataset}\label{sec:dataset}

We train our models on the new \textbf{RWKV World v3 Dataset}, a new multilingual 3.119 trillion token dataset drawn from a wide variety of publicly available data sources. This dataset aims to help close the gap with the amount of data used to train modern LLMs, which may consume as many as 15 - 18 trillion tokens \citep{qwen2025qwen25technicalreport, grattafiori2024llama3herdmodels}. We select the data to approximate the distribution of our previous World datasets, including English, multilingual, and code, while slightly enhancing Chinese novels.
We describe the composition of our dataset in
\hyperref[sec:training_dataset_details]{Appendix \ref*{sec:training_dataset_details}}.

\section{Pre-Trained Models} \label{sec:pretrained_models}

We have pre-trained and publicly released seven Apache 2.0 licensed RWKV-7 models: 
\begin{enumerate}
    \item Trained on Pile: \textbf{RWKV7-Pile} of sizes \textbf{0.1B, 0.4B, and 1.4B}
    \item Trained on RWKV World V3: \textbf{RWKV7-World-3} of sizes \textbf{0.1B, 0.4B, 1.5B, and 2.9B}
\end{enumerate}

See \hyperref[sec:arch_train]{Appendix \ref*{sec:arch_train}} for detailed configurations.

The RWKV-7 Pile models all use the GPT-NeoX-20B tokenizer \citep{black-etal-2022-gpt}, and were all trained from scratch on the Pile dataset, which has 332 billion tokens. 

All RWKV World dataset models use the RWKV World Tokenizer. Due to compute budget constraints, the Goose World 3 0.1B and 0.4B models were trained from pre-existing RWKV-5 World v1 and v2 checkpoints, and the Goose World 3 1.5B and 2.9B models were trained from pre-existing RWKV-6 World v2.1 checkpoints. These checkpoints' parameters were then converted to the RWKV-7 format via a process described below. Once in the new format, the models are trained on either the additional full 3.1 trillion tokens of the World v3 corpus, or an equally weighted sub-sampling of it. Under this methodology, some documents were seen two or even three times.

The World v1, v2, v2.1, and v3 corpora contain 0.6, 1.1, 1.4, and 3.1 trillion tokens, respectively. The amounts of training in each stage at with each successive model architecture and corpus are shown in Table \ref{tab:staged-training-details}.

\begin{table}[ht]
\setlength\extrarowheight{-5pt}
\centering
\begin{adjustbox}{max width=1.0\linewidth}
\begin{tabular}{lrrrrr}
\toprule
Model & World v1 & World v2 & World v2.1 & World v3 & Total \\
\midrule
RWKV7-World3-0.1B  & 0.6 (RWKV-5) & & & 1.0 (RWKV-7) & 1.6\\
RWKV7-World3-0.4B &      & 1.1 (RWKV-5) & & 2.0 (RWKV-7) & 3.1 \\
RWKV7-World3-1.5B &      & 1.1 (RWKV-6) & 1.4 (RWKV-6) & 3.1 (RWKV-7) & 5.6 \\
RWKV7-World3-2.9B &      & 1.1 (RWKV-6) & 1.4 (RWKV-6) & 3.1 (RWKV-7) & 5.6 \\
\bottomrule
\end{tabular}
\end{adjustbox}
\caption{\centering{Total trillions of tokens trained for all RWKV-7 World 3 models}}
\label{tab:staged-training-details}
\end{table}


Our model format conversion process involves removing the token-shift low-rank MLPs, rescaling by half the embeddings, wkv receptance, wkv output matrix weights, and Layernorm and Groupnorm bias values. Layernorm and Groupnorm weights are clamped above zero and square rooted. We widen the FFN MLP from 3.5x (in RWKV-6) to 4x and add new small (\num{1e-3}) uniform initalizations in the new regions, removing the RWKV-6 FFN receptance weights. We widen the time decay Low-rank MLP and add new small (\num{1e-4}) uniform initializations in the new regions. We replace the gate weights with a LoRA obtained through singular value decomposition and rescaling by half. 

\section{Language Modeling Experiments}\label{evaluations}


\subsection{LM Evaluation Harness Benchmarks}\label{subsec:evals}


RWKV-7 models are evaluated on a series of common English-focused and multilingual benchmarks using \textit{LM Evaluation Harness} \citep{gao10256836framework} as shown in Tables~\ref{tab:eng_bench} and~\ref{tab:multilang_bench}. We benchmarked RWKV-7 along with several new open models which are state-of-the-art in their parameter count ranges. All numbers are evaluated under fp32 precision with lm-eval v0.4.8 using 0-shot, except for MMLU in which case 5-shot was used. 

We find that RWKV-7 is generally able to match the English performance of Qwen2.5 \citep{qwen2025qwen25technicalreport} with less than one third as many training tokens. Interestingly, we found that RWKV-7 models have shown giant leaps in MMLU performance compared to RWKV-6. We also find that RWKV-7-World models expand upon RWKV-6-World models already strong capabilities on multi-lingual benchmarks, outperforming SmolLM2 \citep{allal2025smollm2smolgoesbig}, Llama-3.2 \citep{grattafiori2024llama3herdmodels}, and Qwen-2.5 \citep{qwen2025qwen25technicalreport} by a significant margin.

In Figures \ref{fig:multi-flops} and \ref{fig:eng-flops} we plot FLOPs used to train several open models versus average accuracy across the same sets of common english and multi-lingual benchmarks. The multilingual evals show a very dramatic Pareto improvement versus the transformer models. Also note the similar english-language eval scores, but dramatically lower total FLOPs usage of RWKV7-World models versus other highly trained open transformer models. We theorize that if we were less constrained by compute and were able to train these models from scratch with the same amount of total tokens instead of from pre-trained checkpoints of earlier RWKV versions, the difference would be even more dramatic. Note that we did not plot the Llama 3.2 series of models, as they have no corresponding FLOPs amounts due to having been created via pruning and distillation from larger models.

\begin{figure*}[t!]
    \centering
    \begin{subfigure}[t]{0.48\linewidth}
        \centering
        \includegraphics[height=2.6in]{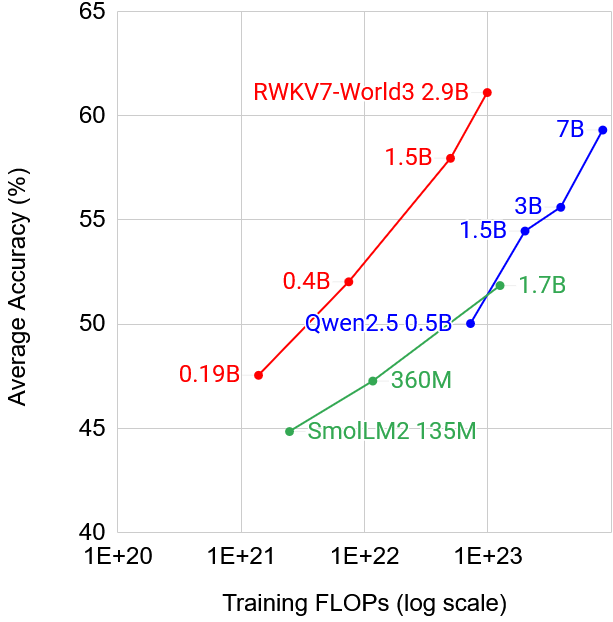}
        \caption{FLOPS vs. Average Benchmark Accuracy}
        \label{fig:multi-flops}
    \end{subfigure}%
    ~ 
    \begin{subfigure}[t]{0.5\linewidth}
        \centering
        \includegraphics[height=2.6in]{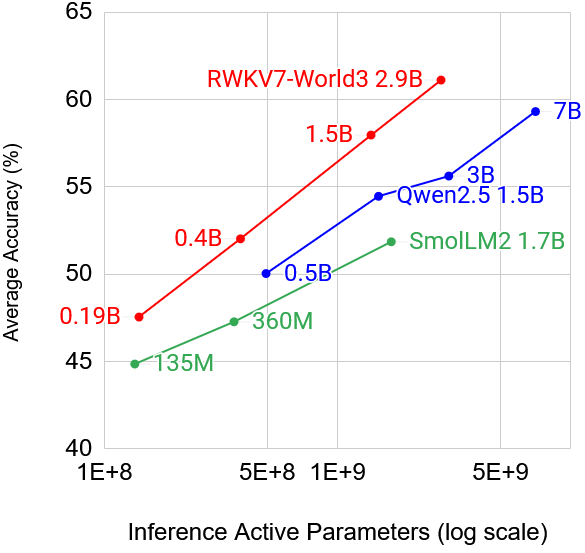}
        \caption{Active Parameters vs. Average Benchmark Accuracy}
        \label{fig:multi-activeparams}
    \end{subfigure}
    \caption{Model Comparisons across Multilingual Benchmarks}
\end{figure*}

\begin{figure*}[t!]
    \centering
    \begin{subfigure}[t]{0.48\linewidth}
        \centering
        \includegraphics[height=2.6in]{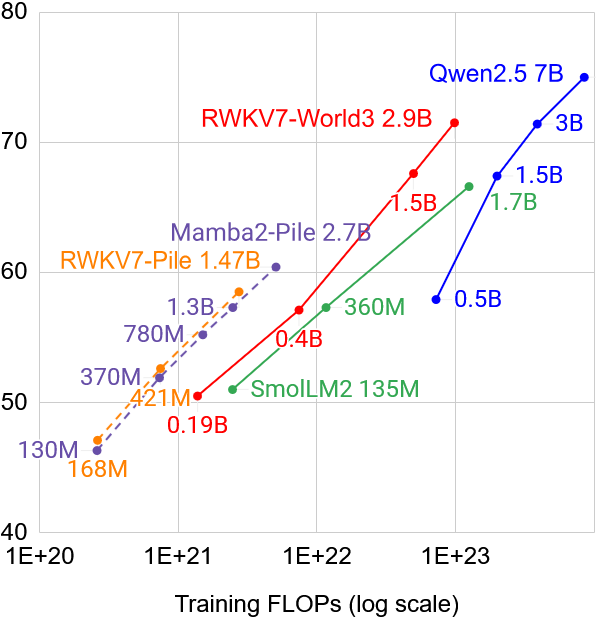}
        \caption{FLOPS vs. Average Benchmark Accuracy}
        \label{fig:eng-flops}
    \end{subfigure}%
    ~ 
    \begin{subfigure}[t]{0.5\linewidth}
        \centering
        \includegraphics[height=2.6in]{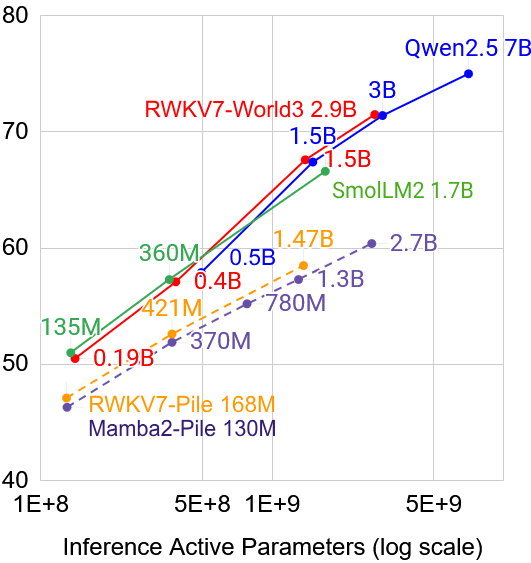}
        \caption{Active Parameters vs. Average Benchmark Accuracy}
        \label{fig:eng-activeparams}
    \end{subfigure}
    \caption{Model Comparisons across English-Language Benchmarks}
\end{figure*}





\begin{table}[ht]
    \setlength\extrarowheight{-5pt}
    \setlength\tabcolsep{5pt}
    \centering
    \begin{adjustbox}{max width=1.0\linewidth}
    \begin{tabular}{lrrrrrrrrrrr}
\toprule
    \bf Model & \bf Tokens & \bf lmb.o & \bf hella & \bf piqa & \bf arcE & \bf arcC & \bf glue & \bf WG & \bf sciq & \bf mmlu & \bf avg \\
       (Name) & (T) & acc$\uparrow$ & acc\_n$\uparrow$  & acc$\uparrow$ & acc$\uparrow$ & acc$\uparrow$ & $^a$acc$\uparrow$ & acc$\uparrow$ & acc$\uparrow$ & acc$\uparrow$ & acc$\uparrow$ \\
\midrule
    RWKV5-World1-0.1B & 0.6  & 38.4  & 31.9   & 61.4 & 44.2 & 19.9 & 45.5 & 52.9 & 76.3 & 23.1 & 43.7 \\
    \textbf{SmolLM2-135M} & 2.0 & 42.9 & \textbf{43.1} & \textbf{68.4} & \textbf{64.4} & \textbf{28.1} & \textbf{49.0} & \textbf{53.0} & 84.0 & \bf 25.8 & \textbf{51.0} \\
    RWKV7-World2.8-0.1B & 1.6    & \bf 48.1  & 42.1   & 67.3 & 59.3 & 25.5 & 48.1 & 52.7 & \bf 86.3 & 25.4 & 50.5 \\
 \midrule
    RWKV5-World2-0.4B & 1.1    & 54.0  & 40.9   & 66.5 & 54.0 & 24.0 & 50.0 & 53.2 & 86.9 & 23.8 & 50.4 \\
    SmolLM2-360M & 4.0 & 53.8 & 56.4 & 72.1 & \textbf{70.4} & \textbf{36.5} & \bf 50.7 & 59.0 & 91.2 & 26.3 & 57.4 \\
    \textbf{Qwen2.5-0.5B} & 18.0 & 52.5 & 52.1 & 70.2 & 64.6 & 29.5 & 54.7 & 56.4 & \textbf{93.1} & \textbf{47.8} & \textbf{57.9} \\
    RWKV7-World2.9-0.4B & 3.1    & \bf 58.6  & \bf 56.8   & \bf 72.9 & 68.7 & 31.9 & 49.4 & \bf 59.9 & 89.7 & 26.1 & 57.1 \\
\midrule
    RWKV6-World2.1-1.6B & 2.5 & 67.4  & 61.1   & 74.4 & 64.3 & 31.0 & 51.0 & 60.7 & 89.5 & 25.1 & 58.3 \\
    Llama3.2-1B & $^b$15.0 & 63.0  &  63.7 & 74.5 & 65.5 & 31.3  & 49.7 & 60.7  & 91.4 & 32.1 & 59.1 \\
    SmolLM2-1.7B & 11.0 & 67.7 & \textbf{71.5} & \textbf{77.0} & 77.7 & \textbf{44.7} & 51.5 & 66.1 & 93.3 & 50.3 & 66.6 \\
    \textbf{Qwen2.5-1.5B} & 18.0 & 63.0 & 67.7 & 75.8 & 75.5 & 41.2 & \textbf{65.0} & 63.4 & \bf 94.2 &  \textbf{61.0} & \bf 67.4 \\
    \textbf{RWKV7-World3-1.5B} & 5.6 & \bf 69.5  & 70.8   & \bf 77.1 & \bf 78.1 & \bf 44.5 & 62.4 & \bf 68.2 & \bf 94.3 & 43.3 & \bf 67.6 \\
\midrule
    RWKV6-World2.1-3B & 2.5    & 71.7  & 68.4   & 76.4 & 71.2 & 35.6 & 56.3 & 66.3 & 92.2 & 28.3 & 62.9\\
    Llama3.2-3B  & $^b$15.0 & 70.5 &  73.6 & 76.7 &  74.5 & 42.2 & 50.7 & 69.9  & 95.7  & 56.5 & 67.8 \\
    \textbf{Qwen2.5-3B}  & 18.0 & 67.1 & 73.5 & 78.6 & 77.4 & 45.0 & \textbf{70.2} & 68.5 & \textbf{96.2} & \textbf{65.7}  & \textbf{71.4} \\
    \textbf{RWKV7-World3-2.9B}  & 5.6    & \bf 73.4  & \bf 76.4   & \bf 79.7 & \bf 81.0 & \bf 48.7 & 61.8 & \bf 72.8 & 95.0 & 55.0 & \bf 71.5 \\
\bottomrule
\multicolumn{12}{c}{\footnotesize $^a$ \textbf{glue} is the average accuracy of 8 subtasks: \textbf{mnli}, \textbf{mnli\_mismatch}, \textbf{mrpc}, \textbf{qnli}, \textbf{qqp}, \textbf{rte}, \textbf{sst2} and \textbf{wnli}} \\
\multicolumn{12}{c}{\footnotesize $^b$ Llama3.2-1B and 3B were pruned and distilled from Llama3.1-8B \citep{grattafiori2024llama3herdmodels}}
    \end{tabular}
    \end{adjustbox}
  \caption{\centering{English Focused Benchmarks, including LAMBADA (\textbf{lmb.o}) \citep{paperno2016lambada}, Hellswag (\textbf{hella}) \citep{hampel1974influence}, PIQA \citep{bisk2020piqa}, AI2 ARC (\textbf{arcE}, \textbf{arcC}) \citep{bhakthavatsalam2021think}, GLUE \citep{wang2018glue}, Winogrande (\textbf{WG}) \citep{sakaguchi2021winogrande}, SciQ \citep{welbl2017crowdsourcing}, MMLU \citep{hendrycks2021measuringmassivemultitasklanguage}. 
  }}
  \label{tab:eng_bench}

\end{table}%

\begin{table}[ht]
    \setlength\extrarowheight{-5pt}
    \centering
    \begin{adjustbox}{max width=1.0\linewidth}
    \begin{tabular}{lrrrrrrrrrr}
\toprule
    \bf Model & \bf Tokens & \bf lmb.m & \bf lmb.m & \bf pawsx & \bf xcopa & \bf xnli & \bf xsClz & \bf xwin & \bf avg \\
      (Name) & (T) & $^a$ppl$\downarrow$ & acc$\uparrow$ & acc$\uparrow$ & acc$\uparrow$ & acc$\uparrow$ & acc$\uparrow$ & acc$\uparrow$ & acc$\uparrow$ \\
\midrule
    RWKV5-World1-0.1B & 0.6 & 270   & 22.0  & 48.6  & 53.0  & 36.1 & 51.7 & 59.5 & 45.1 \\
    SmolLM2-135M & 2.0 & 1514 & 18.6 & \bf 51.2 & 52.2 & 34.9 & 50.6 & 61.7 & 44.9 \\
    \textbf{RWKV7-0.1B} & 1.6 & \bf 114 & \bf 31.6  & 46.1 & \bf 53.3 & \bf 37.6 & \bf 52.6 & \bf 64.1 & \bf 47.5 \\
\midrule
    RWKV5-World2-0.4B & 1.1 & 66 & 36.8  & 49.5  & 54.0  & 38.5 & 54.1  & 65.6 & 49.8 \\
    SmolLM2-360M & 4.0 & 389 & 25.8 & 51.4 & 51.7 & 36.0 & 51.2 & 67.8 & 47.3 \\
    Qwen2.5-0.5B & 18.0 & 108 & 32.9 & \bf 52.6 & 54.4 & 38.6 & 53.9 & 67.8 & 50.0 \\
    \textbf{RWKV7-World3-0.4B} & 3.1 & \bf 52 & \bf 39.6 & 48.7  & \bf 55.4  &  \bf 40.3 & \bf 55.3  & \bf 72.9 & \bf 52.0 \\
\midrule
    RWKV6-World2.1-1.6B & 2.5 & 28 & 47.2  & 52.5  & 58.1  & 41.4 & 58.2  & 76.5 & 55.7 \\
    Llama3.2-1B & $^b$15.0 & 52 & 39.0 & 53.9 & 55.3 & 41.2  & 56.6 & 72.2 & 53.0 \\
    SmolLM2-1.7B & 11.0 & 85 & 37.1 & \bf 56.5 & 53.1 & 38.1 & 54.1 & 72.8 & 52.0 \\
    Qwen2.5-1.5B & 18.0 & 49 & 40.0 & 55.3 & 57.4 & 40.6 & 57.7 & 75.8 & 54.5 \\
    \textbf{RWKV7-World3-1.5B} & 5.6 & \bf 25   & \bf 48.4  & 54.8  & \bf 59.7  & \bf 43.7 & \bf 61.4  & \bf 79.8 & \bf 58.0 \\
\midrule
    RWKV6-World2.1-3B & 2.5 & 21 & 51.0  & 53.4  & 60.2  & 42.7 & 61.3  & 78.8 & 57.9 \\
    Llama3.2-3B & $^b$15.0 & 30 & 45.9 & \bf 59.9 & 58.5 & 44.2 & 60.6 & 79.2 & 58.1 \\
    Qwen2.5-3B & 18.0 & 36 & 43.5 & 53.3 & 59.0 & 38.5 & 59.6 & 79.8 & 55.6 \\
    \textbf{RWKV7-World3-2.9B} & 5.6 & \bf 18 & \bf 52.9  & 58.2  & \bf 63.1  & \bf 45.4 & \bf 64.7 & \bf 82.4 & \bf 61.1 \\
\bottomrule
\multicolumn{10}{c}{\footnotesize $^a$ The perplexity is the geometric mean, rather than arithmetic average, across 5 languages} \\
\multicolumn{10}{c}{\footnotesize $^b$ Llama3.2-1B and 3B were pruned and distilled from Llama3.1-8B \citep{grattafiori2024llama3herdmodels}}
\end{tabular}
\end{adjustbox}
\caption{\centering{Multilingual Benchmarks, including LAMBADA Multilingual (\textbf{lmb.m}) \citep{gao10256836framework}, XCOPA \citep{ponti-etal-2020-xcopa}, XNLI \citep{conneau2018xnli},XStoryCloze (\textbf{xsClz}) \citep{lin2022few}, xWinogrande (\textbf{xwin}) \citep{tikhonov2021s}.
}}

\label{tab:multilang_bench}%
\end{table}%


\subsection{Recent Internet Data Evaluation}\label{subsec:uncheatable_eval}

Modern large language models are trained on massive datasets. Despite careful data cleaning, benchmark data leakage remains a challenge, compromising the validity of these evaluations. To complement traditional benchmarks, we evaluated RWKV-7 Goose and other leading open-source models using temporally novel internet data, generated after the models' training periods; this data could not have appeared in the training sets, removing data leakage concerns.

Specifically, we collected new data created after January 2025, including: newly submitted computer science and physics papers on arXiv, newly created Python/C++ open-source repositories on GitHub, recently published Wikipedia entries, new fiction on Archive of Our Own \citep{ao3_whole}, and recent news articles. Inspired by \citet{delétang2024languagemodelingcompression, li2024evaluatinglargelanguagemodels}, we used compression rate as our evaluation metric. See Table~\ref{tab:cr_comparison} for details. 

Remarkably, despite being trained on significantly less data than other top models, RWKV-7 Goose showed competitive performance on this temporally novel data.

\begin{table}[ht]
 \setlength\extrarowheight{-5pt}
 \setlength\tabcolsep{5pt}
 \centering
 \footnotesize
 \begin{tabular}{lrrrrrrrr}
   \toprule
   \addlinespace[2pt]
   \textbf{Model} & \textbf{\shortstack{arXiv\\CS $\downarrow$ }} & \textbf{\shortstack{arXiv\\Phys. $\downarrow$ }} & \textbf{\shortstack{Github\\Python $\downarrow$ }} & \textbf{\shortstack{Github\\C++ $\downarrow$ }} & \textbf{\shortstack{AO3\\Eng $\downarrow$ }} & \textbf{\shortstack{BBC\\news $\downarrow$ }} & \textbf{\shortstack{Wiki\\Eng $\downarrow$ }} & \textbf{average $\downarrow$ } \\
   \midrule
   \textbf{Qwen2.5-1.5B} & \textbf{8.12} & 8.65 & \textbf{4.42} & \textbf{4.40} & 11.76 & 9.58 & 9.49 & \textbf{8.06} \\
   RWKV-7 1.5B & 8.25 & 8.77 & 5.57 & 5.29 & \textbf{10.93} & \textbf{9.34} & \textbf{8.97} & 8.16 \\
   Llama-3.2-1B & 8.37 & 8.76 & 5.18 & 5.16 & 11.69 & \textbf{9.34} & 9.07 & 8.23 \\
   SmolLM2-1.7B & 8.38 & 9.04 & 5.17 & 4.94 & 11.20 & 9.40 & 9.46 & 8.23 \\
   Index-1.9B & 8.34 & \textbf{8.59} & 5.65 & 5.29 & 11.49 & 9.51 & 9.23 & 8.30 \\
   stablelm-2-1.6b & 8.58 & 9.08 & 5.54 & 5.45 & 11.42 & 9.24 & 9.06 & 8.34 \\
   RWKV-6 1.5B & 8.62 & 9.00 & 6.06 & 5.80 & 11.09 & 9.57 & 9.30 & 8.49 \\
   RWKV-5 1.5B & 8.77 & 9.11 & 6.20 & 5.92 & 11.25 & 9.75 & 9.50 & 8.64 \\
   mamba2-1.3b & 8.74 & 8.74 & 6.32 & 5.71 & 11.63 & 9.74 & 9.86 & 8.68 \\
   MobileLLM-1.5B & 8.82 & 9.29 & 6.79 & 6.29 & 11.59 & \textbf{9.15} & 9.22 & 8.73 \\
   mamba-1.4b-hf & 8.88 & 8.86 & 6.43 & 5.81 & 11.70 & 9.83 & 9.97 & 8.78 \\
   Zamba2-1.2B & 8.57 & 9.21 & 6.91 & 7.08 & 11.39 & 9.38 & 9.26 & 8.83 \\
   SmolLM-1.7B & 8.38 & 9.02 & 5.76 & 6.55 & 12.68 & 9.85 & 9.89 & 8.88 \\
   MobileLLM-1B & 9.03 & 9.57 & 7.03 & 6.53 & 11.86 & 9.35 & 9.43 & 8.97 \\
   RWKV-4 1.5B & 9.34 & 9.80 & 6.54 & 6.16 & 11.33 & 10.00 & 9.82 & 9.00 \\
   pythia-1.4b-v0 & 9.12 & 9.20 & 6.79 & 6.15 & 12.19 & 10.20 & 10.43 & 9.15 \\
   Falcon3-1B-Base & 8.60 & 9.20 & 6.92 & 7.16 & 13.04 & 10.45 & 10.75 & 9.45 \\
   \midrule
   \textbf{Llama-3.2-3B} & \textbf{7.78} & \textbf{8.10} & 4.15 & 4.59 & 10.90 & \textbf{8.70} & \textbf{8.28} & \textbf{7.57} \\
   Qwen2.5-3B & 7.79 & 8.25 & \textbf{4.15} & \textbf{4.12} & 11.23 & 9.15 & 8.96 & 7.66 \\
   RWKV-7 2.9B & 7.90 & 8.34 & 5.16 & 4.88 & \textbf{10.48} & 8.92 & 8.47 & 7.74 \\
   stablelm-3b-4e1t & 8.15 & 8.50 & 5.28 & 4.85 & 10.89 & 8.82 & 8.51 & 7.86 \\
   Minitron-4B-Base & 8.09 & 8.70 & 5.13 & 4.74 & 11.05 & 9.08 & 8.90 & 7.96 \\
   recurrentgemma-2b & 8.24 & 8.52 & 5.22 & 4.80 & 11.30 & 8.94 & 8.88 & 7.99 \\
   RWKV-6 3B & 8.27 & 8.58 & 5.66 & 5.39 & 10.67 & 9.17 & 8.82 & 8.08 \\
   gemma-2-2b & 8.39 & 8.81 & 5.36 & 5.01 & 11.35 & 8.90 & 9.03 & 8.12 \\
   mamba2attn-2.7b & 8.33 & 8.29 & 5.78 & 5.22 & 11.13 & 9.28 & 9.26 & 8.18 \\
   RWKV-5 3B & 8.42 & 8.70 & 5.78 & 5.51 & 10.83 & 9.36 & 9.00 & 8.23 \\
   mamba2-2.7b & 8.43 & 8.37 & 5.93 & 5.34 & 11.21 & 9.37 & 9.38 & 8.29 \\
   Zamba2-2.7B & 8.17 & 8.70 & 6.30 & 6.39 & 10.97 & 8.95 & 8.74 & 8.32 \\
   mamba-2.8b-hf & 8.57 & 8.52 & 6.03 & 5.46 & 11.31 & 9.49 & 9.53 & 8.41 \\
   RWKV-4 3B & 8.90 & 9.27 & 6.07 & 5.67 & 10.90 & 9.57 & 9.30 & 8.53 \\
   pythia-2.8b-v0 & 8.72 & 8.73 & 6.29 & 5.71 & 11.66 & 9.74 & 9.82 & 8.67 \\
   \bottomrule
 \end{tabular}
 \caption{Compression rate (unit: \%) compared across different language models on various data sources, including arXiv papers, GitHub repositories, AO3 fiction, and news articles created after January 2025.}
 \label{tab:cr_comparison}
\end{table}

\subsection{Associative Recall}
\label{subsec:associative_recall}

Associative recall (AR) \citep{arora2023zoology} evaluates the ability of the model to recall previously encountered information within a given context. Research indicates that a model’s capacity for AR can reflect its effectiveness in learning from context \citep{elhage2021mathematical, olsson2022incontext}. Consequently, AR has become a standard benchmark for developing new architectural designs in language models \citep{fu2023hungry, poli2023hyena, lutati2023focus}.



We train two-layer RWKV-7 with MQAR and increased the difficulty by scaling the sequence length to as long as 2048. We use the RWKV-7 specific initialization, and set the $\epsilon$ of AdamW to \num{1e-18} to stabilize learning in later stages. Weight decay of 0.1 is only applied to weight matrices, preventing the degeneration of certain modules (such as weights and biases of LayerNorm).

\begin{table}[ht]
\setlength\extrarowheight{-5pt}
\centering
\begin{tabular}{rrcccccc}
\toprule
Dim & WKV state dim & $(64, 4)$ & $(128, 8)$ & $(256,16)$ & $(512,64)$ & $(1024,128)$ & $(2048,256)$ \\
\midrule
64 & 8192 & \checkmark & \checkmark & \checkmark & 98.43 & 95.01 & 72.93 \\
128 & 16384 & \checkmark & \checkmark & \checkmark & \checkmark & \checkmark & 94.97 \\
256 & 32768 & \checkmark & \checkmark & \checkmark & \checkmark & \checkmark & 98.97 \\
512 & 65536 & \checkmark & \checkmark & \checkmark & \checkmark & \checkmark & \checkmark \\
\bottomrule
\end{tabular}
\caption{RWKV-7 MQAR test results. We use $(a, b)$ to denote the sequence length and number of KV pairs respectively. A check mark \checkmark indicates that the model achieves over 99\% accuracy. Results are maxed over 3 different learning rate settings.}
\label{tab:mqar}
\end{table}

Interestingly, with only a WKV size of 8192, RWKV-7 is able to recall 72.93\% at the setting of 256 Key-value pairs. This suggests that a total of roughly $256 \times 0.7293 \times \log_2(\text{number of key tokens} \times \text{number of value tokens}) = 186.2 \times 2 \times \log_2(4096) = 4480.8$ bits of information, is stored in a $8192$ dimensional state, yielding an information density of $0.547$ bits per dimension.

\subsection{Mechanistic Architecture Design}
\label{subsec:mad}
We evaluate RWKV-7 on the Mechanistic Architecture Design (MAD) benchmark \citep{poli2024mechanisticdesignscalinghybrid}, a suite of synthetic token manipulation tasks designed to probe architectural capabilities in sequence modeling, as shown in Table \ref{tab:mad_results}.

\begin{table}[ht]
\footnotesize
\centering
\begin{tabular}{lccccccr}
\toprule
Model & Compress & Fuzzy & In-Context & Memorize & Noisy & Selective & Avg \\
 & & Recall & Recall & & Recall & Copy & \\
\midrule
RWKV-7 & 44.5 & \textbf{43.2} & \textbf{100} & 89.1 & \textbf{100} & 98.8 & \textbf{79.3} \\
\midrule
Transformer & 51.6 & 29.8 & 94.1 & 85.2 & 86.8 & 99.6 & 74.5 \\
Multihead Hyena & 44.8 & 14.4 & 99.0 & 89.4 & 98.6 & 93.0 & 73.2 \\
DeltaNet & 42.2 & 35.7 & \textbf{100} & 52.8 & \textbf{100} & \textbf{100} & 71.8 \\
Mamba & \textbf{52.7} & 6.7 & 90.4 & \textbf{89.5} & 90.1 & 86.3 & 69.3 \\
Hyena & 45.2 & 7.9 & 81.7 & 89.5 & 78.8 & 93.1 & 66.0 \\
GLA & 38.8 & 6.9 & 80.8 & 63.3 & 81.6 & 88.6 & 60.0 \\
\bottomrule
\multicolumn{8}{l}{Results for comparison models from \cite{yang2024_deltanet}}
\end{tabular}
\caption{Results on the MAD benchmark}
\label{tab:mad_results}
\end{table}

RWKV-7 achieves the highest average score across all six tasks, outperforming previous architectures. It demonstrates perfect accuracy on In-Context and Noisy Recall tasks, matching DeltaNet while setting a new state-of-the-art for Fuzzy Recall. RWKV-7 also shows strong performance in memorization and selective copying, suggesting effective combination of attention-based and recurrent model strengths.

\subsection{Long Context Experiments}\label{subsec:long_context_experiments}


To evaluate the ability of RWKV models to retain information over long sequences, we measured loss versus sequence position (we select tokens in range $[L/2-16384, L/2+16384)$ for document length $L$) on the PG19 test set \citep{rae2019compressivetransformerslongrangesequence} for two types of RWKV7 models and their predecessors trained on either The Pile dataset or World dataset. Despite sharing the same architecture and being pretrained on 4k context windows, models trained on different datasets exhibited different behaviors. The Pile-trained RWKV7 showed more significant loss reduction on long contexts compared to its predecessors, demonstrating effective long-context extrapolation (see Figure \ref{fig:pile_pg19_loss}). Surprisingly, for RWKV7 trained on the World dataset, when processing contexts longer than 10k, the loss began to show an increasing trend (see Figure \ref{fig:world_pg19_loss}). We speculate this is because the larger dataset and model size created inductive biases that caused overfitting to specific context lengths. Further experiments showed that fine-tuning on long contexts can restore its long context capabilities.

\begin{figure*}[ht!]
    \centering
    \includegraphics[width=\linewidth]{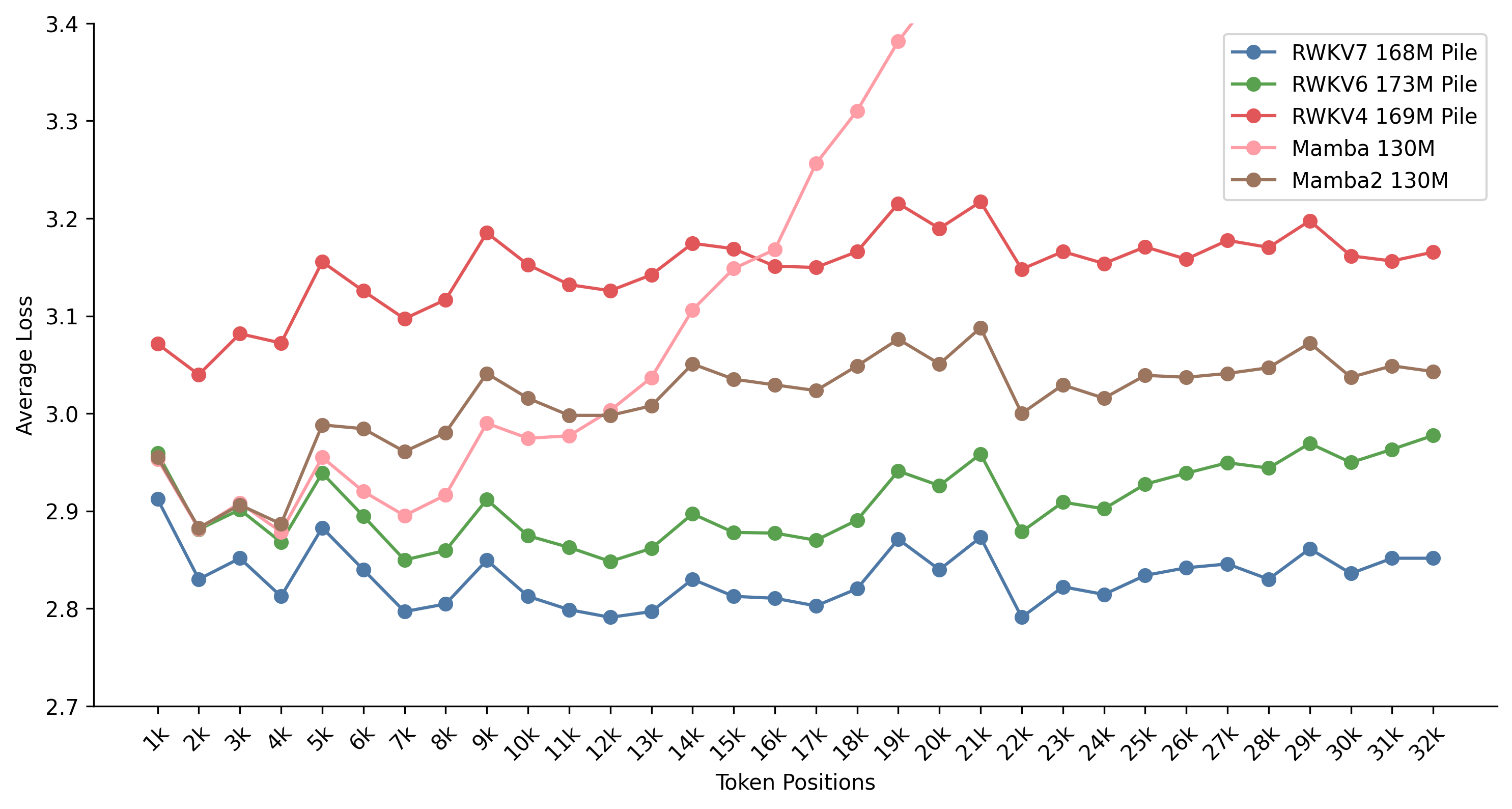}
    \caption{PG19 loss versus sequence position for RWKV and Mamba models trained on The Pile datasets.}
    \label{fig:pile_pg19_loss}
\end{figure*}

\begin{figure*}[ht!]
    \centering
    \includegraphics[width=\linewidth]{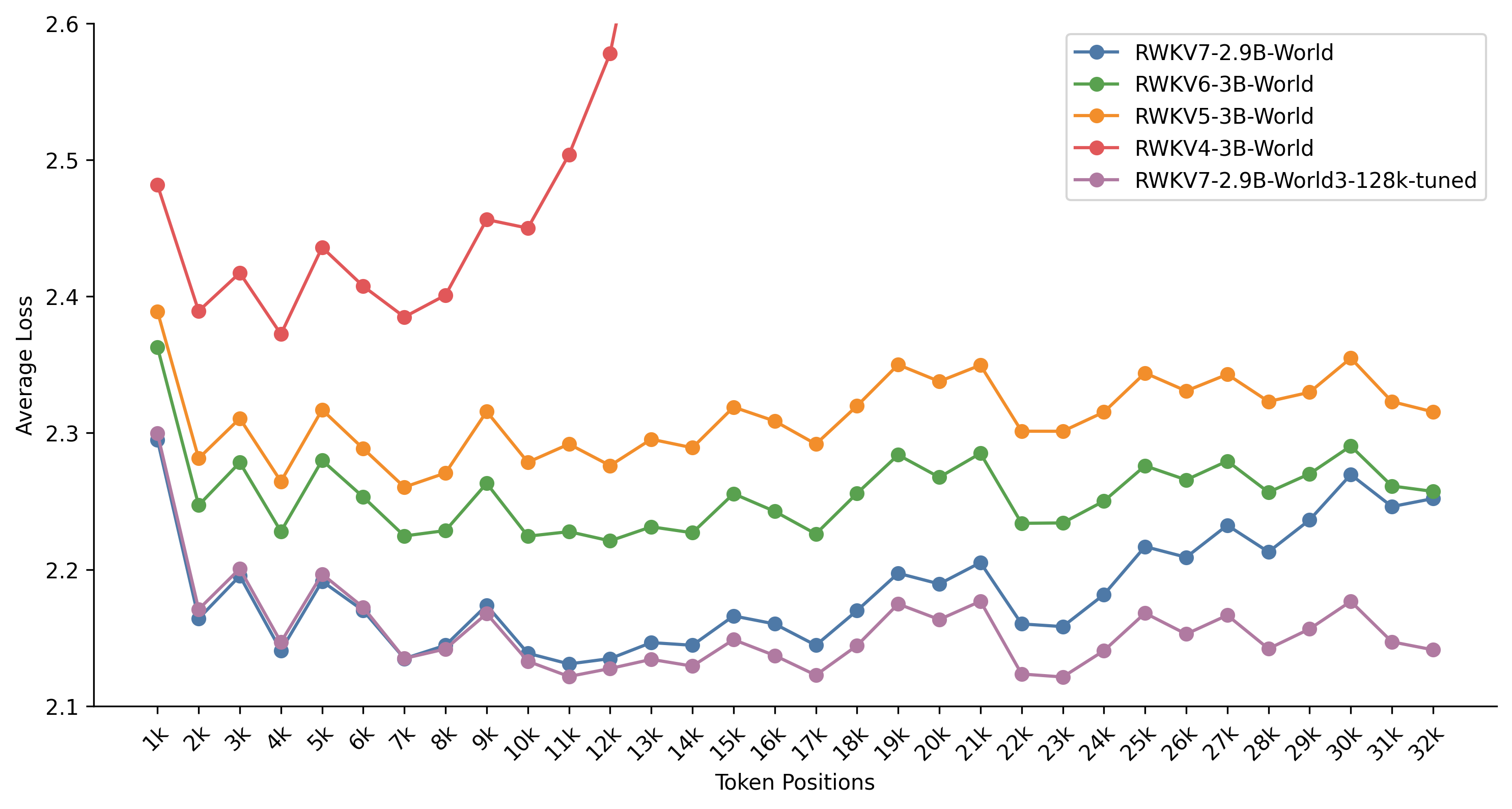}
    \caption{PG19 loss versus sequence position for RWKV7 models and predecessors trained on the World dataset.}
    \label{fig:world_pg19_loss}
\end{figure*}


To further test RWKV-7 long-context retrieval abilities, we conduct a pass-key retrieval evaluation following the approach of \citep{chen2024longloraefficientfinetuninglongcontext} and plot the results in Figure \ref{fig:NIAH}. In this evaluation, a single sentence is repeated multiple times within a long context window, with a key phrase embedded at different positions. RWKV7-World3-1.5B achieves perfect accuracy up to a context length of 19600 tokens but exhibits degradation beyond 20600 tokens. The larger RWKV7-World3-2.9B extends perfect retrieval up to 35000 tokens, highlighting the benefits of scaling. However, performance begins to degrade beyond this point.

\begin{figure*}[ht!]
    \centering
    \begin{subfigure}[b]{0.48\textwidth}
        \includegraphics[width=\textwidth]{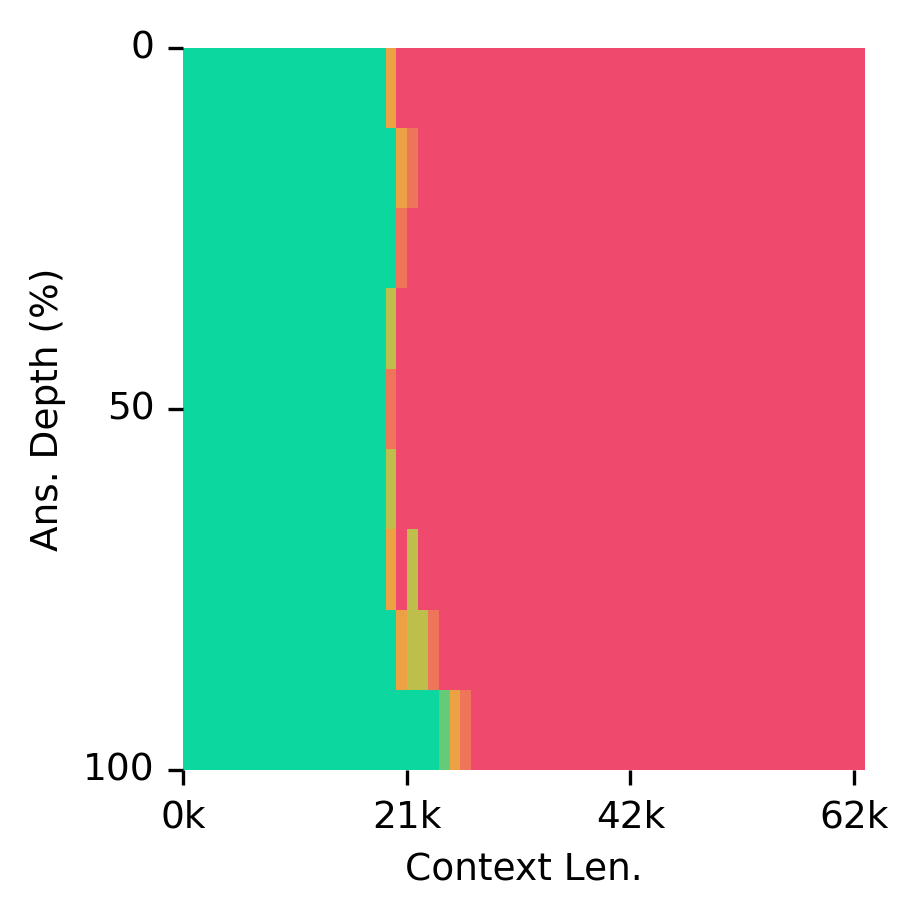}
        \caption{1.5B}
        \label{fig:NIAH_subfig1}
    \end{subfigure}
    \hfill
    \begin{subfigure}[b]{0.48\textwidth}
        \includegraphics[width=\textwidth]{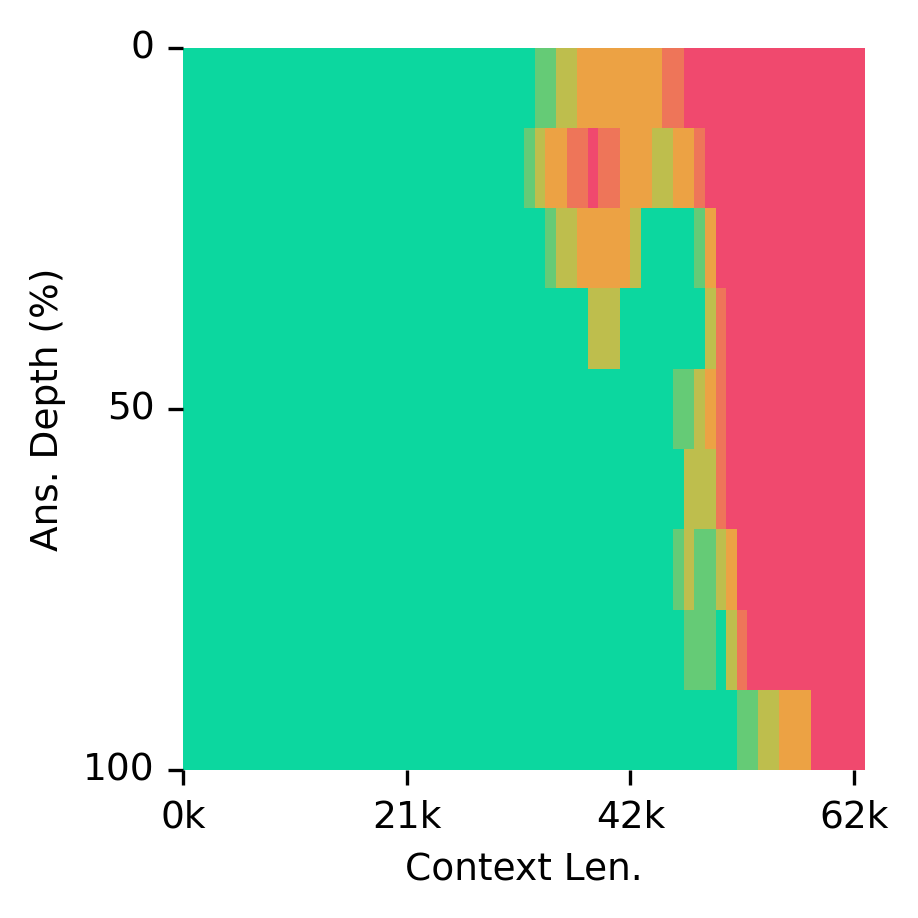}
        \caption{3B}
        \label{fig:NIAH_subfig2}
    \end{subfigure}

    \vspace{1em}

    \begin{subfigure}[b]{0.48\textwidth}
        \includegraphics[width=\textwidth]{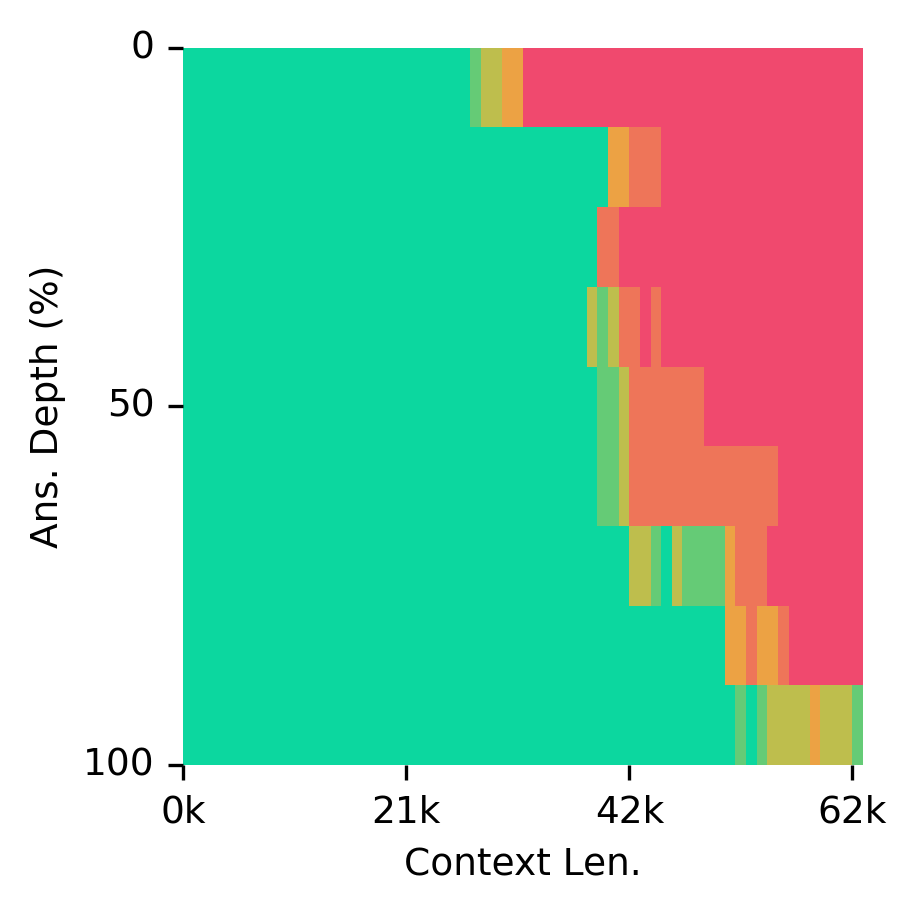}
        \caption{1.5B extended}
        \label{fig:NIAH_subfig3}
    \end{subfigure}
    \hfill
    \begin{subfigure}[b]{0.48\textwidth}
        \includegraphics[width=\textwidth]{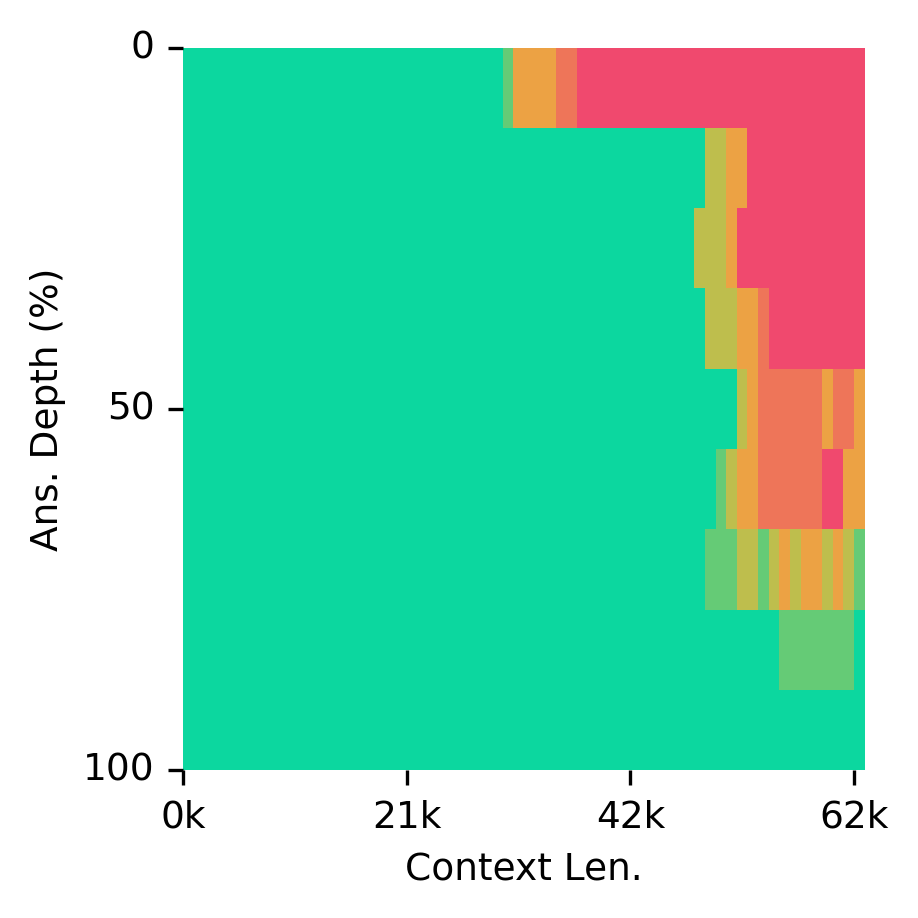}
        \caption{3B extended}
        \label{fig:NIAH_subfig4}
    \end{subfigure}

    \caption{RWKV7-World3 pass-key retrieval evaluation}
    \label{fig:NIAH}
\end{figure*}

To explore potential improvements, we fine-tuned RWKV7-World3-1.5B and RWKV7-World3-2.9B on packed training sequences of length 128k tokens from a specially constructed dataset, which leads to further improvements in retrieval accuracy. With this fine-tuning, RWKV-7 (1.5B) reliably retrieves key phrases up to 29k tokens, and degradation is observed only around 40k tokens. RWKV-7 (2.9B) reliably retrieves the pass key up to 30k tokens, and degrades around 50k tokens.

Our context length extension dataset is comprised of both public and custom sources listed in Table \ref{tab:ctxlen_dataset}. We employed a document length-based weighting scheme to prioritize longer contexts during training. To approximate document lengths of 128,000 tokens, we used character counts of 512,000. Documents of less than 32,768 characters were assigned a weight of 1.0, while longer documents were assigned linearly increasing weights between 2.0 and 3.0, with a cap of 3.0 beyond 512,000. This method increases the inclusion of longer documents to bolster the model's handling of extended contexts while retaining shorter documents for diversity.

\begin{table}[ht]
\setlength\extrarowheight{-5pt}
\centering
\begin{tabular}{llr}
\toprule
\textbf{Dataset} & \textbf{Type} & \textbf{Amount} \\
\midrule
\href{https://huggingface.co/datasets/mlfoundations/dclm-baseline-1.0}{dclm-baseline-1.0}  & Public     & 25\% \\
\href{https://huggingface.co/datasets/HuggingFaceFW/fineweb-edu}{fineweb-edu}        & Public     & 15\% \\
\href{https://huggingface.co/datasets/HuggingFaceFW/fineweb}{fineweb}            & Public     & 5\%  \\
\href{https://huggingface.co/datasets/codeparrot/github-code}{codeparrot/github-code}           & Public     & 10\% \\
\href{https://huggingface.co/datasets/recursal/arXiv-CC0-v0.5}{arXiv-CC0-v0.5}                            & Custom     & 10\% \\
\href{https://huggingface.co/datasets/recursal/SuperWikiNEXT-32B}{SuperWikiNEXT-32B}                        & Custom     & 10\% \\
public domain books              & Custom     & 15\% \\
\href{https://huggingface.co/datasets/bigcode/the-stack}{the-stack} (filtered)             & Custom     & 10\% \\
\bottomrule
\end{tabular}
\caption{Context Length Extension Dataset Components}
\label{tab:ctxlen_dataset}
\end{table}

\subsection{Evaluating State Tracking Using Group Multiplication}
\label{subsec:evaluate_state_tracking}

We adopt the experimental setting from \cite{illusionstate_2024_merril} to evaluate the state-tracking capabilities of RWKV7 in comparison to Transformer, Mamba, S4, and classical RNN models. Given a sequence \( g_0, g_1, g_2, \dots, g_n \) drawn from \( A_5 \), \( A_4 \times \mathbb{Z}_5 \), or \( \mathbb{Z}_{60} \), each step \( i \) is labeled with the cumulative product of the first \( i \) elements.  

We plot the minimum number of layers required to achieve over 95\% validation accuracy on group multiplication tasks, as a function of sequence length and group structure. The results are shown in Figure~\ref{fig:rwkv7-state-tracking}. Our findings indicate that RWKV-7 exhibits stronger state-tracking capabilities than Transformers, Mamba, and S4, though slightly weaker than classical RNNs. Figure~\ref{fig:rwkv7-state-tracking} also aligns with our theory from Appendix \ref{sec:reg_lang}, which predicts that RWKV-7 can perform state tracking and recognize any regular language with a constant number of layers. RWKV-7 has no expressivity advantage for state tracking compared to classical RNNs, which can recognize any regular language in a single layer. However, classical RNNs, while being theoretically expressive, typically suffer from gradient vanishing and memorization problems \citep{zucchet2024recurrentneuralnetworksvanishing} and cannot be parallelized efficiently, unlike RWKV-7.

\begin{figure*}[ht!]
    \centering
    \includegraphics[width=\linewidth]{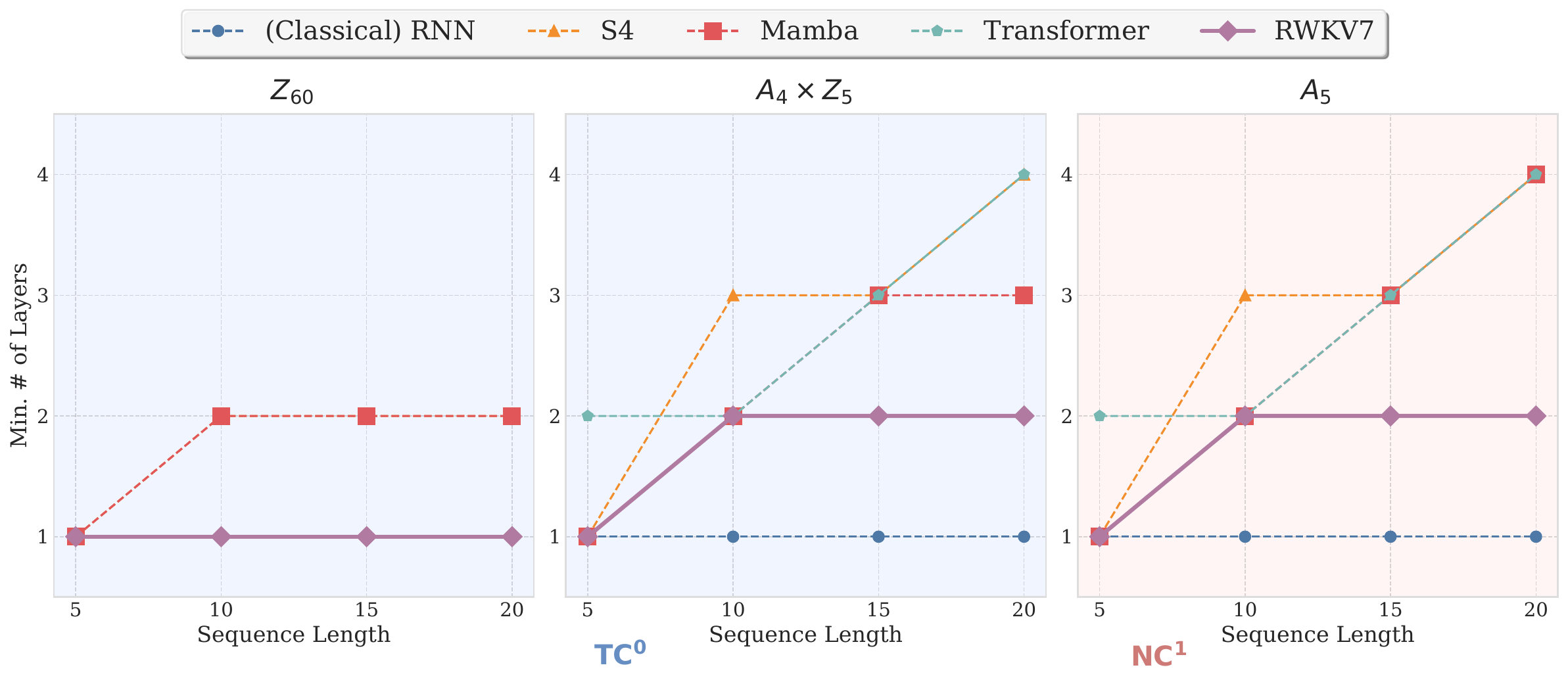}
    \caption{Minimum number of layers (lower is better) required to attain > 95\% validation accuracy on group multiplication problems by sequence length and group. }
    \label{fig:rwkv7-state-tracking}
\end{figure*}

\section{Speed and Memory Usage}
\label{sec:Speed and Memory Usage}

\begin{figure}
    \centering
    \includegraphics[width=\linewidth]{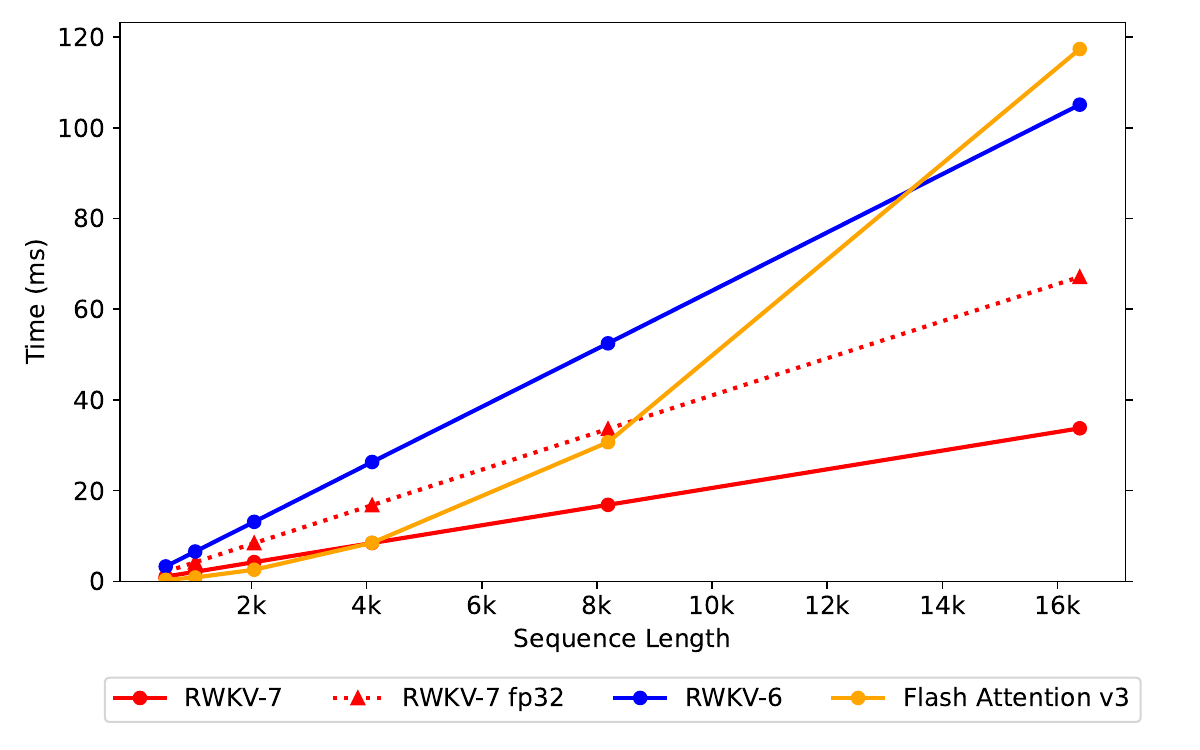}
    \caption{Time vs. Sequence Length (H100)}
    \label{fig:H100_speed_comparison}
\end{figure}

We compare the training speed and memory usage of the RWKV-7 attention-like kernel with the RWKV-6 kernel and Flash Attention v3 \citep{shah2024flashattention}. The "RWKV-7" kernel accelerates bfloat16 matrix multiplications with modern CUDA instructions. We also include the "RWKV-7 fp32" kernel, which is simpler and performs all its internal calculations using float32. Although the bfloat16 kernel is faster, to maximize precision, the RWKV-7 fp32 kernel was used to train the RWKV-7 World models.

Our CUDA kernels are tuned for head dimension 64, as used in the RWKV-7-World models. The kernels still perform well for head dimension 128, but their efficiency drops off at larger head dimensions. There exist other RWKV-7 implementations which focus on head dimensions greater than 128. A key example is the Flash Linear Attention library \citep{yang2024fla}, which offers a Triton-based implementation designed for these larger configurations.

\paragraph{Speed}
In Figure \ref{fig:H100_speed_comparison}, we time the forward + backward pass of each kernel for batch size 8, head dimension 64 and model dimension 4096 (64 wkv heads) on an H100 SXM GPU, for varying sequence lengths. Although Flash Attention v3 is heavily optimized for the H100 GPU, it scales quadratically with sequence length, while the RWKV models scale linearly. This makes the RWKV models faster than attention for large sequence lengths. Furthermore, the optimized RWKV-7 kernel is about three times faster than the official RWKV-6 kernel.

The forward pass of RWKV-7 is about twice as fast as the backward pass. For inference, the forward pass does not need to store the wkv state, making it faster. For example, for sequence length 16k, the forward pass without storing state takes 7.9 ms, while the forward pass with storing state takes 11.2 ms, the backward pass takes 22.5 ms, and the Flash Attention v3 forward pass takes 33.9 ms.

\paragraph{Memory}
The peak training memory usages of the tested models are well described by the formulas derived below. For example, the runs in Figure \ref{fig:H100_speed_comparison} required peak memory within 2\% of the estimates.

In the tested kernels, the memory required per stored variable (e.g. q, k, or v in attention) is $$\text{batch size} \times \text{model dimension} \times \text{sequence length} \times \text{2 bytes for bfloat16}.$$ For sequence length 1024, this is 64MB per variable. To calculate memory usage, we may use Flash Attention v3: 10 variables, RWKV-6: 10 variables, RWKV-7: 18 variables, RWKV-7 fp32: 24 variables.

Flash Attention v3 requires 4 variables q, k, v and output for the forward pass, and the corresponding 4 gradients. Finally, the backward pass uses a temporary variable to accumulate the gradient of q in float32, yielding 2 variables worth of addition memory, for a total of 4+4+2 = 10.

RWKV-6 requires 5 variables r, w, k, v, and output for the forward pass and also the corresponding 5 gradients for the backward pass, for a total of 10.

RWKV-7 uses 7 variables in the forward pass (r, w, k, v, $-\hat\kappa$, $\hat\kappa\odot a$, and output), and the corresponding 7 gradients. Additionally, it stores the wkv state every 16 timesteps. At head size 64, the state contributes the equivalent of 4 variables, for a total of 18. "RWKV-7 fp32" has the same 14 forward variables and gradients, but uses more memory to store the states in float32, for a total of 24 variable equivalents.

Memory usage is constant for single token inference and follows the formulas above, minus the gradients and state storage. Pre-fill can easily be accomplished in a chunked manner, with memory usage growing linearly with regard to chunk size. This allows an easy trade-off for fast parallelized pre-fill, with user selectable maximum memory usage.

\section{Multimodal Experiments}
\label{sec:multimodal_experiments}

In this section, we explore the capabilities of Goose when extended to handle multimodal tasks, where the model processes and integrates textual inputs with inputs from a different domain.


\paragraph{RWKV for Image Understanding}
\label{sec:image-modelling}

\begin{figure*}[ht!]
    \centering
    \includegraphics[width=\linewidth]{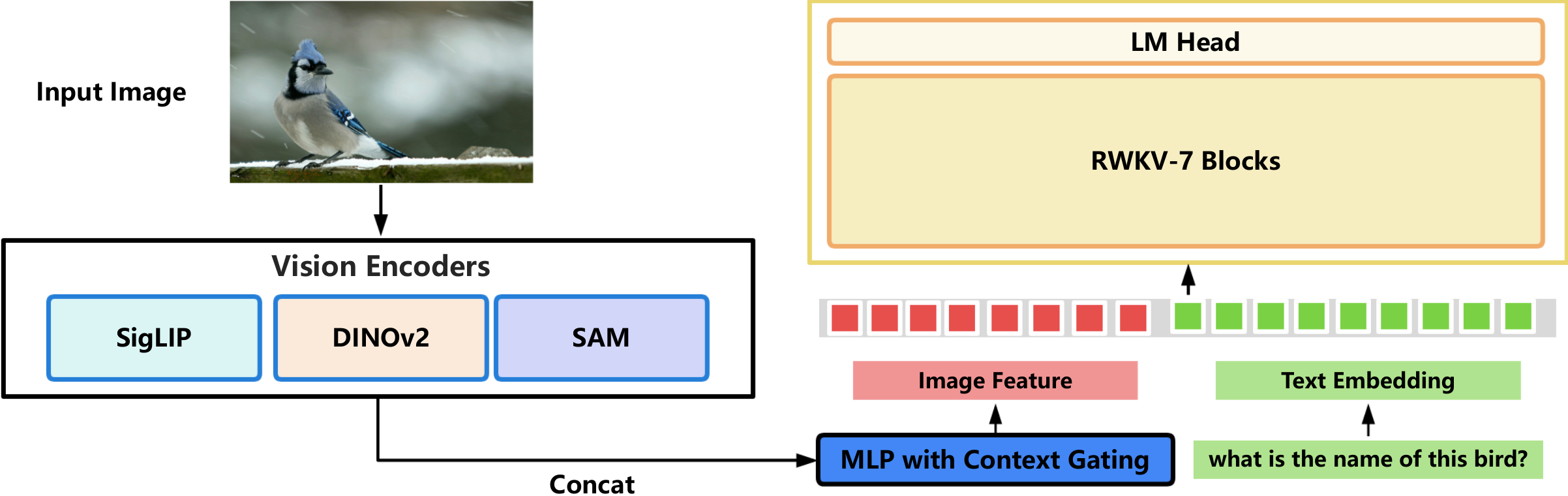}
    \caption{The architecture of VisualRWKV-7. The input image is processed by three vision encoders, and the obtained features are concatenated. Afterward, they are projected through an MLP with context gating to align with the dimensions of the RWKV-7 block. Finally, the image features are concatenated with the text embeddings and fed into the RWKV-7 LLM.}
    \label{fig:visualrwkv7-arch}
\end{figure*}

To demonstrate the modeling capabilities of RWKV-7, we constructed VisualRWKV-7~(Figure\ref{fig:visualrwkv7-arch}), a visual language model based on the RWKV-7 block, to evaluate the image understanding capabilities of RWKV-7.
VisualRWKV-6~\citep{Hou2024VisualRWKVER} used the CLIP encoder, which focused on processing low-resolution images and achieved good results. 
VisualRWKV-7 replaces the CLIP encoder with SigLIP and DINO visual encoders, and introduced a new high-resolution SAM vision encoder, which enhances the model’s supported resolution to 1024 x 1024.

\begin{table*}[htb]
\centering
\scalebox{1.0}{
\begin{tabular}{l| l l| c c c c}
\toprule
Method & Vision Encoder & LLM & VQA & SQA & TQA & GQA \\
\midrule
VisualRWKV-6 & SigLIP+DINOv2+SAM & RWKV6-1.6B & 73.6 & 57.0 & 48.7 & 58.2\\
VisualRWKV-6 & SigLIP+DINOv2+SAM & RWKV6-3.1B & 79.1 & 62.9 & 52.7 & 61.0\\
\midrule
VisualRWKV-7 & SigLIP+DINOv2+SAM &RWKV7-0.1B & 75.2 & 50.6 & 37.9 & 59.9\\
VisualRWKV-7 & SigLIP+DINOv2+SAM &RWKV7-0.4B & 77.9 & 55.0 & 41.1 & 62.3\\
VisualRWKV-7 & SigLIP+DINOv2+SAM &RWKV7-1.5B & 79.8& 59.7 & 49.5 & 63.2\\
VisualRWKV-7 & SigLIP+DINOv2+SAM &RWKV7-2.9B & \textbf{80.5} & \textbf{63.4} & \textbf{58.0} & \textbf{63.7}\\
\bottomrule
\end{tabular}
}
\caption{A comparison of VisualRWKV-7 to other Visual Language Models across 4 distinct benchmarks. We evaluate these models on benchmarks: GQA\citep{Hudson2019GQAAN}, SQA\citep{Lu2022LearnTE}, TQA\citep{singh2019vqa} and VQA\citep{Li2023EvaluatingOH}.}
\label{tab:visualrwkv7_results}
\end{table*}

The experimental results of VisualRWKV-7 are shown in Table~\ref{tab:visualrwkv7_results}. The vision encoders used in both VisualRWKV-7 and VisualRWKV-6 are identical, and the training data remains consistent, aligned with the training data of LLaVA-1.5. The first stage consists of 558k alignment data, while the second stage includes 665k SFT data.

VisualRWKV-7 0.1B and 0.4B outperform VisualRWKV-6 1.6B on the in-domain benchmarks VQAv2 and GQA and rapidly approach VisualRWKV-6 1.6B on two other benchmarks. The experimental results are highly compelling. With only 1/4 of the parameters (1.6B vs. 0.4B), VisualRWKV-7 surpasses VisualRWKV-6 on the VQAv2 and GQA benchmarks, demonstrating the powerful modeling capabilities of RWKV-7.

On the out-of-domain benchmark SQA, VisualRWKV-7 2.9B also outperforms VisualRWKV-6 3.1B, indicating that VisualRWKV-7 possesses strong generalization ability. In the TextQA (TQA) benchmark, which assesses a model’s associative recall, VisualRWKV-7 2.9B achieves a 5.3-point improvement over VisualRWKV-6 3.1B, further proving its superior associative recall capabilities.

\paragraph{RWKV for Audio Modeling}
\label{sec:audio-modeling}
To investigate the effectiveness of RWKV-7 for audio modeling, we introduce AudioRWKV-7, a novel adaptation of RWKV-7 for audio embedding analysis. We use a bi-directional modification to RWKV-7, similar to the modification used by \citet{duan2024visionrwkv}. We employ this approach to interpret and process complex, high-dimensional spectrogram features. To further capture acoustic and temporal characteristics, an mel-spectrogram is divided into patch tokens using a Patch-Embed CNN with a kernel size of $(P \times P)$ and sequentially fed into the model. The width and height of an audio mel-spectrogram represent the time and frequency bins, respectively. Typically, the time dimension is significantly longer than the frequency dimension. To effectively capture relationships among frequency bins within the same time frame, the mel-spectrogram is first segmented into patch windows $w_1, w_2, ..., w_n$, followed by further division of patches within each window. The token sequence follows the order: \textit{time} $\rightarrow$ \textit{frequency} $\rightarrow$ \textit{window}. This arrangement ensures that patches corresponding to different frequency bins within the same time frame are positioned adjacent to each other in the input sequence.

We evaluated the performance of AudioRWKV-7 across multiple model scales and architectures, using the AudioSet dataset \citep{gemmeke2017audio}. Detailed experimental outcomes are presented in Table \ref{tab:audiorwkv_results}. 
From the results we find that our model achieves comparable performance with a much smaller parameter count when compared with CNN, Transformer and Mamba based architectures, and exceeds the performance of AudioRWKV-6. These findings demonstrate the robustness and versatility of AudioRWKV. Note that to ensure a fair comparison, we retrained AudioRWKV-6 without ensembling models with different patch settings, which accounts for the difference in results from \citep{peng2024eaglefinchrwkvmatrixvalued}.  

\begin{table*}[!h]
 \centering
 \begin{tabular}{cccc}

\hline
Model                                                       & \#Parameters & Architecture & mAP $\uparrow$   \\ \hline
DeepRes \citep{ford2019deep}          & 26M          & CNN          & 0.392  \\
HST-AT                                                      & 88.5M        & Transformer  & 0.433* \\
HST-AT pretrained\citep{chen2022hts} & 88.5M        & Transformer  & 0.429* \\
MambaOut\citep{yu2024mambaout}       & 101.3M       & Mamba        & 0.397  \\ 
AudioRWKV-6                         & 8.9M         &RWKV6        & 0.381 \\
AudioRWKV-6                         & 19.8M         &RWKV6        & 0.426   \\ \hline
AudioRWKV-7                          & 8.9M         & RWKV7        & 0.392  \\
AudioRWKV-7                        & 19.8M        & RWKV7        & 0.431  \\ \hline
\end{tabular}
 \caption{A comparison of mean Average Precision (mAP) among AudioRWKV7 and other baselines on AudioSet dataset. HST-AT pretrained is a variation that uses vision transformer to initialize the weights. *Results reproduced by ourselves. }
\label{tab:audiorwkv_results}
\end{table*}

\section{Conclusions} \label{conclusions}
We introduced RWKV-7, a novel RNN architecture that pushes the boundaries of recurrent neural networks to new heights. RWKV-7 achieves state-of-the-art performance for its size across a wide range of benchmarks, demonstrating its potential to rival even highly optimized models such as Qwen2.5 despite being trained on many fewer tokens. As an RNN, RWKV-7 maintains high parameter efficiency, linear time complexity, and constant memory usage, offering a compelling alternative to traditional Transformer-based architectures.

\subsection{Limitations}
Despite its strengths, the RWKV-7 architecture and models face certain limitations yet to be mitigated in future work.

\paragraph{Numerical Precision.} We observed that some operators, particularly the WKV7 kernel, are sensitive to the numerical precision of the implementation. This highlights the need for careful handling of numerical precision during model deployment. We also observed differences in training dynamics when using different kernels, which implies that the correct handling of precision while calculating and applying state updates is of utmost importance in this architecture.

\paragraph{Lack of Instruction Tuning and Alignment.} All RWKV-7 models presented in this work are pretrained base models and have not undergone the phase of Supervised Fine-Tuning (SFT) for instruction following nor alignment with human preferences (RLHF). Future efforts should focus on incorporating these capabilities to enhance the model's usability in real-world applications.

\paragraph{Prompt Sensitivity.} We found that the absence of the special token \texttt{<|endoftext|>} results in degraded performance of RWKV-7 models, e.g. inability to remember the first token of the input. See Appendix \ref{sec:initial_sensitivity} for details.

\paragraph{Compute Resources.} Due to computational budget constraints, our training was limited on at most $12\times 8=96$ Nvidia H800 GPUs. This falls short of the resources required for recent large-scale training efforts, such as DeepSeek-V3 \citep{deepseekai2025deepseekv3technicalreport}. Additionally, we are forced to continue training from pre-existing checkpoints of earlier RWKV architectures and therefore re-use some parts of our dataset. This may limit the capabilities of our models versus pre-training from scratch. Scaling up RWKV-7 to larger sizes and datasets will require additional computational resources.

\subsection{Future Work}
In addition to training larger RWKV-7 models with more tokens in the future, we also aim to explore several promising directions to further enhance the architecture and its capabilities.

\paragraph{Speedup Techniques.} A variety of speed optimization techniques were highlighted in the technical report of DeepSeek-V3 \citep{deepseekai2025deepseekv3technicalreport}, including Dual Pipelining Mechanism, Mixture-of-Experts, Multi-Token Prediction, and FP8 Training. We are aware that many of these techniques are orthogonal to RWKV-7's architectural optimizations, therefore could be integrated to further accelerate training in later RWKV models. However, RWKV-7, like its predecessors, has been trained completely without pipeline parallelism. We also noticed that there is room for speed optimization of RWKV-7 kernels and operators. We will explore both kernel-level optimizations and distributed training strategies in the future.

\paragraph{Incorporating Chain-of-Thought Reasoning.} We believe that RWKV-7, as a linear RNN, is well-suited for efficient Chain-of-Thought reasoning \citep{wei2022chain}. However, this capability has been barely explored due to the lack of suitable reinforcement learning pipelines. In future work, we plan to incorporate deep thinking abilities into RWKV-7, enabling it to excel in tasks requiring multi-step logical reasoning and complex problem-solving.

\subsubsection*{Acknowledgments}
We extend our gratitude to Shenzhen Yuanshi Intelligent Co. Ltd. and Shanghai Yuanwo Intelligent Co. Ltd. for providing computational resources and their dedication to promoting and commercializing RWKV. We thank Featherless AI for their extensive experimentation with the RWKV architecture and their contributions to this paper. We are grateful to the members of the RWKV and EleutherAI Discord communities for their collaborative efforts in extending the applicability of RWKV to diverse domains. We extend a special thank you to Stella Biderman, Songlin Yang and Yu Zhang.


\newpage

\bibliography{main}

\begin{thebibliography}{99}
\providecommand{\natexlab}[1]{#1}
\providecommand{\url}[1]{\texttt{#1}}
\expandafter\ifx\csname urlstyle\endcsname\relax
  \providecommand{\doi}[1]{doi: #1}\else
  \providecommand{\doi}{doi: \begingroup \urlstyle{rm}\Url}\fi

\bibitem[Allal et~al.(2025)Allal, Lozhkov, Bakouch, Blázquez, Penedo, Tunstall, Marafioti, Kydlíček, Lajarín, Srivastav, Lochner, Fahlgren, Nguyen, Fourrier, Burtenshaw, Larcher, Zhao, Zakka, Morlon, Raffel, von Werra, and Wolf]{allal2025smollm2smolgoesbig}
Loubna~Ben Allal, Anton Lozhkov, Elie Bakouch, Gabriel~Martín Blázquez, Guilherme Penedo, Lewis Tunstall, Andrés Marafioti, Hynek Kydlíček, Agustín~Piqueres Lajarín, Vaibhav Srivastav, Joshua Lochner, Caleb Fahlgren, Xuan-Son Nguyen, Clémentine Fourrier, Ben Burtenshaw, Hugo Larcher, Haojun Zhao, Cyril Zakka, Mathieu Morlon, Colin Raffel, Leandro von Werra, and Thomas Wolf.
\newblock Smollm2: When smol goes big -- data-centric training of a small language model, 2025.
\newblock URL \url{https://arxiv.org/abs/2502.02737}.

\bibitem[Arora et~al.(2023)Arora, Eyuboglu, Timalsina, Johnson, Poli, Zou, Rudra, and Re]{arora2023zoology}
Simran Arora, Sabri Eyuboglu, Aman Timalsina, Isys Johnson, Michael Poli, James Zou, Atri Rudra, and Christopher Re.
\newblock Zoology: Measuring and improving recall in efficient language models, 2023.

\bibitem[Azerbayev et~al.(2024)Azerbayev, Schoelkopf, Paster, Santos, McAleer, Jiang, Deng, Biderman, and Welleck]{azerbayev2024llemma}
Zhangir Azerbayev, Hailey Schoelkopf, Keiran Paster, Marco~Dos Santos, Stephen~Marcus McAleer, Albert~Q. Jiang, Jia Deng, Stella Biderman, and Sean Welleck.
\newblock Llemma: An open language model for mathematics.
\newblock In \emph{The Twelfth International Conference on Learning Representations}, 2024.
\newblock URL \url{https://openreview.net/forum?id=4WnqRR915j}.

\bibitem[Barrington(1989)]{barrington1989bounded}
David~A. Barrington.
\newblock Bounded-width polynomial-size branching programs recognize exactly those languages in nc1.
\newblock \emph{Journal of Computer and System Sciences}, 38\penalty0 (1):\penalty0 150--164, 1989.
\newblock URL \url{https://www.sciencedirect.com/science/article/pii/0022000089900378}.

\bibitem[Behrouz et~al.(2024)Behrouz, Zhong, and Mirrokni]{behrouz2024titans}
Ali Behrouz, Peilin Zhong, and Vahab Mirrokni.
\newblock Titans: Learning to memorize at test time, 2024.

\bibitem[Ben~Allal et~al.(2024{\natexlab{a}})Ben~Allal, Lozhkov, Penedo, Wolf, and von Werra]{benallal2024cosmopedia}
Loubna Ben~Allal, Anton Lozhkov, Guilherme Penedo, Thomas Wolf, and Leandro von Werra.
\newblock Cosmopedia, February 2024{\natexlab{a}}.
\newblock URL \url{https://huggingface.co/datasets/HuggingFaceTB/cosmopedia}.

\bibitem[Ben~Allal et~al.(2024{\natexlab{b}})Ben~Allal, Lozhkov, Penedo, Wolf, and von Werra]{benallal2024smollmcorpus}
Loubna Ben~Allal, Anton Lozhkov, Guilherme Penedo, Thomas Wolf, and Leandro von Werra.
\newblock Smollm-corpus, July 2024{\natexlab{b}}.
\newblock URL \url{https://huggingface.co/datasets/HuggingFaceTB/smollm-corpus}.

\bibitem[Bhakthavatsalam et~al.(2021)Bhakthavatsalam, Khashabi, Khot, Mishra, Richardson, Sabharwal, Schoenick, Tafjord, and Clark]{bhakthavatsalam2021think}
Sumithra Bhakthavatsalam, Daniel Khashabi, Tushar Khot, Bhavana~Dalvi Mishra, Kyle Richardson, Ashish Sabharwal, Carissa Schoenick, Oyvind Tafjord, and Peter Clark.
\newblock Think you have solved direct-answer question answering? try arc-da, the direct-answer ai2 reasoning challenge.
\newblock \emph{arXiv preprint arXiv:2102.03315}, 2021.

\bibitem[Bisk et~al.(2020)Bisk, Zellers, Gao, Choi, et~al.]{bisk2020piqa}
Yonatan Bisk, Rowan Zellers, Jianfeng Gao, Yejin Choi, et~al.
\newblock Piqa: Reasoning about physical commonsense in natural language.
\newblock In \emph{Proceedings of the AAAI conference on artificial intelligence}, volume~34, pp.\  7432--7439, 2020.

\bibitem[Black et~al.(2022)Black, Biderman, Hallahan, Anthony, Gao, Golding, He, Leahy, McDonell, Phang, Pieler, Prashanth, Purohit, Reynolds, Tow, Wang, and Weinbach]{black-etal-2022-gpt}
Sidney Black, Stella Biderman, Eric Hallahan, Quentin Anthony, Leo Gao, Laurence Golding, Horace He, Connor Leahy, Kyle McDonell, Jason Phang, Michael Pieler, Usvsn~Sai Prashanth, Shivanshu Purohit, Laria Reynolds, Jonathan Tow, Ben Wang, and Samuel Weinbach.
\newblock {GPT}-{N}eo{X}-20{B}: An open-source autoregressive language model.
\newblock In Angela Fan, Suzana Ilic, Thomas Wolf, and Matthias Gall{\'e} (eds.), \emph{Proceedings of BigScience Episode {\#}5 -- Workshop on Challenges {\&} Perspectives in Creating Large Language Models}, pp.\  95--136, virtual+Dublin, May 2022. Association for Computational Linguistics.
\newblock \doi{10.18653/v1/2022.bigscience-1.9}.
\newblock URL \url{https://aclanthology.org/2022.bigscience-1.9}.

\bibitem[Chen et~al.(2022)Chen, Du, Zhu, Ma, Berg-Kirkpatrick, and Dubnov]{chen2022hts}
Ke~Chen, Xingjian Du, Bilei Zhu, Zejun Ma, Taylor Berg-Kirkpatrick, and Shlomo Dubnov.
\newblock Hts-at: A hierarchical token-semantic audio transformer for sound classification and detection.
\newblock In \emph{ICASSP 2022-2022 IEEE International Conference on Acoustics, Speech and Signal Processing (ICASSP)}, pp.\  646--650. IEEE, 2022.

\bibitem[Chen et~al.(2024{\natexlab{a}})Chen, Zhang, Hu, Han, Liu, and Sun]{chen2024stuffedmambastatecollapse}
Yingfa Chen, Xinrong Zhang, Shengding Hu, Xu~Han, Zhiyuan Liu, and Maosong Sun.
\newblock Stuffed mamba: State collapse and state capacity of rnn-based long-context modeling, 2024{\natexlab{a}}.
\newblock URL \url{https://arxiv.org/abs/2410.07145}.

\bibitem[Chen et~al.(2024{\natexlab{b}})Chen, Qian, Tang, Lai, Liu, Han, and Jia]{chen2024longloraefficientfinetuninglongcontext}
Yukang Chen, Shengju Qian, Haotian Tang, Xin Lai, Zhijian Liu, Song Han, and Jiaya Jia.
\newblock Longlora: Efficient fine-tuning of long-context large language models, 2024{\natexlab{b}}.
\newblock URL \url{https://arxiv.org/abs/2309.12307}.

\bibitem[Conneau et~al.(2018)Conneau, Rinott, Lample, Williams, Bowman, Schwenk, and Stoyanov]{conneau2018xnli}
Alexis Conneau, Ruty Rinott, Guillaume Lample, Adina Williams, Samuel Bowman, Holger Schwenk, and Veselin Stoyanov.
\newblock Xnli: Evaluating cross-lingual sentence representations.
\newblock In \emph{Proceedings of the 2018 Conference on Empirical Methods in Natural Language Processing}, pp.\  2475--2485, 2018.

\bibitem[Dao \& Gu(2024)Dao and Gu]{dao2024transformersssmsgeneralizedmodels}
Tri Dao and Albert Gu.
\newblock Transformers are ssms: Generalized models and efficient algorithms through structured state space duality, 2024.
\newblock URL \url{https://arxiv.org/abs/2405.21060}.

\bibitem[{Dao AI Lab}(2023)]{causal-conv1d}
{Dao AI Lab}.
\newblock Causal conv1d.
\newblock \url{https://github.com/Dao-AILab/causal-conv1d}, 2023.
\newblock Accessed: 2025-02-26.

\bibitem[DeepSeek-AI et~al.(2025)DeepSeek-AI, Liu, Feng, Xue, Wang, Wu, Lu, Zhao, Deng, Zhang, Ruan, Dai, Guo, Yang, Chen, Ji, Li, Lin, Dai, Luo, Hao, Chen, Li, Zhang, Bao, Xu, Wang, Zhang, Ding, Xin, Gao, Li, Qu, Cai, Liang, Guo, Ni, Li, Wang, Chen, Chen, Yuan, Qiu, Li, Song, Dong, Hu, Gao, Guan, Huang, Yu, Wang, Zhang, Xu, Xia, Zhao, Wang, Zhang, Li, Wang, Zhang, Zhang, Tang, Li, Tian, Huang, Wang, Zhang, Wang, Zhu, Chen, Du, Chen, Jin, Ge, Zhang, Pan, Wang, Xu, Zhang, Chen, Li, Lu, Zhou, Chen, Wu, Ye, Ye, Ma, Wang, Zhou, Yu, Zhou, Pan, Wang, Yun, Pei, Sun, Xiao, Zeng, Zhao, An, Liu, Liang, Gao, Yu, Zhang, Li, Jin, Wang, Bi, Liu, Wang, Shen, Chen, Zhang, Chen, Nie, Sun, Wang, Cheng, Liu, Xie, Liu, Yu, Song, Shan, Zhou, Yang, Li, Su, Lin, Li, Wang, Wei, Zhu, Zhang, Xu, Xu, Huang, Li, Zhao, Sun, Li, Wang, Yu, Zheng, Zhang, Shi, Xiong, He, Tang, Piao, Wang, Tan, Ma, Liu, Guo, Wu, Ou, Zhu, Wang, Gong, Zou, He, Zha, Xiong, Ma, Yan, Luo, You, Liu, Zhou, Wu, Ren, Ren, Sha, Fu, Xu, Huang, Zhang, Xie, Zhang, Hao,
  Gou, Ma, Yan, Shao, Xu, Wu, Zhang, Li, Gu, Zhu, Liu, Li, Xie, Song, Gao, and Pan]{deepseekai2025deepseekv3technicalreport}
DeepSeek-AI, Aixin Liu, Bei Feng, Bing Xue, Bingxuan Wang, Bochao Wu, Chengda Lu, Chenggang Zhao, Chengqi Deng, Chenyu Zhang, Chong Ruan, Damai Dai, Daya Guo, Dejian Yang, Deli Chen, Dongjie Ji, Erhang Li, Fangyun Lin, Fucong Dai, Fuli Luo, Guangbo Hao, Guanting Chen, Guowei Li, H.~Zhang, Han Bao, Hanwei Xu, Haocheng Wang, Haowei Zhang, Honghui Ding, Huajian Xin, Huazuo Gao, Hui Li, Hui Qu, J.~L. Cai, Jian Liang, Jianzhong Guo, Jiaqi Ni, Jiashi Li, Jiawei Wang, Jin Chen, Jingchang Chen, Jingyang Yuan, Junjie Qiu, Junlong Li, Junxiao Song, Kai Dong, Kai Hu, Kaige Gao, Kang Guan, Kexin Huang, Kuai Yu, Lean Wang, Lecong Zhang, Lei Xu, Leyi Xia, Liang Zhao, Litong Wang, Liyue Zhang, Meng Li, Miaojun Wang, Mingchuan Zhang, Minghua Zhang, Minghui Tang, Mingming Li, Ning Tian, Panpan Huang, Peiyi Wang, Peng Zhang, Qiancheng Wang, Qihao Zhu, Qinyu Chen, Qiushi Du, R.~J. Chen, R.~L. Jin, Ruiqi Ge, Ruisong Zhang, Ruizhe Pan, Runji Wang, Runxin Xu, Ruoyu Zhang, Ruyi Chen, S.~S. Li, Shanghao Lu, Shangyan Zhou, Shanhuang
  Chen, Shaoqing Wu, Shengfeng Ye, Shengfeng Ye, Shirong Ma, Shiyu Wang, Shuang Zhou, Shuiping Yu, Shunfeng Zhou, Shuting Pan, T.~Wang, Tao Yun, Tian Pei, Tianyu Sun, W.~L. Xiao, Wangding Zeng, Wanjia Zhao, Wei An, Wen Liu, Wenfeng Liang, Wenjun Gao, Wenqin Yu, Wentao Zhang, X.~Q. Li, Xiangyue Jin, Xianzu Wang, Xiao Bi, Xiaodong Liu, Xiaohan Wang, Xiaojin Shen, Xiaokang Chen, Xiaokang Zhang, Xiaosha Chen, Xiaotao Nie, Xiaowen Sun, Xiaoxiang Wang, Xin Cheng, Xin Liu, Xin Xie, Xingchao Liu, Xingkai Yu, Xinnan Song, Xinxia Shan, Xinyi Zhou, Xinyu Yang, Xinyuan Li, Xuecheng Su, Xuheng Lin, Y.~K. Li, Y.~Q. Wang, Y.~X. Wei, Y.~X. Zhu, Yang Zhang, Yanhong Xu, Yanhong Xu, Yanping Huang, Yao Li, Yao Zhao, Yaofeng Sun, Yaohui Li, Yaohui Wang, Yi~Yu, Yi~Zheng, Yichao Zhang, Yifan Shi, Yiliang Xiong, Ying He, Ying Tang, Yishi Piao, Yisong Wang, Yixuan Tan, Yiyang Ma, Yiyuan Liu, Yongqiang Guo, Yu~Wu, Yuan Ou, Yuchen Zhu, Yuduan Wang, Yue Gong, Yuheng Zou, Yujia He, Yukun Zha, Yunfan Xiong, Yunxian Ma, Yuting Yan, Yuxiang
  Luo, Yuxiang You, Yuxuan Liu, Yuyang Zhou, Z.~F. Wu, Z.~Z. Ren, Zehui Ren, Zhangli Sha, Zhe Fu, Zhean Xu, Zhen Huang, Zhen Zhang, Zhenda Xie, Zhengyan Zhang, Zhewen Hao, Zhibin Gou, Zhicheng Ma, Zhigang Yan, Zhihong Shao, Zhipeng Xu, Zhiyu Wu, Zhongyu Zhang, Zhuoshu Li, Zihui Gu, Zijia Zhu, Zijun Liu, Zilin Li, Ziwei Xie, Ziyang Song, Ziyi Gao, and Zizheng Pan.
\newblock Deepseek-v3 technical report, 2025.
\newblock URL \url{https://arxiv.org/abs/2412.19437}.

\bibitem[Delétang et~al.(2024)Delétang, Ruoss, Duquenne, Catt, Genewein, Mattern, Grau-Moya, Wenliang, Aitchison, Orseau, Hutter, and Veness]{delétang2024languagemodelingcompression}
Grégoire Delétang, Anian Ruoss, Paul-Ambroise Duquenne, Elliot Catt, Tim Genewein, Christopher Mattern, Jordi Grau-Moya, Li~Kevin Wenliang, Matthew Aitchison, Laurent Orseau, Marcus Hutter, and Joel Veness.
\newblock Language modeling is compression, 2024.
\newblock URL \url{https://arxiv.org/abs/2309.10668}.

\bibitem[Duan et~al.(2024)Duan, Wang, Chen, Zhu, Lu, Lu, Qiao, Li, Dai, and Wang]{duan2024visionrwkv}
Yuchen Duan, Weiyun Wang, Zhe Chen, Xizhou Zhu, Lewei Lu, Tong Lu, Yu~Qiao, Hongsheng Li, Jifeng Dai, and Wenhai Wang.
\newblock Vision-rwkv: Efficient and scalable visual perception with rwkv-like architectures.
\newblock \emph{arXiv preprint arXiv:2403.02308}, 2024.

\bibitem[Elhage et~al.(2021)Elhage, Nanda, Olsson, Henighan, Joseph, Mann, Askell, Bai, Chen, Conerly, DasSarma, Drain, Ganguli, Hatfield-Dodds, Hernandez, Jones, Kernion, Lovitt, Ndousse, Amodei, Brown, Clark, Kaplan, McCandlish, and Olah]{elhage2021mathematical}
Nelson Elhage, Neel Nanda, Catherine Olsson, Tom Henighan, Nicholas Joseph, Ben Mann, Amanda Askell, Yuntao Bai, Anna Chen, Tom Conerly, Nova DasSarma, Dawn Drain, Deep Ganguli, Zac Hatfield-Dodds, Danny Hernandez, Andy Jones, Jackson Kernion, Liane Lovitt, Kamal Ndousse, Dario Amodei, Tom Brown, Jack Clark, Jared Kaplan, Sam McCandlish, and Chris Olah.
\newblock A mathematical framework for transformer circuits.
\newblock \emph{Transformer Circuits Thread}, 2021.
\newblock https://transformer-circuits.pub/2021/framework/index.html.

\bibitem[Fan et~al.(2025)Fan, Huang, and He]{fan2025breakinglowrankdilemmalinear}
Qihang Fan, Huaibo Huang, and Ran He.
\newblock Breaking the low-rank dilemma of linear attention, 2025.
\newblock URL \url{https://arxiv.org/abs/2411.07635}.

\bibitem[Ford et~al.(2019)Ford, Tang, Grondin, and Glass]{ford2019deep}
Logan Ford, Hao Tang, Fran{\c{c}}ois Grondin, and James~R Glass.
\newblock A deep residual network for large-scale acoustic scene analysis.
\newblock In \emph{InterSpeech}, pp.\  2568--2572, 2019.

\bibitem[Fu et~al.(2023)Fu, Dao, Saab, Thomas, Rudra, and Re]{fu2023hungry}
Daniel~Y. Fu, Tri Dao, Khaled~K. Saab, Armin~W. Thomas, Atri Rudra, and Christopher Re.
\newblock Hungry hungry hippos: Towards language modeling with state space models, 2023.

\bibitem[Gao et~al.(2020)Gao, Biderman, Black, Golding, Hoppe, Foster, Phang, He, Thite, Nabeshima, Presser, and Leahy]{gao2020pile}
Leo Gao, Stella Biderman, Sid Black, Laurence Golding, Travis Hoppe, Charles Foster, Jason Phang, Horace He, Anish Thite, Noa Nabeshima, Shawn Presser, and Connor Leahy.
\newblock The pile: An 800gb dataset of diverse text for language modeling, 2020.

\bibitem[Gao et~al.(2023)Gao, Tow, Abbasi, Biderman, Black, DiPofi, Foster, Golding, Hsu, Le~Noac'h, Li, McDonell, Muennighoff, Ociepa, Phang, Reynolds, Schoelkopf, Skowron, Sutawika, Tang, Thite, Wang, Wang, and Zou]{gao10256836framework}
Leo Gao, Jonathan Tow, Baber Abbasi, Stella Biderman, Sid Black, Anthony DiPofi, Charles Foster, Laurence Golding, Jeffrey Hsu, Alain Le~Noac'h, Haonan Li, Kyle McDonell, Niklas Muennighoff, Chris Ociepa, Jason Phang, Laria Reynolds, Hailey Schoelkopf, Aviya Skowron, Lintang Sutawika, Eric Tang, Anish Thite, Ben Wang, Kevin Wang, and Andy Zou.
\newblock A framework for few-shot language model evaluation, 12 2023.
\newblock URL \url{https://zenodo.org/records/10256836}.

\bibitem[Gemmeke et~al.(2017)Gemmeke, Ellis, Freedman, Jansen, Lawrence, Moore, Plakal, and Ritter]{gemmeke2017audio}
Jort~F Gemmeke, Daniel~PW Ellis, Dylan Freedman, Aren Jansen, Wade Lawrence, R~Channing Moore, Manoj Plakal, and Marvin Ritter.
\newblock Audio set: An ontology and human-labeled dataset for audio events.
\newblock In \emph{2017 IEEE international conference on acoustics, speech and signal processing (ICASSP)}, pp.\  776--780. IEEE, 2017.

\bibitem[Grattafiori et~al.(2024)Grattafiori, Dubey, Jauhri, Pandey, Kadian, Al-Dahle, Letman, Mathur, Schelten, Vaughan, Yang, Fan, Goyal, Hartshorn, Yang, Mitra, Sravankumar, Korenev, Hinsvark, Rao, Zhang, Rodriguez, Gregerson, Spataru, Roziere, Biron, Tang, Chern, Caucheteux, Nayak, Bi, Marra, McConnell, Keller, Touret, Wu, Wong, Ferrer, Nikolaidis, Allonsius, Song, Pintz, Livshits, Wyatt, Esiobu, Choudhary, Mahajan, Garcia-Olano, Perino, Hupkes, Lakomkin, AlBadawy, Lobanova, Dinan, Smith, Radenovic, Guzmán, Zhang, Synnaeve, Lee, Anderson, Thattai, Nail, Mialon, Pang, Cucurell, Nguyen, Korevaar, Xu, Touvron, Zarov, Ibarra, Kloumann, Misra, Evtimov, Zhang, Copet, Lee, Geffert, Vranes, Park, Mahadeokar, Shah, van~der Linde, Billock, Hong, Lee, Fu, Chi, Huang, Liu, Wang, Yu, Bitton, Spisak, Park, Rocca, Johnstun, Saxe, Jia, Alwala, Prasad, Upasani, Plawiak, Li, Heafield, Stone, El-Arini, Iyer, Malik, Chiu, Bhalla, Lakhotia, Rantala-Yeary, van~der Maaten, Chen, Tan, Jenkins, Martin, Madaan, Malo, Blecher,
  Landzaat, de~Oliveira, Muzzi, Pasupuleti, Singh, Paluri, Kardas, Tsimpoukelli, Oldham, Rita, Pavlova, Kambadur, Lewis, Si, Singh, Hassan, Goyal, Torabi, Bashlykov, Bogoychev, Chatterji, Zhang, Duchenne, Çelebi, Alrassy, Zhang, Li, Vasic, Weng, Bhargava, Dubal, Krishnan, Koura, Xu, He, Dong, Srinivasan, Ganapathy, Calderer, Cabral, Stojnic, Raileanu, Maheswari, Girdhar, Patel, Sauvestre, Polidoro, Sumbaly, Taylor, Silva, Hou, Wang, Hosseini, Chennabasappa, Singh, Bell, Kim, Edunov, Nie, Narang, Raparthy, Shen, Wan, Bhosale, Zhang, Vandenhende, Batra, Whitman, Sootla, Collot, Gururangan, Borodinsky, Herman, Fowler, Sheasha, Georgiou, Scialom, Speckbacher, Mihaylov, Xiao, Karn, Goswami, Gupta, Ramanathan, Kerkez, Gonguet, Do, Vogeti, Albiero, Petrovic, Chu, Xiong, Fu, Meers, Martinet, Wang, Wang, Tan, Xia, Xie, Jia, Wang, Goldschlag, Gaur, Babaei, Wen, Song, Zhang, Li, Mao, Coudert, Yan, Chen, Papakipos, Singh, Srivastava, Jain, Kelsey, Shajnfeld, Gangidi, Victoria, Goldstand, Menon, Sharma, Boesenberg,
  Baevski, Feinstein, Kallet, Sangani, Teo, Yunus, Lupu, Alvarado, Caples, Gu, Ho, Poulton, Ryan, Ramchandani, Dong, Franco, Goyal, Saraf, Chowdhury, Gabriel, Bharambe, Eisenman, Yazdan, James, Maurer, Leonhardi, Huang, Loyd, Paola, Paranjape, Liu, Wu, Ni, Hancock, Wasti, Spence, Stojkovic, Gamido, Montalvo, Parker, Burton, Mejia, Liu, Wang, Kim, Zhou, Hu, Chu, Cai, Tindal, Feichtenhofer, Gao, Civin, Beaty, Kreymer, Li, Adkins, Xu, Testuggine, David, Parikh, Liskovich, Foss, Wang, Le, Holland, Dowling, Jamil, Montgomery, Presani, Hahn, Wood, Le, Brinkman, Arcaute, Dunbar, Smothers, Sun, Kreuk, Tian, Kokkinos, Ozgenel, Caggioni, Kanayet, Seide, Florez, Schwarz, Badeer, Swee, Halpern, Herman, Sizov, Guangyi, Zhang, Lakshminarayanan, Inan, Shojanazeri, Zou, Wang, Zha, Habeeb, Rudolph, Suk, Aspegren, Goldman, Zhan, Damlaj, Molybog, Tufanov, Leontiadis, Veliche, Gat, Weissman, Geboski, Kohli, Lam, Asher, Gaya, Marcus, Tang, Chan, Zhen, Reizenstein, Teboul, Zhong, Jin, Yang, Cummings, Carvill, Shepard, McPhie,
  Torres, Ginsburg, Wang, Wu, U, Saxena, Khandelwal, Zand, Matosich, Veeraraghavan, Michelena, Li, Jagadeesh, Huang, Chawla, Huang, Chen, Garg, A, Silva, Bell, Zhang, Guo, Yu, Moshkovich, Wehrstedt, Khabsa, Avalani, Bhatt, Mankus, Hasson, Lennie, Reso, Groshev, Naumov, Lathi, Keneally, Liu, Seltzer, Valko, Restrepo, Patel, Vyatskov, Samvelyan, Clark, Macey, Wang, Hermoso, Metanat, Rastegari, Bansal, Santhanam, Parks, White, Bawa, Singhal, Egebo, Usunier, Mehta, Laptev, Dong, Cheng, Chernoguz, Hart, Salpekar, Kalinli, Kent, Parekh, Saab, Balaji, Rittner, Bontrager, Roux, Dollar, Zvyagina, Ratanchandani, Yuvraj, Liang, Alao, Rodriguez, Ayub, Murthy, Nayani, Mitra, Parthasarathy, Li, Hogan, Battey, Wang, Howes, Rinott, Mehta, Siby, Bondu, Datta, Chugh, Hunt, Dhillon, Sidorov, Pan, Mahajan, Verma, Yamamoto, Ramaswamy, Lindsay, Lindsay, Feng, Lin, Zha, Patil, Shankar, Zhang, Zhang, Wang, Agarwal, Sajuyigbe, Chintala, Max, Chen, Kehoe, Satterfield, Govindaprasad, Gupta, Deng, Cho, Virk, Subramanian, Choudhury,
  Goldman, Remez, Glaser, Best, Koehler, Robinson, Li, Zhang, Matthews, Chou, Shaked, Vontimitta, Ajayi, Montanez, Mohan, Kumar, Mangla, Ionescu, Poenaru, Mihailescu, Ivanov, Li, Wang, Jiang, Bouaziz, Constable, Tang, Wu, Wang, Wu, Gao, Kleinman, Chen, Hu, Jia, Qi, Li, Zhang, Zhang, Adi, Nam, Yu, Wang, Zhao, Hao, Qian, Li, He, Rait, DeVito, Rosnbrick, Wen, Yang, Zhao, and Ma]{grattafiori2024llama3herdmodels}
Aaron Grattafiori, Abhimanyu Dubey, Abhinav Jauhri, Abhinav Pandey, Abhishek Kadian, Ahmad Al-Dahle, Aiesha Letman, Akhil Mathur, Alan Schelten, Alex Vaughan, Amy Yang, Angela Fan, Anirudh Goyal, Anthony Hartshorn, Aobo Yang, Archi Mitra, Archie Sravankumar, Artem Korenev, Arthur Hinsvark, Arun Rao, Aston Zhang, Aurelien Rodriguez, Austen Gregerson, Ava Spataru, Baptiste Roziere, Bethany Biron, Binh Tang, Bobbie Chern, Charlotte Caucheteux, Chaya Nayak, Chloe Bi, Chris Marra, Chris McConnell, Christian Keller, Christophe Touret, Chunyang Wu, Corinne Wong, Cristian~Canton Ferrer, Cyrus Nikolaidis, Damien Allonsius, Daniel Song, Danielle Pintz, Danny Livshits, Danny Wyatt, David Esiobu, Dhruv Choudhary, Dhruv Mahajan, Diego Garcia-Olano, Diego Perino, Dieuwke Hupkes, Egor Lakomkin, Ehab AlBadawy, Elina Lobanova, Emily Dinan, Eric~Michael Smith, Filip Radenovic, Francisco Guzmán, Frank Zhang, Gabriel Synnaeve, Gabrielle Lee, Georgia~Lewis Anderson, Govind Thattai, Graeme Nail, Gregoire Mialon, Guan Pang,
  Guillem Cucurell, Hailey Nguyen, Hannah Korevaar, Hu~Xu, Hugo Touvron, Iliyan Zarov, Imanol~Arrieta Ibarra, Isabel Kloumann, Ishan Misra, Ivan Evtimov, Jack Zhang, Jade Copet, Jaewon Lee, Jan Geffert, Jana Vranes, Jason Park, Jay Mahadeokar, Jeet Shah, Jelmer van~der Linde, Jennifer Billock, Jenny Hong, Jenya Lee, Jeremy Fu, Jianfeng Chi, Jianyu Huang, Jiawen Liu, Jie Wang, Jiecao Yu, Joanna Bitton, Joe Spisak, Jongsoo Park, Joseph Rocca, Joshua Johnstun, Joshua Saxe, Junteng Jia, Kalyan~Vasuden Alwala, Karthik Prasad, Kartikeya Upasani, Kate Plawiak, Ke~Li, Kenneth Heafield, Kevin Stone, Khalid El-Arini, Krithika Iyer, Kshitiz Malik, Kuenley Chiu, Kunal Bhalla, Kushal Lakhotia, Lauren Rantala-Yeary, Laurens van~der Maaten, Lawrence Chen, Liang Tan, Liz Jenkins, Louis Martin, Lovish Madaan, Lubo Malo, Lukas Blecher, Lukas Landzaat, Luke de~Oliveira, Madeline Muzzi, Mahesh Pasupuleti, Mannat Singh, Manohar Paluri, Marcin Kardas, Maria Tsimpoukelli, Mathew Oldham, Mathieu Rita, Maya Pavlova, Melanie Kambadur,
  Mike Lewis, Min Si, Mitesh~Kumar Singh, Mona Hassan, Naman Goyal, Narjes Torabi, Nikolay Bashlykov, Nikolay Bogoychev, Niladri Chatterji, Ning Zhang, Olivier Duchenne, Onur Çelebi, Patrick Alrassy, Pengchuan Zhang, Pengwei Li, Petar Vasic, Peter Weng, Prajjwal Bhargava, Pratik Dubal, Praveen Krishnan, Punit~Singh Koura, Puxin Xu, Qing He, Qingxiao Dong, Ragavan Srinivasan, Raj Ganapathy, Ramon Calderer, Ricardo~Silveira Cabral, Robert Stojnic, Roberta Raileanu, Rohan Maheswari, Rohit Girdhar, Rohit Patel, Romain Sauvestre, Ronnie Polidoro, Roshan Sumbaly, Ross Taylor, Ruan Silva, Rui Hou, Rui Wang, Saghar Hosseini, Sahana Chennabasappa, Sanjay Singh, Sean Bell, Seohyun~Sonia Kim, Sergey Edunov, Shaoliang Nie, Sharan Narang, Sharath Raparthy, Sheng Shen, Shengye Wan, Shruti Bhosale, Shun Zhang, Simon Vandenhende, Soumya Batra, Spencer Whitman, Sten Sootla, Stephane Collot, Suchin Gururangan, Sydney Borodinsky, Tamar Herman, Tara Fowler, Tarek Sheasha, Thomas Georgiou, Thomas Scialom, Tobias Speckbacher,
  Todor Mihaylov, Tong Xiao, Ujjwal Karn, Vedanuj Goswami, Vibhor Gupta, Vignesh Ramanathan, Viktor Kerkez, Vincent Gonguet, Virginie Do, Vish Vogeti, Vítor Albiero, Vladan Petrovic, Weiwei Chu, Wenhan Xiong, Wenyin Fu, Whitney Meers, Xavier Martinet, Xiaodong Wang, Xiaofang Wang, Xiaoqing~Ellen Tan, Xide Xia, Xinfeng Xie, Xuchao Jia, Xuewei Wang, Yaelle Goldschlag, Yashesh Gaur, Yasmine Babaei, Yi~Wen, Yiwen Song, Yuchen Zhang, Yue Li, Yuning Mao, Zacharie~Delpierre Coudert, Zheng Yan, Zhengxing Chen, Zoe Papakipos, Aaditya Singh, Aayushi Srivastava, Abha Jain, Adam Kelsey, Adam Shajnfeld, Adithya Gangidi, Adolfo Victoria, Ahuva Goldstand, Ajay Menon, Ajay Sharma, Alex Boesenberg, Alexei Baevski, Allie Feinstein, Amanda Kallet, Amit Sangani, Amos Teo, Anam Yunus, Andrei Lupu, Andres Alvarado, Andrew Caples, Andrew Gu, Andrew Ho, Andrew Poulton, Andrew Ryan, Ankit Ramchandani, Annie Dong, Annie Franco, Anuj Goyal, Aparajita Saraf, Arkabandhu Chowdhury, Ashley Gabriel, Ashwin Bharambe, Assaf Eisenman, Azadeh
  Yazdan, Beau James, Ben Maurer, Benjamin Leonhardi, Bernie Huang, Beth Loyd, Beto~De Paola, Bhargavi Paranjape, Bing Liu, Bo~Wu, Boyu Ni, Braden Hancock, Bram Wasti, Brandon Spence, Brani Stojkovic, Brian Gamido, Britt Montalvo, Carl Parker, Carly Burton, Catalina Mejia, Ce~Liu, Changhan Wang, Changkyu Kim, Chao Zhou, Chester Hu, Ching-Hsiang Chu, Chris Cai, Chris Tindal, Christoph Feichtenhofer, Cynthia Gao, Damon Civin, Dana Beaty, Daniel Kreymer, Daniel Li, David Adkins, David Xu, Davide Testuggine, Delia David, Devi Parikh, Diana Liskovich, Didem Foss, Dingkang Wang, Duc Le, Dustin Holland, Edward Dowling, Eissa Jamil, Elaine Montgomery, Eleonora Presani, Emily Hahn, Emily Wood, Eric-Tuan Le, Erik Brinkman, Esteban Arcaute, Evan Dunbar, Evan Smothers, Fei Sun, Felix Kreuk, Feng Tian, Filippos Kokkinos, Firat Ozgenel, Francesco Caggioni, Frank Kanayet, Frank Seide, Gabriela~Medina Florez, Gabriella Schwarz, Gada Badeer, Georgia Swee, Gil Halpern, Grant Herman, Grigory Sizov, Guangyi, Zhang, Guna
  Lakshminarayanan, Hakan Inan, Hamid Shojanazeri, Han Zou, Hannah Wang, Hanwen Zha, Haroun Habeeb, Harrison Rudolph, Helen Suk, Henry Aspegren, Hunter Goldman, Hongyuan Zhan, Ibrahim Damlaj, Igor Molybog, Igor Tufanov, Ilias Leontiadis, Irina-Elena Veliche, Itai Gat, Jake Weissman, James Geboski, James Kohli, Janice Lam, Japhet Asher, Jean-Baptiste Gaya, Jeff Marcus, Jeff Tang, Jennifer Chan, Jenny Zhen, Jeremy Reizenstein, Jeremy Teboul, Jessica Zhong, Jian Jin, Jingyi Yang, Joe Cummings, Jon Carvill, Jon Shepard, Jonathan McPhie, Jonathan Torres, Josh Ginsburg, Junjie Wang, Kai Wu, Kam~Hou U, Karan Saxena, Kartikay Khandelwal, Katayoun Zand, Kathy Matosich, Kaushik Veeraraghavan, Kelly Michelena, Keqian Li, Kiran Jagadeesh, Kun Huang, Kunal Chawla, Kyle Huang, Lailin Chen, Lakshya Garg, Lavender A, Leandro Silva, Lee Bell, Lei Zhang, Liangpeng Guo, Licheng Yu, Liron Moshkovich, Luca Wehrstedt, Madian Khabsa, Manav Avalani, Manish Bhatt, Martynas Mankus, Matan Hasson, Matthew Lennie, Matthias Reso, Maxim
  Groshev, Maxim Naumov, Maya Lathi, Meghan Keneally, Miao Liu, Michael~L. Seltzer, Michal Valko, Michelle Restrepo, Mihir Patel, Mik Vyatskov, Mikayel Samvelyan, Mike Clark, Mike Macey, Mike Wang, Miquel~Jubert Hermoso, Mo~Metanat, Mohammad Rastegari, Munish Bansal, Nandhini Santhanam, Natascha Parks, Natasha White, Navyata Bawa, Nayan Singhal, Nick Egebo, Nicolas Usunier, Nikhil Mehta, Nikolay~Pavlovich Laptev, Ning Dong, Norman Cheng, Oleg Chernoguz, Olivia Hart, Omkar Salpekar, Ozlem Kalinli, Parkin Kent, Parth Parekh, Paul Saab, Pavan Balaji, Pedro Rittner, Philip Bontrager, Pierre Roux, Piotr Dollar, Polina Zvyagina, Prashant Ratanchandani, Pritish Yuvraj, Qian Liang, Rachad Alao, Rachel Rodriguez, Rafi Ayub, Raghotham Murthy, Raghu Nayani, Rahul Mitra, Rangaprabhu Parthasarathy, Raymond Li, Rebekkah Hogan, Robin Battey, Rocky Wang, Russ Howes, Ruty Rinott, Sachin Mehta, Sachin Siby, Sai~Jayesh Bondu, Samyak Datta, Sara Chugh, Sara Hunt, Sargun Dhillon, Sasha Sidorov, Satadru Pan, Saurabh Mahajan,
  Saurabh Verma, Seiji Yamamoto, Sharadh Ramaswamy, Shaun Lindsay, Shaun Lindsay, Sheng Feng, Shenghao Lin, Shengxin~Cindy Zha, Shishir Patil, Shiva Shankar, Shuqiang Zhang, Shuqiang Zhang, Sinong Wang, Sneha Agarwal, Soji Sajuyigbe, Soumith Chintala, Stephanie Max, Stephen Chen, Steve Kehoe, Steve Satterfield, Sudarshan Govindaprasad, Sumit Gupta, Summer Deng, Sungmin Cho, Sunny Virk, Suraj Subramanian, Sy~Choudhury, Sydney Goldman, Tal Remez, Tamar Glaser, Tamara Best, Thilo Koehler, Thomas Robinson, Tianhe Li, Tianjun Zhang, Tim Matthews, Timothy Chou, Tzook Shaked, Varun Vontimitta, Victoria Ajayi, Victoria Montanez, Vijai Mohan, Vinay~Satish Kumar, Vishal Mangla, Vlad Ionescu, Vlad Poenaru, Vlad~Tiberiu Mihailescu, Vladimir Ivanov, Wei Li, Wenchen Wang, Wenwen Jiang, Wes Bouaziz, Will Constable, Xiaocheng Tang, Xiaojian Wu, Xiaolan Wang, Xilun Wu, Xinbo Gao, Yaniv Kleinman, Yanjun Chen, Ye~Hu, Ye~Jia, Ye~Qi, Yenda Li, Yilin Zhang, Ying Zhang, Yossi Adi, Youngjin Nam, Yu, Wang, Yu~Zhao, Yuchen Hao, Yundi
  Qian, Yunlu Li, Yuzi He, Zach Rait, Zachary DeVito, Zef Rosnbrick, Zhaoduo Wen, Zhenyu Yang, Zhiwei Zhao, and Zhiyu Ma.
\newblock The llama 3 herd of models, 2024.
\newblock URL \url{https://arxiv.org/abs/2407.21783}.

\bibitem[Grazzi et~al.(2024)Grazzi, Siems, Franke, Zela, Hutter, and Pontil]{grazzi2024unlockingstatetrackinglinearrnns}
Riccardo Grazzi, Julien Siems, Jörg K.~H. Franke, Arber Zela, Frank Hutter, and Massimiliano Pontil.
\newblock Unlocking state-tracking in linear rnns through negative eigenvalues, 2024.
\newblock URL \url{https://arxiv.org/abs/2411.12537}.

\bibitem[Gu \& Dao(2023)Gu and Dao]{gu2023mamba}
Albert Gu and Tri Dao.
\newblock Mamba: Linear-time sequence modeling with selective state spaces, 2023.

\bibitem[Gu et~al.(2022)Gu, Goel, and Re]{gu2022efficiently}
Albert Gu, Karan Goel, and Christopher Re.
\newblock Efficiently modeling long sequences with structured state spaces, 2022.

\bibitem[Hampel(1974)]{hampel1974influence}
Frank~R Hampel.
\newblock The influence curve and its role in robust estimation.
\newblock \emph{Journal of the american statistical association}, 69\penalty0 (346):\penalty0 383--393, 1974.

\bibitem[Han et~al.(2024)Han, Pu, Xia, Han, Pan, Li, Lu, Song, and Huang]{han2024bridgingdividereconsideringsoftmax}
Dongchen Han, Yifan Pu, Zhuofan Xia, Yizeng Han, Xuran Pan, Xiu Li, Jiwen Lu, Shiji Song, and Gao Huang.
\newblock Bridging the divide: Reconsidering softmax and linear attention, 2024.
\newblock URL \url{https://arxiv.org/abs/2412.06590}.

\bibitem[Hendrycks et~al.(2021)Hendrycks, Burns, Basart, Zou, Mazeika, Song, and Steinhardt]{hendrycks2021measuringmassivemultitasklanguage}
Dan Hendrycks, Collin Burns, Steven Basart, Andy Zou, Mantas Mazeika, Dawn Song, and Jacob Steinhardt.
\newblock Measuring massive multitask language understanding, 2021.
\newblock URL \url{https://arxiv.org/abs/2009.03300}.

\bibitem[Horn \& Johnson(2012)Horn and Johnson]{horn2012matrix}
Roger~A Horn and Charles~R Johnson.
\newblock \emph{Matrix analysis}.
\newblock Cambridge university press, 2012.

\bibitem[Hou et~al.(2024)Hou, Zeng, Ma, and Yu]{Hou2024VisualRWKVER}
Haowen Hou, Peigen Zeng, Fei Ma, and Fei~Richard Yu.
\newblock Visualrwkv: Exploring recurrent neural networks for visual language models.
\newblock \emph{ArXiv}, abs/2406.13362, 2024.
\newblock URL \url{https://api.semanticscholar.org/CorpusID:270620870}.

\bibitem[Hudson \& Manning(2019)Hudson and Manning]{Hudson2019GQAAN}
Drew~A. Hudson and Christopher~D. Manning.
\newblock Gqa: A new dataset for real-world visual reasoning and compositional question answering.
\newblock \emph{2019 IEEE/CVF Conference on Computer Vision and Pattern Recognition (CVPR)}, pp.\  6693--6702, 2019.
\newblock URL \url{https://api.semanticscholar.org/CorpusID:152282269}.

\bibitem[Kaddour(2023)]{kaddour2023minipile}
Jean Kaddour.
\newblock The minipile challenge for data-efficient language models, 2023.

\bibitem[Katharopoulos et~al.(2020{\natexlab{a}})Katharopoulos, Vyas, Pappas, and Fleuret]{katharopoulos2020lineartransformers}
Angelos Katharopoulos, Apoorv Vyas, Nikolaos Pappas, and Fran{\c{c}}ois Fleuret.
\newblock Transformers are rnns: Fast autoregressive transformers with linear attention.
\newblock In \emph{International conference on machine learning}, pp.\  5156--5165. PMLR, 2020{\natexlab{a}}.

\bibitem[Katharopoulos et~al.(2020{\natexlab{b}})Katharopoulos, Vyas, Pappas, and Fleuret]{linearTrans2020inputDepend}
Angelos Katharopoulos, Apoorv Vyas, Nikolaos Pappas, and Fran¸cois Fleuret.
\newblock Transformers are rnns: Fast autoregressive transformers with linear attention.
\newblock \emph{Proceedings of the 37 th International Conference on Machine Learning}, 2020{\natexlab{b}}.

\bibitem[Li et~al.(2024{\natexlab{a}})Li, Fang, Smyrnis, Ivgi, Jordan, Gadre, Bansal, Guha, Keh, Arora, Garg, Xin, Muennighoff, Heckel, Mercat, Chen, Gururangan, Wortsman, Albalak, Bitton, Nezhurina, Abbas, Hsieh, Ghosh, Gardner, Kilian, Zhang, Shao, Pratt, Sanyal, Ilharco, Daras, Marathe, Gokaslan, Zhang, Chandu, Nguyen, Vasiljevic, Kakade, Song, Sanghavi, Faghri, Oh, Zettlemoyer, Lo, El-Nouby, Pouransari, Toshev, Wang, Groeneveld, Soldaini, Koh, Jitsev, Kollar, Dimakis, Carmon, Dave, Schmidt, and Shankar]{li2024datacomplm}
Jeffrey Li, Alex Fang, Georgios Smyrnis, Maor Ivgi, Matt Jordan, Samir Gadre, Hritik Bansal, Etash Guha, Sedrick Keh, Kushal Arora, Saurabh Garg, Rui Xin, Niklas Muennighoff, Reinhard Heckel, Jean Mercat, Mayee Chen, Suchin Gururangan, Mitchell Wortsman, Alon Albalak, Yonatan Bitton, Marianna Nezhurina, Amro Abbas, Cheng-Yu Hsieh, Dhruba Ghosh, Josh Gardner, Maciej Kilian, Hanlin Zhang, Rulin Shao, Sarah Pratt, Sunny Sanyal, Gabriel Ilharco, Giannis Daras, Kalyani Marathe, Aaron Gokaslan, Jieyu Zhang, Khyathi Chandu, Thao Nguyen, Igor Vasiljevic, Sham Kakade, Shuran Song, Sujay Sanghavi, Fartash Faghri, Sewoong Oh, Luke Zettlemoyer, Kyle Lo, Alaaeldin El-Nouby, Hadi Pouransari, Alexander Toshev, Stephanie Wang, Dirk Groeneveld, Luca Soldaini, Pang~Wei Koh, Jenia Jitsev, Thomas Kollar, Alexandros~G. Dimakis, Yair Carmon, Achal Dave, Ludwig Schmidt, and Vaishaal Shankar.
\newblock Datacomp-lm: In search of the next generation of training sets for language models, 2024{\natexlab{a}}.

\bibitem[Li et~al.(2023{\natexlab{a}})Li, Allal, Zi, Muennighoff, Kocetkov, Mou, Marone, Akiki, Li, Chim, Liu, Zheltonozhskii, Zhuo, Wang, Dehaene, Davaadorj, Lamy-Poirier, Monteiro, Shliazhko, Gontier, Meade, Zebaze, Yee, Umapathi, Zhu, Lipkin, Oblokulov, Wang, Murthy, Stillerman, Patel, Abulkhanov, Zocca, Dey, Zhang, Fahmy, Bhattacharyya, Yu, Singh, Luccioni, Villegas, Kunakov, Zhdanov, Romero, Lee, Timor, Ding, Schlesinger, Schoelkopf, Ebert, Dao, Mishra, Gu, Robinson, Anderson, Dolan-Gavitt, Contractor, Reddy, Fried, Bahdanau, Jernite, Ferrandis, Hughes, Wolf, Guha, von Werra, and de~Vries]{li2023starcoder}
Raymond Li, Loubna~Ben Allal, Yangtian Zi, Niklas Muennighoff, Denis Kocetkov, Chenghao Mou, Marc Marone, Christopher Akiki, Jia Li, Jenny Chim, Qian Liu, Evgenii Zheltonozhskii, Terry~Yue Zhuo, Thomas Wang, Olivier Dehaene, Mishig Davaadorj, Joel Lamy-Poirier, João Monteiro, Oleh Shliazhko, Nicolas Gontier, Nicholas Meade, Armel Zebaze, Ming-Ho Yee, Logesh~Kumar Umapathi, Jian Zhu, Benjamin Lipkin, Muhtasham Oblokulov, Zhiruo Wang, Rudra Murthy, Jason Stillerman, Siva~Sankalp Patel, Dmitry Abulkhanov, Marco Zocca, Manan Dey, Zhihan Zhang, Nour Fahmy, Urvashi Bhattacharyya, Wenhao Yu, Swayam Singh, Sasha Luccioni, Paulo Villegas, Maxim Kunakov, Fedor Zhdanov, Manuel Romero, Tony Lee, Nadav Timor, Jennifer Ding, Claire Schlesinger, Hailey Schoelkopf, Jan Ebert, Tri Dao, Mayank Mishra, Alex Gu, Jennifer Robinson, Carolyn~Jane Anderson, Brendan Dolan-Gavitt, Danish Contractor, Siva Reddy, Daniel Fried, Dzmitry Bahdanau, Yacine Jernite, Carlos~Muñoz Ferrandis, Sean Hughes, Thomas Wolf, Arjun Guha, Leandro von
  Werra, and Harm de~Vries.
\newblock Starcoder: may the source be with you!, 2023{\natexlab{a}}.

\bibitem[Li et~al.(2023{\natexlab{b}})Li, Du, Zhou, Wang, Zhao, and rong Wen]{Li2023EvaluatingOH}
Yifan Li, Yifan Du, Kun Zhou, Jinpeng Wang, Wayne~Xin Zhao, and Ji~rong Wen.
\newblock Evaluating object hallucination in large vision-language models.
\newblock In \emph{Conference on Empirical Methods in Natural Language Processing}, 2023{\natexlab{b}}.
\newblock URL \url{https://api.semanticscholar.org/CorpusID:258740697}.

\bibitem[Li et~al.(2024{\natexlab{b}})Li, Guo, Guerin, and Lin]{li2024evaluatinglargelanguagemodels}
Yucheng Li, Yunhao Guo, Frank Guerin, and Chenghua Lin.
\newblock Evaluating large language models for generalization and robustness via data compression, 2024{\natexlab{b}}.
\newblock URL \url{https://arxiv.org/abs/2402.00861}.

\bibitem[Lin et~al.(2022)Lin, Mihaylov, Artetxe, Wang, Chen, Simig, Ott, Goyal, Bhosale, Du, et~al.]{lin2022few}
Xi~Victoria Lin, Todor Mihaylov, Mikel Artetxe, Tianlu Wang, Shuohui Chen, Daniel Simig, Myle Ott, Naman Goyal, Shruti Bhosale, Jingfei Du, et~al.
\newblock Few-shot learning with multilingual generative language models.
\newblock In \emph{Proceedings of the 2022 Conference on Empirical Methods in Natural Language Processing}, pp.\  9019--9052, 2022.

\bibitem[Liu et~al.(2024)Liu, Wang, Wu, Feng, Stone, and Liu]{liu2024longhornstatespacemodels}
Bo~Liu, Rui Wang, Lemeng Wu, Yihao Feng, Peter Stone, and Qiang Liu.
\newblock Longhorn: State space models are amortized online learners, 2024.
\newblock URL \url{https://arxiv.org/abs/2407.14207}.

\bibitem[Lozhkov et~al.(2024)Lozhkov, Ben~Allal, von Werra, and Wolf]{lozhkov2024fineweb-edu}
Anton Lozhkov, Loubna Ben~Allal, Leandro von Werra, and Thomas Wolf.
\newblock Fineweb-edu: the finest collection of educational content, 2024.
\newblock URL \url{https://huggingface.co/datasets/HuggingFaceFW/fineweb-edu}.

\bibitem[Lu et~al.(2022)Lu, Mishra, Xia, Qiu, Chang, Zhu, Tafjord, Clark, and Kalyan]{Lu2022LearnTE}
Pan Lu, Swaroop Mishra, Tony Xia, Liang Qiu, Kai-Wei Chang, Song-Chun Zhu, Oyvind Tafjord, Peter Clark, and A.~Kalyan.
\newblock Learn to explain: Multimodal reasoning via thought chains for science question answering.
\newblock \emph{ArXiv}, abs/2209.09513, 2022.
\newblock URL \url{https://api.semanticscholar.org/CorpusID:252383606}.

\bibitem[Lutati et~al.(2023)Lutati, Zimerman, and Wolf]{lutati2023focus}
Shahar Lutati, Itamar Zimerman, and Lior Wolf.
\newblock Focus your attention (with adaptive iir filters), 2023.

\bibitem[McCandlish et~al.(2018)McCandlish, Kaplan, Amodei, and Team]{McCandlish2018AnEM}
Sam McCandlish, Jared Kaplan, Dario Amodei, and OpenAI~Dota Team.
\newblock An empirical model of large-batch training.
\newblock \emph{ArXiv}, abs/1812.06162, 2018.
\newblock URL \url{https://api.semanticscholar.org/CorpusID:56262183}.

\bibitem[Merrill et~al.(2024)Merrill, Petty, and Sabharwal]{illusionstate_2024_merril}
William Merrill, Jackson Petty, and Ashish Sabharwal.
\newblock The illusion of state in state-space models.
\newblock \emph{ArXiv}, abs/2404.08819, 2024.
\newblock URL \url{https://api.semanticscholar.org/CorpusID:269149086}.

\bibitem[Molybog et~al.(2023)Molybog, Albert, Chen, DeVito, Esiobu, Goyal, Koura, Narang, Poulton, Silva, Tang, Liskovich, Xu, Zhang, Kambadur, Roller, and Zhang]{molybog2023theoryadaminstabilitylargescale}
Igor Molybog, Peter Albert, Moya Chen, Zachary DeVito, David Esiobu, Naman Goyal, Punit~Singh Koura, Sharan Narang, Andrew Poulton, Ruan Silva, Binh Tang, Diana Liskovich, Puxin Xu, Yuchen Zhang, Melanie Kambadur, Stephen Roller, and Susan Zhang.
\newblock A theory on adam instability in large-scale machine learning, 2023.
\newblock URL \url{https://arxiv.org/abs/2304.09871}.

\bibitem[Olsson et~al.(2022)Olsson, Elhage, Nanda, Joseph, DasSarma, Henighan, Mann, Askell, Bai, Chen, Conerly, Drain, Ganguli, Hatfield-Dodds, Hernandez, Johnston, Jones, Kernion, Lovitt, Ndousse, Amodei, Brown, Clark, Kaplan, McCandlish, and Olah]{olsson2022incontext}
Catherine Olsson, Nelson Elhage, Neel Nanda, Nicholas Joseph, Nova DasSarma, Tom Henighan, Ben Mann, Amanda Askell, Yuntao Bai, Anna Chen, Tom Conerly, Dawn Drain, Deep Ganguli, Zac Hatfield-Dodds, Danny Hernandez, Scott Johnston, Andy Jones, Jackson Kernion, Liane Lovitt, Kamal Ndousse, Dario Amodei, Tom Brown, Jack Clark, Jared Kaplan, Sam McCandlish, and Chris Olah.
\newblock In-context learning and induction heads, 2022.

\bibitem[Paperno et~al.(2016)Paperno, Kruszewski, Lazaridou, Pham, Bernardi, Pezzelle, Baroni, Boleda, and Fern{\'a}ndez]{paperno2016lambada}
Denis Paperno, Germ{\'a}n Kruszewski, Angeliki Lazaridou, Quan~Ngoc Pham, Raffaella Bernardi, Sandro Pezzelle, Marco Baroni, Gemma Boleda, and Raquel Fern{\'a}ndez.
\newblock The lambada dataset: Word prediction requiring a broad discourse context.
\newblock \emph{arXiv preprint arXiv:1606.06031}, 2016.

\bibitem[Paster et~al.(2023)Paster, Santos, Azerbayev, and Ba]{paster2023openwebmath}
Keiran Paster, Marco~Dos Santos, Zhangir Azerbayev, and Jimmy Ba.
\newblock Openwebmath: An open dataset of high-quality mathematical web text, 2023.

\bibitem[Peng et~al.(2023)Peng, Alcaide, Anthony, Albalak, Arcadinho, Biderman, Cao, Cheng, Chung, Derczynski, Du, Grella, Gv, He, Hou, Kazienko, Kocon, Kong, Koptyra, Lau, Lin, Mantri, Mom, Saito, Song, Tang, Wind, Wo{\'z}niak, Zhang, Zhou, Zhu, and Zhu]{peng2023rwkv}
Bo~Peng, Eric Alcaide, Quentin Anthony, Alon Albalak, Samuel Arcadinho, Stella Biderman, Huanqi Cao, Xin Cheng, Michael Chung, Leon Derczynski, Xingjian Du, Matteo Grella, Kranthi Gv, Xuzheng He, Haowen Hou, Przemyslaw Kazienko, Jan Kocon, Jiaming Kong, Bart{\l}omiej Koptyra, Hayden Lau, Jiaju Lin, Krishna Sri~Ipsit Mantri, Ferdinand Mom, Atsushi Saito, Guangyu Song, Xiangru Tang, Johan Wind, Stanis{\l}aw Wo{\'z}niak, Zhenyuan Zhang, Qinghua Zhou, Jian Zhu, and Rui-Jie Zhu.
\newblock {RWKV}: Reinventing {RNN}s for the transformer era.
\newblock In Houda Bouamor, Juan Pino, and Kalika Bali (eds.), \emph{Findings of the Association for Computational Linguistics: EMNLP 2023}, pp.\  14048--14077, Singapore, December 2023. Association for Computational Linguistics.
\newblock \doi{10.18653/v1/2023.findings-emnlp.936}.
\newblock URL \url{https://aclanthology.org/2023.findings-emnlp.936}.

\bibitem[Peng et~al.(2024{\natexlab{a}})Peng, Goldstein, Anthony, Albalak, Alcaide, Biderman, Cheah, Du, Ferdinan, Hou, Kazienko, GV, Kocoń, Koptyra, Krishna, Jr., Lin, Muennighoff, Obeid, Saito, Song, Tu, Wirawan, Woźniak, Zhang, Zhao, Zhao, Zhou, Zhu, and Zhu]{peng2024eaglefinchrwkvmatrixvalued}
Bo~Peng, Daniel Goldstein, Quentin Anthony, Alon Albalak, Eric Alcaide, Stella Biderman, Eugene Cheah, Xingjian Du, Teddy Ferdinan, Haowen Hou, Przemysław Kazienko, Kranthi~Kiran GV, Jan Kocoń, Bartłomiej Koptyra, Satyapriya Krishna, Ronald~McClelland Jr., Jiaju Lin, Niklas Muennighoff, Fares Obeid, Atsushi Saito, Guangyu Song, Haoqin Tu, Cahya Wirawan, Stanisław Woźniak, Ruichong Zhang, Bingchen Zhao, Qihang Zhao, Peng Zhou, Jian Zhu, and Rui-Jie Zhu.
\newblock Eagle and finch: Rwkv with matrix-valued states and dynamic recurrence, 2024{\natexlab{a}}.
\newblock URL \url{https://arxiv.org/abs/2404.05892}.

\bibitem[Peng et~al.(2024{\natexlab{b}})Peng, Goldstein, Anthony, Albalak, Alcaide, Biderman, Cheah, Ferdinan, GV, Hou, Krishna, Jr., Muennighoff, Obeid, Saito, Song, Tu, Zhang, Zhao, Zhao, Zhu, and Zhu]{rwkv6_colm}
Bo~Peng, Daniel Goldstein, Quentin~Gregory Anthony, Alon Albalak, Eric Alcaide, Stella Biderman, Eugene Cheah, Teddy Ferdinan, Kranthi~Kiran GV, Haowen Hou, Satyapriya Krishna, Ronald~McClelland Jr., Niklas Muennighoff, Fares Obeid, Atsushi Saito, Guangyu Song, Haoqin Tu, Ruichong Zhang, Bingchen Zhao, Qihang Zhao, Jian Zhu, and Rui-Jie Zhu.
\newblock Eagle and finch: {RWKV} with matrix-valued states and dynamic recurrence.
\newblock In \emph{First Conference on Language Modeling}, 2024{\natexlab{b}}.
\newblock URL \url{https://openreview.net/forum?id=soz1SEiPeq}.

\bibitem[Poli et~al.(2023)Poli, Massaroli, Nguyen, Fu, Dao, Baccus, Bengio, Ermon, and R{\'e}]{poli2023hyena}
Michael Poli, Stefano Massaroli, Eric Nguyen, Daniel~Y Fu, Tri Dao, Stephen Baccus, Yoshua Bengio, Stefano Ermon, and Christopher R{\'e}.
\newblock Hyena hierarchy: Towards larger convolutional language models.
\newblock In \emph{International Conference on Machine Learning}, pp.\  28043--28078. PMLR, 2023.

\bibitem[Poli et~al.(2024)Poli, Thomas, Nguyen, Ponnusamy, Deiseroth, Kersting, Suzuki, Hie, Ermon, Ré, Zhang, and Massaroli]{poli2024mechanisticdesignscalinghybrid}
Michael Poli, Armin~W Thomas, Eric Nguyen, Pragaash Ponnusamy, Björn Deiseroth, Kristian Kersting, Taiji Suzuki, Brian Hie, Stefano Ermon, Christopher Ré, Ce~Zhang, and Stefano Massaroli.
\newblock Mechanistic design and scaling of hybrid architectures, 2024.
\newblock URL \url{https://arxiv.org/abs/2403.17844}.

\bibitem[Ponti et~al.(2020)Ponti, Glava{\v{s}}, Majewska, Liu, Vuli{\'c}, and Korhonen]{ponti-etal-2020-xcopa}
Edoardo~Maria Ponti, Goran Glava{\v{s}}, Olga Majewska, Qianchu Liu, Ivan Vuli{\'c}, and Anna Korhonen.
\newblock {XCOPA}: A multilingual dataset for causal commonsense reasoning.
\newblock In Bonnie Webber, Trevor Cohn, Yulan He, and Yang Liu (eds.), \emph{Proceedings of the 2020 Conference on Empirical Methods in Natural Language Processing (EMNLP)}, pp.\  2362--2376, Online, November 2020. Association for Computational Linguistics.
\newblock \doi{10.18653/v1/2020.emnlp-main.185}.
\newblock URL \url{https://aclanthology.org/2020.emnlp-main.185}.

\bibitem[Qwen et~al.(2025)Qwen, :, Yang, Yang, Zhang, Hui, Zheng, Yu, Li, Liu, Huang, Wei, Lin, Yang, Tu, Zhang, Yang, Yang, Zhou, Lin, Dang, Lu, Bao, Yang, Yu, Li, Xue, Zhang, Zhu, Men, Lin, Li, Tang, Xia, Ren, Ren, Fan, Su, Zhang, Wan, Liu, Cui, Zhang, and Qiu]{qwen2025qwen25technicalreport}
Qwen, :, An~Yang, Baosong Yang, Beichen Zhang, Binyuan Hui, Bo~Zheng, Bowen Yu, Chengyuan Li, Dayiheng Liu, Fei Huang, Haoran Wei, Huan Lin, Jian Yang, Jianhong Tu, Jianwei Zhang, Jianxin Yang, Jiaxi Yang, Jingren Zhou, Junyang Lin, Kai Dang, Keming Lu, Keqin Bao, Kexin Yang, Le~Yu, Mei Li, Mingfeng Xue, Pei Zhang, Qin Zhu, Rui Men, Runji Lin, Tianhao Li, Tianyi Tang, Tingyu Xia, Xingzhang Ren, Xuancheng Ren, Yang Fan, Yang Su, Yichang Zhang, Yu~Wan, Yuqiong Liu, Zeyu Cui, Zhenru Zhang, and Zihan Qiu.
\newblock Qwen2.5 technical report, 2025.
\newblock URL \url{https://arxiv.org/abs/2412.15115}.

\bibitem[Rae et~al.(2019)Rae, Potapenko, Jayakumar, and Lillicrap]{rae2019compressivetransformerslongrangesequence}
Jack~W. Rae, Anna Potapenko, Siddhant~M. Jayakumar, and Timothy~P. Lillicrap.
\newblock Compressive transformers for long-range sequence modelling, 2019.
\newblock URL \url{https://arxiv.org/abs/1911.05507}.

\bibitem[Rudelson \& Vershynin(2007)Rudelson and Vershynin]{rudelson2007samplingfromlargematrices}
Mark Rudelson and Roman Vershynin.
\newblock Sampling from large matrices: An approach through geometric functional analysis.
\newblock \emph{J. ACM}, 54\penalty0 (4):\penalty0 21–es, July 2007.
\newblock ISSN 0004-5411.
\newblock \doi{10.1145/1255443.1255449}.
\newblock URL \url{https://doi.org/10.1145/1255443.1255449}.

\bibitem[Sakaguchi et~al.(2021)Sakaguchi, Bras, Bhagavatula, and Choi]{sakaguchi2021winogrande}
Keisuke Sakaguchi, Ronan~Le Bras, Chandra Bhagavatula, and Yejin Choi.
\newblock Winogrande: An adversarial winograd schema challenge at scale.
\newblock \emph{Communications of the ACM}, 64\penalty0 (9):\penalty0 99--106, 2021.

\bibitem[Schlag et~al.(2021)Schlag, Irie, and Schmidhuber]{schlag_deltarule}
Imanol Schlag, Kazuki Irie, and Jürgen Schmidhuber.
\newblock Linear transformers are secretly fast weight programmers, 2021.
\newblock URL \url{https://arxiv.org/abs/2102.11174}.

\bibitem[Schmidhuber(1992)]{fastweights}
Jürgen Schmidhuber.
\newblock Learning to control fast-weight memories: An alternative to dynamic recurrent networks.
\newblock \emph{Neural Computation}, 4\penalty0 (1):\penalty0 131--139, 1992.
\newblock \doi{10.1162/neco.1992.4.1.131}.

\bibitem[Schultz et~al.(2024)Schultz, Adamek, Jusup, Lanctot, Kaisers, Perrin, Hennes, Shar, Lewis, Ruoss, Zahavy, Veličković, Prince, Singh, Malmi, and Tomašev]{schultz2024masteringboardgamesexternal}
John Schultz, Jakub Adamek, Matej Jusup, Marc Lanctot, Michael Kaisers, Sarah Perrin, Daniel Hennes, Jeremy Shar, Cannada Lewis, Anian Ruoss, Tom Zahavy, Petar Veličković, Laurel Prince, Satinder Singh, Eric Malmi, and Nenad Tomašev.
\newblock Mastering board games by external and internal planning with language models, 2024.
\newblock URL \url{https://arxiv.org/abs/2412.12119}.

\bibitem[Shah et~al.(2024)Shah, Bikshandi, Zhang, Thakkar, Ramani, and Dao]{shah2024flashattention}
Jay Shah, Ganesh Bikshandi, Ying Zhang, Vijay Thakkar, Pradeep Ramani, and Tri Dao.
\newblock Flash\-attention-3: Fast and accurate attention with asynchrony and low-precision.
\newblock \emph{arXiv preprint arXiv:2407.08608}, 10, 2024.

\bibitem[Singh et~al.(2019)Singh, Natarajan, Shah, Jiang, Chen, Batra, Parikh, and Rohrbach]{singh2019vqa}
Amanpreet Singh, Vivek Natarajan, Meet Shah, Yu~Jiang, Xinlei Chen, Dhruv Batra, Devi Parikh, and Marcus Rohrbach.
\newblock Towards vqa models that can read, 2019.

\bibitem[Smith et~al.(2018)Smith, Kindermans, and Le]{l.2018dontdecay}
Samuel~L. Smith, Pieter-Jan Kindermans, and Quoc~V. Le.
\newblock Don't decay the learning rate, increase the batch size.
\newblock In \emph{International Conference on Learning Representations}, 2018.
\newblock URL \url{https://openreview.net/forum?id=B1Yy1BxCZ}.

\bibitem[Soboleva et~al.(2023)Soboleva, Al-Khateeb, Myers, Steeves, Hestness, and Dey]{cerebras2023slimpajama}
Daria Soboleva, Faisal Al-Khateeb, Robert Myers, Jacob~R Steeves, Joel Hestness, and Nolan Dey.
\newblock {SlimPajama: A 627B token cleaned and deduplicated version of RedPajama}.
\newblock \url{https://www.cerebras.net/blog/slimpajama-a-627b-token-cleaned-and-deduplicated-version-of-redpajama}, June 2023.
\newblock URL \url{https://huggingface.co/datasets/cerebras/SlimPajama-627B}.

\bibitem[Soldaini et~al.(2024)Soldaini, Kinney, Bhagia, Schwenk, Atkinson, Authur, Bogin, Chandu, Dumas, Elazar, Hofmann, Jha, Kumar, Lucy, Lyu, Lambert, Magnusson, Morrison, Muennighoff, Naik, Nam, Peters, Ravichander, Richardson, Shen, Strubell, Subramani, Tafjord, Walsh, Zettlemoyer, Smith, Hajishirzi, Beltagy, Groeneveld, Dodge, and Lo]{dolma}
Luca Soldaini, Rodney Kinney, Akshita Bhagia, Dustin Schwenk, David Atkinson, Russell Authur, Ben Bogin, Khyathi Chandu, Jennifer Dumas, Yanai Elazar, Valentin Hofmann, Ananya~Harsh Jha, Sachin Kumar, Li~Lucy, Xinxi Lyu, Nathan Lambert, Ian Magnusson, Jacob Morrison, Niklas Muennighoff, Aakanksha Naik, Crystal Nam, Matthew~E. Peters, Abhilasha Ravichander, Kyle Richardson, Zejiang Shen, Emma Strubell, Nishant Subramani, Oyvind Tafjord, Pete Walsh, Luke Zettlemoyer, Noah~A. Smith, Hannaneh Hajishirzi, Iz~Beltagy, Dirk Groeneveld, Jesse Dodge, and Kyle Lo.
\newblock {Dolma: an Open Corpus of Three Trillion Tokens for Language Model Pretraining Research}.
\newblock \emph{arXiv preprint}, 2024.

\bibitem[Sun et~al.(2024)Sun, Li, Dalal, Xu, Vikram, Zhang, Dubois, Chen, Wang, Koyejo, Hashimoto, and Guestrin]{sun2024ttt}
Yu~Sun, Xinhao Li, Karan Dalal, Jiarui Xu, Arjun Vikram, Genghan Zhang, Yann Dubois, Xinlei Chen, Xiaolong Wang, Sanmi Koyejo, Tatsunori Hashimoto, and Carlos Guestrin.
\newblock Learning to (learn at test time): Rnns with expressive hidden states, 2024.

\bibitem[Sun et~al.(2023)Sun, Dong, Huang, Ma, Xia, Xue, Wang, and Wei]{sun2023retentive}
Yutao Sun, Li~Dong, Shaohan Huang, Shuming Ma, Yuqing Xia, Jilong Xue, Jianyong Wang, and Furu Wei.
\newblock Retentive network: A successor to transformer for large language models, 2023.

\bibitem[Tikhonov \& Ryabinin(2021)Tikhonov and Ryabinin]{tikhonov2021s}
Alexey Tikhonov and Max Ryabinin.
\newblock It’s all in the heads: Using attention heads as a baseline for cross-lingual transfer in commonsense reasoning.
\newblock In \emph{Findings of the Association for Computational Linguistics: ACL-IJCNLP 2021}, pp.\  3534--3546, 2021.

\bibitem[Topsakal et~al.(2024)Topsakal, Edell, and Harper]{topsakal2024evaluatinglargelanguagemodels}
Oguzhan Topsakal, Colby~Jacob Edell, and Jackson~Bailey Harper.
\newblock Evaluating large language models with grid-based game competitions: An extensible llm benchmark and leaderboard, 2024.
\newblock URL \url{https://arxiv.org/abs/2407.07796}.

\bibitem[Various(2025)]{ao3_whole}
Various.
\newblock Archive of our own, 2025.
\newblock URL \url{https://archiveofourown.org}.

\bibitem[Vaswani et~al.(2023)Vaswani, Shazeer, Parmar, Uszkoreit, Jones, Gomez, Kaiser, and Polosukhin]{vaswani2023attention}
Ashish Vaswani, Noam Shazeer, Niki Parmar, Jakob Uszkoreit, Llion Jones, Aidan~N. Gomez, Lukasz Kaiser, and Illia Polosukhin.
\newblock Attention is all you need, 2023.

\bibitem[Wang et~al.(2018)Wang, Singh, Michael, Hill, Levy, and Bowman]{wang2018glue}
Alex Wang, Amanpreet Singh, Julian Michael, Felix Hill, Omer Levy, and Samuel Bowman.
\newblock Glue: A multi-task benchmark and analysis platform for natural language understanding.
\newblock In \emph{Proceedings of the 2018 EMNLP Workshop BlackboxNLP: Analyzing and Interpreting Neural Networks for NLP}, pp.\  353--355, 2018.

\bibitem[Wei et~al.(2021)Wei, Bosma, abd Kelvin~Guu, Yu, Lester, Du, Dai, and Le]{wei2019flan}
Jason Wei, Maarten Bosma, Vincent Y.~Zhao abd Kelvin~Guu, Adams~Wei Yu, Brian Lester, Nan Du, Andrew~M. Dai, and Quoc~V. Le.
\newblock Finetuned language models are zero-shot learners, 2021.

\bibitem[Wei et~al.(2022)Wei, Wang, Schuurmans, Bosma, brian ichter, Xia, Chi, Le, and Zhou]{wei2022chain}
Jason Wei, Xuezhi Wang, Dale Schuurmans, Maarten Bosma, brian ichter, Fei Xia, Ed~H. Chi, Quoc~V Le, and Denny Zhou.
\newblock Chain of thought prompting elicits reasoning in large language models.
\newblock In Alice~H. Oh, Alekh Agarwal, Danielle Belgrave, and Kyunghyun Cho (eds.), \emph{Advances in Neural Information Processing Systems}, 2022.
\newblock URL \url{https://openreview.net/forum?id=_VjQlMeSB_J}.

\bibitem[Welbl et~al.(2017)Welbl, Liu, and Gardner]{welbl2017crowdsourcing}
Johannes Welbl, Nelson~F Liu, and Matt Gardner.
\newblock Crowdsourcing multiple choice science questions.
\newblock In \emph{Proceedings of the 3rd Workshop on Noisy User-generated Text}, pp.\  94--106, 2017.

\bibitem[Widrow et~al.(1960)Widrow, Hoff, et~al.]{widrow_delta_rule}
Bernard Widrow, Marcian~E Hoff, et~al.
\newblock Adaptive switching circuits.
\newblock In \emph{IRE WESCON convention record}, volume~4, pp.\  96--104. New York, 1960.

\bibitem[Xu et~al.(2024)Xu, Jiang, Niu, Deng, Poovendran, Choi, and Lin]{xu2024magpiealignmentdatasynthesis}
Zhangchen Xu, Fengqing Jiang, Luyao Niu, Yuntian Deng, Radha Poovendran, Yejin Choi, and Bill~Yuchen Lin.
\newblock Magpie: Alignment data synthesis from scratch by prompting aligned llms with nothing, 2024.
\newblock URL \url{https://arxiv.org/abs/2406.08464}.

\bibitem[Yamana(2025)]{Yamana_Egaroucid_2025}
Takuto Yamana.
\newblock {Egaroucid}, January 2025.
\newblock URL \url{https://www.egaroucid.nyanyan.dev/}.

\bibitem[Yang \& Zhang(2024)Yang and Zhang]{yang2024fla}
Songlin Yang and Yu~Zhang.
\newblock Fla: A triton-based library for hardware-efficient implementations of linear attention mechanism, January 2024.
\newblock URL \url{https://github.com/fla-org/flash-linear-attention}.

\bibitem[Yang et~al.(2023{\natexlab{a}})Yang, Wang, Shen, Panda, and Kim]{yang2023gated}
Songlin Yang, Bailin Wang, Yikang Shen, Rameswar Panda, and Yoon Kim.
\newblock Gated linear attention transformers with hardware-efficient training, 2023{\natexlab{a}}.

\bibitem[Yang et~al.(2024{\natexlab{a}})Yang, Kautz, and Hatamizadeh]{yang2024gated}
Songlin Yang, Jan Kautz, and Ali Hatamizadeh.
\newblock Gated delta networks: Improving mamba2 with delta rule, 2024{\natexlab{a}}.

\bibitem[Yang et~al.(2024{\natexlab{b}})Yang, Wang, Shen, Panda, and Kim]{yang2024gatedlinearattentiontransformers}
Songlin Yang, Bailin Wang, Yikang Shen, Rameswar Panda, and Yoon Kim.
\newblock Gated linear attention transformers with hardware-efficient training, 2024{\natexlab{b}}.
\newblock URL \url{https://arxiv.org/abs/2312.06635}.

\bibitem[Yang et~al.(2024{\natexlab{c}})Yang, Wang, Zhang, Shen, and Kim]{yang2024_deltanet}
Songlin Yang, Bailin Wang, Yu~Zhang, Yikang Shen, and Yoon Kim.
\newblock Parallelizing linear transformers with the delta rule over sequence length, 2024{\natexlab{c}}.
\newblock URL \url{https://arxiv.org/abs/2406.06484}.

\bibitem[Yang et~al.(2023{\natexlab{b}})Yang, Huang, Zhou, and Chen]{yang2023parameterefficienttuningspecialtoken}
Xiaocong Yang, James~Y. Huang, Wenxuan Zhou, and Muhao Chen.
\newblock Parameter-efficient tuning with special token adaptation, 2023{\natexlab{b}}.
\newblock URL \url{https://arxiv.org/abs/2210.04382}.

\bibitem[Yu \& Wang(2024)Yu and Wang]{yu2024mambaout}
Weihao Yu and Xinchao Wang.
\newblock Mambaout: Do we really need mamba for vision?
\newblock \emph{arXiv preprint arXiv:2405.07992}, 2024.

\bibitem[Yuan et~al.(2024)Yuan, Cui, Wang, Ding, Wang, Deng, Shan, Chen, Xie, Lin, Liu, Zhou, Peng, Liu, and Sun]{yuan2024advancing}
Lifan Yuan, Ganqu Cui, Hanbin Wang, Ning Ding, Xingyao Wang, Jia Deng, Boji Shan, Huimin Chen, Ruobing Xie, Yankai Lin, Zhenghao Liu, Bowen Zhou, Hao Peng, Zhiyuan Liu, and Maosong Sun.
\newblock Advancing llm reasoning generalists with preference trees, 2024.

\bibitem[Yue et~al.(2024)Yue, Zheng, Zhang, and Chen]{yue2024mammoth2}
Xiang Yue, Tuney Zheng, Ge~Zhang, and Wenhu Chen.
\newblock Mammoth2: Scaling instructions from the web.
\newblock \emph{Advances in Neural Information Processing Systems}, 2024.

\bibitem[Zellers et~al.(2019)Zellers, Holtzman, Bisk, Farhadi, and Choi]{zellers2019hellaswagmachinereallyfinish}
Rowan Zellers, Ari Holtzman, Yonatan Bisk, Ali Farhadi, and Yejin Choi.
\newblock Hellaswag: Can a machine really finish your sentence?, 2019.
\newblock URL \url{https://arxiv.org/abs/1905.07830}.

\bibitem[Zhang et~al.(2024)Zhang, Luo, Yuan, and Yao]{zhang2024training}
Yifan Zhang, Yifan Luo, Yang Yuan, and Andrew Chi-Chih Yao.
\newblock Training and evaluating language models with template-based data generation.
\newblock \emph{arXiv preprint arXiv:2411.18104}, 2024.

\bibitem[Zhou et~al.(2024)Zhou, Wu, Jiang, and Lan]{zhou2024valueresiduallearningalleviating}
Zhanchao Zhou, Tianyi Wu, Zhiyun Jiang, and Zhenzhong Lan.
\newblock Value residual learning for alleviating attention concentration in transformers, 2024.
\newblock URL \url{https://arxiv.org/abs/2410.17897}.

\bibitem[Zhuang et~al.(2024)Zhuang, Zhang, Zheng, Du, Wang, Ren, Huang, Fu, Yue, and Chen]{zhuang2024structlm}
Alex Zhuang, Ge~Zhang, Tianyu Zheng, Xinrun Du, Junjie Wang, Weiming Ren, Stephen~W. Huang, Jie Fu, Xiang Yue, and Wenhu Chen.
\newblock Structlm: Towards building generalist models for structured knowledge grounding, 2024.

\bibitem[Zucchet \& Orvieto(2024)Zucchet and Orvieto]{zucchet2024recurrentneuralnetworksvanishing}
Nicolas Zucchet and Antonio Orvieto.
\newblock Recurrent neural networks: vanishing and exploding gradients are not the end of the story, 2024.
\newblock URL \url{https://arxiv.org/abs/2405.21064}.

\end{thebibliography}
\bibliographystyle{colm2024_conference}

\newpage
\appendix

\section{Author Contributions}

\paragraph{Bo Peng} Original RWKV-7 ideas, original code, performance optimizations, original experiments, dataset composition, and trained models from 0.1B to 2.9B.

\paragraph{Ruichong Zhang} Personnel organization, wrote Sections \ref{sec:rwkv_arch}, \ref{sec:method}, \ref{subsec:evals}, \ref{subsec:associative_recall}, \ref{conclusions} and Appendices \ref{sec:eigenvalue-proof}, \ref{sec:arch_train}, 
\ref{sec:pseudocode_wkv7}, 
\ref{sec:state_inspections}, 
\ref{subsec:ablation-pile}, \ref{sec:initial_sensitivity}, Figures \ref{fig:rwkv7-overall}, 
\ref{fig:rwkv7-upd}, \ref{fig:rwkv7-details}, \ref{fig:trainingloss_all}, \ref{fig:rwkv_state_inspect_appearance}, \ref{fig:rwkv_state_rms_sr} and Tables \ref{tab:rnn_state_evolution}, \ref{tab:eng_bench}, \ref{tab:multilang_bench}, \ref{tab:mqar}, \ref{tab:dataset-categories}, \ref{tab:model_flop_count}, \ref{tab:lora_dims}, \ref{tab:training-schedules}, \ref{tab:pile}, \ref{tab:lambada_results}. Additional contributions on implementing RWKV-7 for Flash-linear-attention and converting RWKV-7 models on HuggingFace.

\paragraph{Daniel Goldstein} Manuscript organization, initial draft sections \ref{sec:introduction}, \ref{sec:background}, \ref{sec:rwkv_arch}, \ref{sec:method}, \ref{sec:dataset}, FLOPS portion of \ref{subsec:evals}, figures \ref{fig:multi-flops} and \ref{fig:eng-flops}, and appendices \ref{sec:training_dataset_details}, \ref{sec:arch_details}, \ref{sec:pseudocode}. Proofreading and revisions of full manuscript. Oversaw and chose experiments for appendix \ref{sec:ablation-minipile} and for pass-key in subsection \ref{subsec:long_context_experiments}. Assistance with appendix \ref{sec:expressivity-proof} revisions and ideas. Developed and tested initial RWKV-7 Hugging Face model code.

\paragraph{Eric Alcaide} Section \ref{sec:rwkv_arch}, validation of CUDA kernels for scalable training, and manuscript proofreading.

\paragraph{Xingjian Du} Experiments and writing of audio modeling in section \ref{sec:audio-modeling}.

\paragraph{Haowen Hou} Wrote Section \ref{sec:image-modelling}, covering architectural design, coding, model training, experimental evaluation, as well as figure (Figure~\ref{fig:visualrwkv7-arch}), table (Table~\ref{tab:visualrwkv7_results}), and text writing.

\paragraph{Jiaju Lin} Experiments and writing audio modeling in section \ref{sec:audio-modeling}.

\paragraph{Jiaxing Liu} Experiments and writing audio modeling in section \ref{sec:audio-modeling}.

\paragraph{Janna Lu} Needle-in-haystack evaluations for Section \ref{subsec:long_context_experiments}, experiments for figure \ref{fig:multi-flops} and \ref{fig:multi-activeparams}, figure \ref{fig:eng-flops} and \ref{fig:eng-activeparams}, figure \ref{fig:NIAH}, table \ref{tab:eng_bench}, and table \ref{tab:pile}. Edits to Section \ref{sec:introduction}, \ref{sec:pseudocode}, and abstract.

\paragraph{William Merrill} Developmental discussion, proofreading and revisions for appendix \ref{sec:expressivity-proof}.

\paragraph{Guangyu Song} Section \ref{subsec:mad}. Experiments for \ref{subsec:mad}.

\paragraph{Kaifeng Tan} Section \ref{subsec:uncheatable_eval}, Figures \ref{fig:pile_pg19_loss}, \ref{fig:world_pg19_loss}. Appendix \ref{sec:rwkv_othello} on board game modeling with Othello/Reversi, including training data design, model implementation, experiments, and result analysis.

\paragraph{Saiteja Utpala} Section \ref{subsec:evaluate_state_tracking} and state tracking experiments for Figure \ref{fig:rwkv7-state-tracking}.

\paragraph{Nathan Wilce} Extended context-length dataset development and extended-length model training. Description of extended context-length dataset in \ref{subsec:long_context_experiments}.

\paragraph{Johan S. Wind} Main author of RWKV-7 CUDA kernel implementations. Experiments for Figure \ref{fig:H100_speed_comparison}. Section \ref{sec:Speed and Memory Usage} and Appendix \ref{sec:expressivity-proof}. Contributions to Appendix \ref{sec:eigenvalue-proof}.

\paragraph{Tianyi Wu} Appendix \ref{sec:expressivity-proof} and Appendix \ref{sec:param_stat}. Contributions to section \ref{sec:method} and Appendix \ref{sec:eigenvalue-proof} (proofreading and revisions).

\paragraph{Daniel Wuttke} Constructed an itemized table of v2-v3 world datasets.
Proofreading and revision of the manuscript.
Contributed to Abstract, Sections \ref{sec:introduction} and \ref{sec:training_dataset_details}.
Contributed to Tables \ref{tab:dataset2.1}, \ref{tab:dataset3}, \ref{tab:dataset-cite}, and \ref{tab:lora_dims}
Performed evaluations of RWKV-7 world models and reference base models (Table \ref{tab:eng_bench} and \ref{tab:multilang_bench}).

\paragraph{Christian Zhou-Zheng} Experiments and writing for Appendix \ref{sec:ablation-minipile} and Table~\ref{tab:abla_specific}. Contributions to Sections \ref{sec:introduction} and \ref{sec:background}. Proofreading and revisions of full manuscript.

\section{Training Dataset Details}
\label{sec:training_dataset_details}

The RWKV World v3 corpus builds upon the RWKV World v2 corpus \citep{rwkv6_colm} in two steps that we describe separately here for the purposes of reproducibility for the Goose training runs: World v2.1 adds the entries listed in Table \ref{tab:dataset2.1} to sum to a total of approximately 1.4 trillion RWKV World Tokenizer tokens. World v3 adds more entries, listed in Table \ref{tab:dataset3}, to sum to a total of approximately 3.1 trillion tokens. In the combined corpora, all tokens are given equal weighting unless otherwise noted.

\begin{table}[!t]
\centering
\small
\begin{subtable}[t]{0.49\textwidth}
\begin{tabular}{p{45mm}p{15mm}}
    \toprule
    Dataset  & Domain \\
    \midrule
\href{https://huggingface.co/datasets/cerebras/SlimPajama-627B}{slimpajama C4} & Web \\
\href{https://huggingface.co/datasets/allenai/dolma/blob/main/urls/v1_6.txt}{dolma v1.6 (reddit only)}$^a$ & Forums \\
\href{https://huggingface.co/datasets/glaiveai/glaive-code-assistant-v3}{glaive-code-assistant-v3} & Code \\
\href{https://huggingface.co/datasets/m-a-p/Code-Feedback}{m-a-p\_Code-Feedback} & Code \\
\href{https://huggingface.co/datasets/HuggingFaceTB/cosmopedia}{cosmopedia-v0.1} & Synthetic \\
\href{https://huggingface.co/datasets/cognitivecomputations/SystemChat-2.0}{SystemChat-2.0} & Instruct \\
\href{https://huggingface.co/datasets/migtissera/Tess-v1.5}{Tess-v1.5} & Instruct \\
\href{https://huggingface.co/datasets/openbmb/UltraInteract_sft}{UltraInteract\_sft} & Instruct \\
    \bottomrule
\end{tabular}
\end{subtable}
\begin{subtable}[t]{0.49\textwidth}
\begin{tabular}{p{45mm}p{15mm}}
    \toprule
    Dataset & Domain \\
    \midrule
\href{https://huggingface.co/datasets/Magpie-Align/Llama-3-Magpie-Pro-1M-v0.1}{Llama-3-Magpie-Pro-1M-v0.1} & Align \\
\href{https://huggingface.co/datasets/Magpie-Align/Magpie-Pro-MT-300K-v0.1}{Magpie-Pro-MT-300K-v0.1} & Align \\
\href{https://huggingface.co/datasets/Magpie-Align/Magpie-Air-MT-300K-v0.1}{Magpie-Air-MT-300K-v0.1} & Align \\
\href{https://huggingface.co/datasets/Magpie-Align/Magpie-Qwen2-Pro-1M-v0.1}{Magpie-Qwen2-Pro-1M-v0.1} & Align \\
\href{https://huggingface.co/datasets/Magpie-Align/Magpie-Phi3-Pro-300K-Filtered-v0.1}{Magpie-Phi3-Pro-300K-Filtered-} & Align \\
\href{https://huggingface.co/datasets/Magpie-Align/Magpie-Phi3-Pro-300K-Filtered-v0.1}{v1} & \\
\href{https://huggingface.co/datasets/Magpie-Align/Magpie-Gemma2-Pro-200K-Filtered-v0.1}{Magpie-Gemma2-Pro-200K-} & Align \\
\href{https://huggingface.co/datasets/Magpie-Align/Magpie-Gemma2-Pro-200K-Filtered-v0.1}{Filtered-v0.1} & \\
    \bottomrule
\end{tabular}
\end{subtable}
\caption{Components added into the RWKV World v2.1 dataset, their source links, and their domains.\\
\footnotesize{$^a$We added only the reddit datasets from dolma v1.6}\\
\footnotesize{$^b$ \href{https://huggingface.co/datasets/timaeus/pile-dm_mathematics}{DM\_math} as part of The Pile was in World v2 but missed being mentioned explicitly in \citep{rwkv6_colm}}}

\label{tab:dataset2.1}
\end{table}

\begin{table}[!t]
\centering
\small
\begin{subtable}[t]{0.49\textwidth}
\begin{tabular}{p{45mm}p{15mm}}
    \toprule
    Dataset & Domain \\
    \midrule
\href{https://huggingface.co/datasets/cerebras/SlimPajama-627B}{REMOVED slimpajama parts}$^a$ & Web \\
\href{https://huggingface.co/datasets/mlfoundations/dclm-baseline-1.0/tree/main/global-shard_10_of_10}{dclm-baseline-10-of-10}$^b$ & Web \\
\href{https://huggingface.co/datasets/stanford-oval/ccnews}{ccnews} & Web \\
\href{https://huggingface.co/datasets/HuggingFaceFW/fineweb-edu}{fineweb-edu} & Web Edu \\
\href{https://huggingface.co/datasets/math-ai/TemplateGSM}{TemplateGSM} & Math \\
\href{https://huggingface.co/datasets/EleutherAI/proof-pile-2/tree/main/open-web-math}{open-web-math} & Math \\
\href{https://huggingface.co/datasets/EleutherAI/proof-pile-2/tree/main/algebraic-stack}{algebraic-stack} & Math \\
    \bottomrule
\end{tabular}
\end{subtable}
\begin{subtable}[t]{0.49\textwidth}
\begin{tabular}{p{45mm}p{15mm}}
    \toprule
    Dataset & Domain \\
    \midrule
\href{https://huggingface.co/datasets/bigcode/starcoderdata}{StarCoder}$^c$ & Code\\
\href{https://huggingface.co/datasets/HuggingFaceTB/smollm-corpus/tree/main/python-edu}{python-edu} & Code \\
\href{https://huggingface.co/datasets/HuggingFaceTB/smollm-corpus/tree/main/cosmopedia-v2}{cosmopedia-v0.2} & Synthetic \\
\href{https://huggingface.co/datasets/TIGER-Lab/WebInstructSub}{WebInstructSub} & Forums \\
\href{https://huggingface.co/datasets/H-D-T/Buzz-V1.2}{Buzz-v1.2} & Instruct \\
\href{https://huggingface.co/datasets/TIGER-Lab/SKGInstruct}{SKGInstruct} & Instruct \\
\href{https://huggingface.co/datasets/Muennighoff/flan}{FLAN} & Instruct \\
    \bottomrule
\end{tabular}
\end{subtable}
\caption{Components added into the RWKV World v3 dataset, their source links, and their domains.\\
\footnotesize{$^a$We removed the CC and C4 components of SlimPajama from the corpus for World v3}\\
\footnotesize{$^b$For DCLM-baseline, we include only global-shard\_10\_of\_10}\\
\footnotesize{$^c$For StarCoder, we now include all datasets, instead of just those datasets with at least 10 stars}}
\label{tab:dataset3}
\end{table}

\begin{table}[htb]
    \centering
        \begin{tabular}{lc} 
 \toprule
 SlimPajama & \cite{cerebras2023slimpajama}\\
 StarCoder & \cite{li2023starcoder}\\
 Cosmopedia & \cite{benallal2024cosmopedia}\\
 Dolma & \cite{dolma}\\
 UltraInteract & \cite{yuan2024advancing}\\
 Magpie & \cite{xu2024magpiealignmentdatasynthesis}\\
 FineWeb & \cite{lozhkov2024fineweb-edu}\\
 DataComp LM (DCLM) & \cite{li2024datacomplm}\\
 WebInstructSub & \cite{yue2024mammoth2} \\
 StructLM & \cite{zhuang2024structlm} \\
 TemplateGSM & \cite{zhang2024training} \\
 SmolLM Corpus & \cite{benallal2024smollmcorpus} \\
 FLAN & \cite{wei2019flan} \\
 OpenWebMath & \cite{paster2023openwebmath} \\
 Algebraic-Stack & \cite{azerbayev2024llemma} \\
\bottomrule
\end{tabular}
\caption{RWKV World v3 dataset component citations}
\label{tab:dataset-cite}
\end{table}

\begin{table}[htb]
    \centering
        \begin{tabular}{lr} 
 \toprule
 \textbf{Category} & \textbf{Tokens (B)} \\
 \midrule
Web & 1945.2 \\
Books & 337.2 \\
Code & 258.4 \\
Science \& Wiki & 222.7 \\
Fiction & 192.6 \\
Chat \& QA \& Instruction & 110.0 \\
Math & 32.3 \\
Law \& Government & 19.0 \\
Poetry \& Lyrics & 1.7 \\
\midrule
Total & 3119.2 \\
\bottomrule
\end{tabular}
\caption{RWKV World v3 dataset categories}
\label{tab:dataset-categories}
\end{table}

Most of the component data sources for the RWKV World v3 dataset are used intact, with no up- or down-sampling done so that all tokens are given equal weighting. Some sub-sampling is done for over-represented languages within a few data sources in the original World v2 corpus. All newly added tokens in v2.1 and v3 are given equal weighting.

\section{Transition Matrix Eigenvalues and Stability}\label{sec:eigenvalue-proof}

We are interested in all the eigenvalues of the transition matrix $A_t = \mathrm{diag}(w_t) - \hat{\kappa}^T_t (a_t \odot \hat{\kappa}_t)$, and when it can be stable. We drop the subscript $t$ for statements which hold at all timesteps, to avoid clutter.

\begin{thm}
    Let $A = \mathrm{diag}(w) - c\hat{\kappa}^T (a \odot \hat{\kappa}) \in M_m(\mathbb{R})$ be a matrix, where all entries of $w$ belong to $(u, 1)$, where $u = \exp(-e^{-1/2}) = 0.5452 \cdots$ is the clamping lower bound. Also, all entries of $a$ are located in $(0,1)$, and $\hat{\kappa}$ is a unit row vector. When $c \in (0, 1+u)$, the following holds:
    \begin{enumerate}
        \item The matrix $A$ is similar to a symmetric matrix, hence similar to a diagonal matrix.
        \item All eigenvalues of $A$ lie in the interval $(-1, 1)$.
        \item The matrix $A$ admits at most one negative eigenvalue.
        \item If further assumed that $a_t$ is time-independent, then the update formula is guaranteed to be stable, i.e. there exists a time-independent constant $K$ such that
        $$\left\lVert \prod_{t=1}^{T} A_t \right\rVert _{2} \le K,$$
        where $\left\lVert \cdot \right\rVert _{2} $ denotes the spectral norm.
    \end{enumerate}
    This hyperparameter $c$ in the formula can be regarded as a "Global ICLR Multiplier". This hyperparameter $c$ is set to $1$ in the current implementations of RWKV-7 language modeling.
\end{thm}
\begin{proof}
    \ 
    \begin{enumerate}
        \item We notice that 
        $$A = \mathrm{diag}(w) - c\hat{\kappa}^T (a \odot \hat{\kappa}) = \mathrm{diag}(w) - c\hat{\kappa}^T \hat{\kappa} \mathrm{diag}(a).
        $$
        The matrix $A$ itself is not necessarily a symmetric matrix. However, we can use the fact that $\mathrm{diag}(a)$ is positive definite, so we can compute its square root. This allows us to rewrite
        $$
        \begin{aligned}
        \mathrm{diag}(a)^{1/2}A\ \mathrm{diag}(a)^{-1/2} & = \mathrm{diag}(a)^{1/2}\mathrm{diag}(w)\mathrm{diag}(a)^{-1/2} - c\, \mathrm{diag}(a)^{1/2}\hat{\kappa}^T \hat{\kappa}\, \mathrm{diag}(a)\mathrm{diag}(a)^{-1/2} \\
        & = \mathrm{diag}(w) - c\left(\hat{\kappa} \mathrm{diag}(a)^{1/2}\right)^T \left(\hat{\kappa} \mathrm{diag}(a)^{1/2}\right),
        \end{aligned}
        $$
        which is a symmetric matrix. We denote this matrix by $B$. It has exactly the same eigenvalues as $A$, since it is formed by a similarity transformation of $A$. 
        \item Since $B=\mathrm{diag}(a)^{1/2}A\mathrm{diag}(a)^{-1/2}$ is symmetric, all eigenvalues of $B$ (and hence $A$) are real and located on the interval of
        $$
        \left[\min_{|\hat{s}| = 1} \hat{s}B\hat{s}^T, \ \max_{|\hat{s}| = 1} \hat{s}B\hat{s}^T\right].
        $$
        It suffices to show that $\hat{s}B\hat{s}^T \in (-1,1)$. 
        This can be proved via direct expansion of $B$ by
        $$
        \begin{aligned}
        \hat{s}B\hat{s}^T & = \hat{s}\mathrm{diag}(w)\hat{s}^T - c\hat{s}(\hat{\kappa} \mathrm{diag}(a)^{1/2})^T (\hat{\kappa} \mathrm{diag}(a)^{1/2})\hat{s}^T \\
        & \ge u \hat{s}\hat{s}^T -  c\left\lvert(\hat{\kappa} \mathrm{diag}(a)^{1/2})\hat{s}^T\right\rvert^2 \\
        & \ge u-c \\
        & > -1.
        \end{aligned}
        $$
        Similarly,
        $$
        \begin{aligned}
        \hat{s}B\hat{s}^T & = \hat{s}\mathrm{diag}(w)\hat{s}^T - c\hat{s}(\hat{\kappa} \mathrm{diag}(a)^{1/2})^T (\hat{\kappa} \mathrm{diag}(a)^{1/2})\hat{s}^T \\
        & <  \hat{s}\hat{s}^T -  c\left\lvert(\hat{\kappa} \mathrm{diag}(a)^{1/2})\hat{s}^T\right\rvert^2 \\
        & \le \hat{s}\hat{s}^T \\
        & = 1,
        \end{aligned}
        $$
        which completes this part.
        \item Recall $B = \mathrm{diag}(w) - c\left(\hat{\kappa}\,\mathrm{diag}(a)^{1/2}\right)^T \left(\hat{\kappa}\, \mathrm{diag}(a)^{1/2}\right)$. Then $B$ is congruent with $$\hat B = \mathrm{diag}(w)^{-1/2}B\,\mathrm{diag}(w)^{-1/2} = I - u^T u,$$ where $u = \sqrt c\, \hat{\kappa}\,\mathrm{diag}(a)^{1/2}\mathrm{diag}(w)^{-1/2}$.
        
        Clearly, $\hat B$ has at most one negative eigenvalue with value $1-\|u\|^2$, and all other eigenvalues equal to $1$. By Sylvester's law of inertia \citep{horn2012matrix}, congruency preserves the number of negative eigenvalues. Hence, $B$ also has at most one negative eigenvalue.
        \item We drop the subscript in the time-invariant $a_t$.
        
        Since $B_t=\mathrm{diag}(a)^{1/2}A_t\mathrm{diag}(a)^{-1/2}$ is symmetric, the spectral norm of $B_t$ is equal to the largest absolute value for eigenvalues of $B_t$. We proved previously that the eigenvalues lie in $(-1,1)$, so $\left\lVert B_t \right\rVert_2  \le 1$. That is, $B_t$ is a contraction matrix.
        
        Furthermore, we have
        $$
        \begin{aligned}
        \prod_{t=1}^T A_t & = \prod_{i=t}^T \mathrm{diag}(a)^{-1/2}B_t\mathrm{diag}(a)^{1/2} \\
        & = \mathrm{diag}(a)^{-1/2} \left(\prod_{i=1}^t B_t\right)\mathrm{diag}(a)^{1/2}.
        \end{aligned}
        $$
        Then
        $$
        \begin{aligned}
        \left\lVert \prod_{t=1}^T A_t \right\rVert _2 &\le 
        \left\lVert \mathrm{diag}(a)^{-1/2} \right\rVert _2
        \left\lVert \prod_{t=1}^T B_t \right\rVert _2
        \left\lVert \mathrm{diag}(a)^{1/2} \right\rVert _2 \\
        & \le \min(a)^{-1/2} \cdot 1 \cdot 1.
        \end{aligned}
        $$

        We can set $K = \min(a)^{-1/2} < \infty$, which is time-independent.
    \end{enumerate}
\end{proof}

While our stability proof only holds for time-independent $a_t$, we do not observe any problems with time-varying $a_t$ in practice. We therefore put greater emphasis on the expressivity of RWKV-7, and include time-varying $a_t$.



\section{Expressivity of RWKV-7}
\label{sec:expressivity-proof}

We show that the RWKV-7 architecture can express $\mathsf{NC}^1$-complete state tracking problems that cannot be expressed by transformers or other recurrent architectures such as S4 and Mamba, under standard complexity conjectures.
We first show a particular $\mathsf{NC}^1$-complete problem that can be expressed by RWKV-7 in Section~\ref{sec:expressivity-proof} and then generalize the argument to show that any regular language can be recognized by RWKV-7 in Section~\ref{sec:reg_lang}. As regular language recognition can be understood to formalize finite state tracking problems, this suggests an expressivity advantage of RWKV-7 on state-tracking problems.

\subsection[Warmup: Expressivity Beyond TC0]{Warmup: Expressivity Beyond $\mathsf{TC}^0$}\label{sec:s5-proof}

Recall that the RWKV-7 wkv state is updated, at each token, by multiplication with $A_t = \mathrm{diag}(w_t) - c\hat{\kappa}^T_t (a_t \odot \hat{\kappa}_t)$, where $c = 1$. In the following, we will consider $c = 2$, and show that this yields expressivity beyond $\mathsf{TC}^0$ (unless $\mathsf{TC}^0 = \mathsf{NC}^1$). $\mathsf{TC}^0$ is the complexity class which includes transformers as well as all non-gated or diagonal SSMs \citep{illusionstate_2024_merril, barrington1989bounded}.



\begin{thm}
    RWKV-7 can solve a problem which is $\mathsf{NC}^1$-complete under $\mathsf{AC}^0$ reductions.
\end{thm}
\begin{proof}
    RWKV-7 can, by Lemma \ref{lem:swap_tracking}, solve the problem of tracking swaps on five elements. This problem is $\mathsf{NC}^1$-complete under $\mathsf{AC}^0$ reductions \citep{illusionstate_2024_merril}. 
\end{proof}

\begin{lemma}\label{lem:swaps}
    The RWKV-7 transition matrix can represent an arbitrary swap matrix, where a swap matrix is an identity matrix with two of its rows swapped.
\end{lemma}
\begin{proof}
    Given indices $x$ and $y$, let
    $$w_t = 1,\quad c = 2,\quad \hat{\kappa}_t = (e_x-e_y)/\sqrt 2,\text{ and } a_t = 1.$$
    Here $e_i$ denotes the vector with 1 at position $i$ and 0 elsewhere.
    
    Then the transition matrix becomes $A_t = I-e_x^Te_x-e_y^Te_y+e_x^Te_y+e_y^Te_x$, which is the permutation matrix that swaps indices $x$ and $y$.
\end{proof}
\begin{lemma}[RWKV-7 can track swaps on 5 elements]\label{lem:swap_tracking}
    Let a sequence of swaps on 5 elements be encoded in the format
    $$\#[x_1 \leftrightarrow y_1][x_2 \leftrightarrow y_n]\dots,$$
    where $\#$ is a special beginning-of-sequence token, and $[x_i \leftrightarrow y_i]$ is a token that denotes a swap between elements $x_i$ and $y_i$. Then there exists a one-layer RWKV-7 model which outputs 1 if the sequence of swaps encode the identity permutation, and outputs 0 otherwise.
\end{lemma}
\begin{proof}
    Let the RWKV-7 model have 5 wkv heads of head dimension 5. The embedding weights and "weight preparation" part of the time mixing layer are set such that the following two properties hold:
    
    Firstly, when the model sees the special beginning-of-sequence token, the $i$th wkv head receives 
    $$w = 1,\ c = 2,\ \hat{\kappa} = e_i,\ a = 0, \text{ and } v = \tilde k = e_i.$$
    Here $e_i$ denotes the vector with 1 at position $i$ and 0 elsewhere.
    This sets the state of the $i$th wkv head to $\textbf{wkv} = e_i^T e_i$, which represents state $i$.
    
    Secondly, when token represents a swap between tokens $1 \le x < y \le 5$. In this case, using Lemma \ref{lem:swaps}, all wkv heads receive
    $$w = 1,\ c = 2,\ \hat{\kappa} = (e_x-e_y)/\sqrt 2,\ a = 1, \text{ and }v = \tilde k = 0.$$
    
    This changes the state to $y$ if it was $x$, or $x$ if it was $y$, or keeps it unchanged otherwise.
    
    To calculate the output, the $i$th wkv head checks whether it represents state $i$, by applying receptance $r = e_i$. Finally, the MLP layer combines these outputs to check if all 5 heads agree with the identity permutation, and the output head outputs 1 if this is true, and 0 otherwise.
\end{proof}

\subsection{Main Result: RWKV-7 Can Recognize Any Regular Language}\label{sec:reg_lang}

Moreover, we are able to demonstrate that RWKV-7 has the capability to recognize any regular language.
The regular languages are precisely those which can be recognized by a deterministic finite automaton (DFA). Therefore, it is sufficient to show that RWKV-7 can simulate any DFA. We define DFAs in the usual way:

\begin{definition}\label{def:DFA}
Classically, a DFA is a tuple 
$\mathcal{A}= (Q, \Sigma, \delta, q_0, F)$ where $Q$ is a finite set of states, $\Sigma$ is a finite vocabulary of tokens, $\delta_\sigma : Q \to Q$ is a transition function for each token $\sigma \in \Sigma$, $q_0 \in Q$ is the initial state, and $F \subseteq Q$ is a set of accepting states.

Equivalently, the DFA's computation on $w \in \Sigma^{\ast}$ can be represented by matrix computations.
Each $\delta_\sigma$ can be represented by a boolean matrix $M_\sigma \in \{0, 1\}^{|Q| \times |Q|}$, where $M_w (i, j) = 1$ iff $\delta_w (q_j) = q_i$.
The initial state $q_0$ can be represented as a one-hot vector $\alpha \in \{0, 1\}^{\lvert Q \rvert}$, and the set of accepting states $F$ can be represented as a multi-hot vector $\omega \in \{0, 1\}^{\lvert Q \rvert}$.

For a given string $w_1 \cdots w_T$, the DFA computes
\begin{equation} \alpha \cdot M_{w_1} \cdots M_{w_T} \cdot \omega^\top \label{eq:DFA_calc} \end{equation}
We say that $w \in L$ if and only if this expression evaluates to 1.
\end{definition}



Having defined regular languages, we are in a position to show our main result:

\begin{thm}\label{thm:reg_lang}
    For any regular language, there exists a 4-layer RWKV-7 model that recognizes it.
\end{thm}
\begin{proof}
    We demonstrate that the RWKV-7 architecture can recognize strings in an arbitrary regular language $L$. Consider a string $w \in \Sigma^{\ast}$ and its membership in $L$. There exists a DFA (Definition \ref{def:DFA}) which recognizes $L$ by evaluating \eqref{eq:DFA_calc}. To prove that RWKV-7 can recognize any regular language, it is therefore sufficient to construct a RWKV-7 that evaluates \eqref{eq:DFA_calc}.

    \textbf{Construction overview.}
    A standard way to recognize a regular language would be to construct the transition matrices for each token and then multiply them to compute the final state.
    However, this does not work for RWKV-7 because because an arbitrary DFA transition can have rank $n = \lvert Q \rvert$ (i.e., the number of states), whereas a wkv head can only implement a simple elementary transition matrix per input token.
    A natural idea is to factor each DFA transition into $n$ \textit{elementary transition matrices} (Lemma \ref{lem:factor_transition}), each of which can be directly implemented by wkv heads (Lemma \ref{lem:elementary_construction}).
    However, expanding DFA transitions in this way gives more elementary transition matrices than tokens, which means it cannot be directly implemented by RWKV-7.
    
    Fortunately, a simple modification of this idea will allow us to implement regular language recognition in RWKV-7. Rather than expanding each transition matrix to $n$ elementary matrices, we use the first three layers to convert blocks of $n$ DFA transitions to products of $n$ elementary matrices with the same product. The final layer then multiplies these elementary matrices to obtain the final state.

    \textbf{Details: information routing via residual stream and Layernorm.}
    The output of each layer is stored in the residual stream of the architecture, in independent subspaces, which makes them all available to deeper layers. Since the outputs come from a finite set, they may be represented by one-hot encoding.

    The input to each time mixing block therefore contains the outputs of all previous layers. First, a Layernorm is applied. Note that one-hot encodings have constant norm, which ensures that the Layernorm preserves the encoded information.

    Next is the "weight preparation" part of the time mixing block, which takes input $x_t$ encoding the output from previous layers, and constructs $r_t,w_t,\tilde k_t,v_t,\hat{\kappa}_t$ and $a_t$ for the wkv heads. The weight preparation is sufficiently expressive to allow each of the 6 output variables to be an arbitrary linear transform of $x_t$. This allows selecting $r_t,w_t,\tilde k_t,v_t,\hat{\kappa}_t$ and $a_t$ based on arbitrary outputs from previous layers. Therefore, the wkv heads are the main part of the construction.

    \textbf{WKV head construction.} For any $1 \leq t \leq T$, the current position $t$ can be split into $t = ln + \hat t$, for integers $0 \le l$ and $1 \le \hat t \le n$. We view the input sequence as blocks of length $n$, indexed by $l$.

    First, we describe $l \ge 1$. Consider the product of DFA transitions $\tilde M_l = M_{w_{(l-1)n+1}}\dots M_{w_{(l-1)n+n}}$. This product is also a DFA transition matrix (i.e., it has a single 1 per column). Hence, Lemma \ref{lem:factor_transition} allows us to factor $\tilde M_l = G_{l,1}G_{l,2}\dots G_{l,n}$, where $G_{l,1},\dots,G_{l,n}$ are elementary transition matrices. Fix one such factorization for each possible DFA transition matrix. At position $t = ln+\hat t$, the wkv state in the fourth (final) layer is right-multiplied by the elementary transition matrix $G_{l,\hat t}$. Since $G_{l,\hat t}$ is uniquely defined by the last $2n$ tokens and position modulo $2n$, it can be computed in the third layer by Lemma \ref{lem:lookup}, and fed to the fourth layer.

    The first block $l = 0$ is handled as a special case. Again, we consider the fourth layer. At the first token, set $v_1 = e_1$ and $\tilde k_1 = \alpha$. All subsequent transitions $t \ge 2$ set $v_t = \tilde k_t = {\bf 0}$, and implement identity transition matrices from Lemma \ref{lem:elementary_construction}. Then ${\bf wkv}_t = e_1^T \alpha$ for $1 \le t \le n$.

    With this construction, for all $1 \le t \le T$, the fourth layer's wkv state becomes \[{\bf wkv}_t = e_1^T \hat\alpha_t\text{, where } \hat\alpha_t = \alpha M_{w_1}M_{w_2} \dots M_{w_{(l-1)n}} G_{l,1}G_{l,2} \dots G_{l,\hat t},\]
    where as usual, empty products (such as for $l \le 1$) evaluate to identity matrices.

    Duplicate this construction into $n$ wkv heads, where the $i$th applies receptance $e_i$. In combination, the wkv heads read out the whole vector $\hat\alpha_t$.

    \textbf{Final tokens.} Consider $t = T$. To match the DFA evaluation formula from \eqref{eq:DFA_calc}, $\hat\alpha_T$ can be multiplied by
    \[\hat\omega_T^T = G_{l,\hat t+1}\dots G_{l,n} M_{w_{ln+1}}M_{w_{ln+2}}\dots M_{w_{ln+\hat t}} \omega^T.\]

    Fortunately, this expression is a fixed function of the current position modulo $2n$ and the last $2n$ tokens, and can hence be found by a lookup in the third layer by Lemma \ref{lem:lookup}. The final MLP on the final token can thus output the scalar product $\hat\alpha_T \hat\omega_T^T$, which is equal to \eqref{eq:DFA_calc}, thus completing the construction.
\end{proof}

Our construction uses MLPs to implement lookup tables with sizes on the order of $|\Sigma|^{2n}$, which may require MLP layers that are exponentially wide in the number of states of the original DFA.

\subsection{Lemmas for Theorem~\ref{thm:reg_lang}}
The proof for Theorem~\ref{thm:reg_lang} requires many different lemmas, which are stated and proved in the following.

A single RWKV-7 wkv state transition cannot directly implement an arbitrary DFA transition. However, DFA transition matrices can be decomposed into a product of \textit{elementary transition matrices} that can be directly simulated by wkv state transitions (cf. Lemma:~\ref{lem:swaps}):

\begin{lemma}\label{lem:factor_transition}
    Let $M$ be a DFA transition matrix. I.e., $M$ has shape $n\times n$, and contains a single 1 in each column. Then $M$ can be factored into a product of $n$ \textit{elementary transition matrices} $G_1,\dots,G_n$. Specifically, 
    \[M = G_1G_2\dots G_n,\]
    where each of $G_1,\dots,G_n$ has one of the following forms:
    \begin{enumerate}
        \item Identity matrix.
        \item Swap matrix $x \leftrightarrow y$; an identity matrix with rows $x$ and $y$ swapped.
        \item Copy matrix $x \to y$; an identity matrix with column $y$ replaced by a copy of column $x$.
    \end{enumerate}
\end{lemma}
\begin{proof}
    We will greedily build $M$ from the identity matrix by right-multiplying elementary transition matrices. Right-multiplying by the identity matrix does nothing, right-multiplying with a swap matrix $x \leftrightarrow y$ swaps columns $x$ and $y$, and right-multiplying with a copy matrix $x \to y$ replaces column $x$ by a copy of column $y$.

    We use $X$ to denote the current partial product of elementary transition matrices. Initially, $X$ is the identity matrix, and the goal is to apply $n$ transitions to make $X = M$. We proceed greedily in three stages.
    \begin{enumerate}
        \item Find a column $c$ of $X$ that differs from column $c$ of $M$, but which matches a different column $c'$ of $M$. If no such position exists, proceed to the next stage. Otherwise, right-multiply $X$ by a swap matrix that swaps columns $c$ and $c'$.

            Note that $M$ and $X$ differed in columns $c$ and $c'$ before the swap, but match in column $c$ after the swap.
        \item Find a column $c$ where $M$ and $X$ differ. If no such column exists, proceed to the next stage. Otherwise, the previous stage has ensured that there exists a column $c'$ of $X$, necessarily different from $c$, which contains a column identical to column $c$ of $M$. 

            Right-multiply by a copy matrix that replaces column $c$ of $X$ with column $c'$ of $X$. Note that $M$ and $X$ differed in column $c$ before the move, while they agree afterwards.
        \item Right-multiply by identity matrices until $X$ is the product of $n$ elementary transition matrices.

            Initially, $M$ and $X$ differed in at most $n$ columns. Each subsequent right-multiplication by an elementary transition matrix reduced the number of differing columns by at least one. Thus, stages 1 and 2 required at most $n$ multiplications.
    \end{enumerate}
\end{proof}

Recall that the RWKV-7 wkv state is at token index $t$ updated by multiplication with \[A_t = \mathrm{diag}(w_t) - c\hat{\kappa}^T_t (a_t \odot \hat{\kappa}_t),\] with $c = 1$. To simplify the presentation of the core idea, we will instead present a construction with $c = 2$, and then show how to remove this assumption in Section \ref{sec:c_2}.
\begin{lemma}\label{lem:elementary_construction}
    For any elementary transition matrix $G$ (in the sense of Lemma~\ref{lem:factor_transition}), there exist $n$-dimensional vectors $\hat{\kappa}$ and $\vec a$, where $\|\hat{\kappa}\| = 1$, $\vec a$ has elements in $\{0,1\}$, and
    \[G = \mathrm{diag}(w) - c\hat{\kappa}^T (\vec a \odot \hat{\kappa}),\]
    where $c = 2$ and $w = {\bf 1}$.
\end{lemma}
\begin{proof}
    We use $e_i$ to denote the $n$-dimensional vector with 1 at position $i$ and 0 elsewhere.
    \begin{enumerate}
        \item Identity matrix: Select any length one vector $\hat{\kappa}$, for example $\hat{\kappa} = e_1$, and $\vec a = {\bf 0}$.
        \item Swap matrix; an identity matrix with rows $x$ and $y$ swapped: Select $\hat{\kappa} = (e_x-e_y)/\sqrt 2$ and $\vec a = {\bf 1}$.
        \item Copy matrix; an identity matrix with column $y$ replaced by a copy of column $x$: Select $\hat{\kappa} = (e_x-e_y)/\sqrt 2$ and $\vec a = e_x$.
    \end{enumerate}
\end{proof}

We now move on to the explicit construction of the first three layers.

\begin{lemma}\label{lem:pos_parity}
    There is a 1-layer RWKV-7 which outputs whether the current position is first, and whether the current position is even or odd.
\end{lemma}
\begin{proof}
    RWKV-7 performs a token-shift operation before the wkv heads are reached. This token shift takes as input the last token, and can therefore detect whether a previous token exists.

    The first layer's wkv state can track position parity by selecting $\tilde k_1 = v_1 = e_1$ for the first position $t = 1$, and subsequently letting $c = 2$, $w_t = a_t = {\bf 1}$, $\hat{\kappa}_t = e_1$ and $\tilde k_t = v_t = {\bf 0}$ for $t \ge 2$. This leads to ${\bf wkv}_1 = e_1^T e_1$ and subsequently ${\bf wkv}_t = {\bf wkv}_{t-1}(I-2e_1^Te_1)$ for $t \ge 2$. Then ${\bf wkv}_t = e_1^T e_1$ for odd $t$ and ${\bf wkv}_t = -e_1^T e_1$ for even $t$.

    Receptance $r = e_1$ can then be used to read out the sign of the wkv state, which encodes the current position's parity.

    The subsequent MLP can transform these two boolean outputs to an arbitrary format.
\end{proof}

\begin{lemma}\label{lem:mod_2n}
    For any positive integer $n$, there is a 2-layer RWKV-7 that outputs the position modulo $2n$.
\end{lemma}
\begin{proof}
    We use Lemma \ref{lem:pos_parity} for the first layer, which tells the second layer whether the current position is first, and the parity of the current position.

    At the first position, set $\tilde k_1 = v_1 = e_1$, such that ${\bf wkv}_1 = e_1^T e_1$. For all subsequent positions $t \ge 2$, set $c = 2$, $w_t = a_t = {\bf 1}$ and $\tilde k_t = v_t = {\bf 0}$. Furthermore, set $\hat{\kappa}_t = e_1$ for even $t$ and $\hat{\kappa}_t = \cos(\pi / n) e_1 + \sin(\pi / n) e_2$ at odd $t$. Then ${\bf wkv}_T = e_1^T e_1 (I-2\hat{\kappa}_2^T \hat{\kappa}_2)\dots(I-2\hat{\kappa}_T^T \hat{\kappa}_T)$. Note that for even $t \ge 2$, the matrix $(I-2\hat{\kappa}_t^T \hat{\kappa}_t)(I-2\hat{\kappa}_{t+1}^T \hat{\kappa}_{t+1})$ rotates the first two coordinates by an angle $2\pi/n$. Thus,
    \[{\bf wkv}_t = \begin{cases}e_1^T \left(\cos(\pi (t-1)/n)e_1 + \sin(\pi (t-1)/n) e_2\right), \quad &\text{if $t$ is odd}\\
    e_1^T \left(-\cos(\pi (t-2)/n)e_1 + \sin(\pi (t-2)/n) e_2\right), \quad &\text{if $t$ is even}\end{cases}.\]

    The wkv heads are immediately followed by group normalization, which discards information about magnitudes. We therefore use $2n$ wkv heads of the type above, where the $k$th head applies receptance $r = \cos(\pi k/n)e_1 + \sin(\pi k/n) e_2$. The signs of these readouts, along with the parity from the first layer, can then be combined in the subsequent MLP layer to deduce the current position modulo $2n$.
\end{proof}

\begin{lemma}\label{lem:lookup}
    Let $\tilde t \equiv t$ modulo $2n$ be the current position modulo $2n$, and let $w_{t},w_{t-1},\dots,w_{t-(2n-1)}$ be the last $2n$ tokens. Define $w_t = |\Sigma|+1$ for before-sequence tokens $t \le 0$. Let $\Xi[\tilde t, w_{t},w_{t-1},\dots,w_{t-(2n-1)}]$ be a lookup table that takes as key the current position modulo $2n$ and the $2n$ most recent tokens.

    Given inputs $\tilde t$ and $w_{t},w_{t-1},\dots,w_{t-(2n-1)}$, there is a layer of RWKV-7 that simulates $\Xi$.

    Moreover, there is a 3-layer RWKV-7 that simulates $\Xi$. 
\end{lemma}
\begin{proof}
    Recall that the wkv state is initialized to all zeros. We apply the wkv state update
    \[{\bf wkv}_t = {\bf wkv}_{t-1}(I-e_{\tilde t}^Te_{\tilde t})+e_{w_t}^T e_{\tilde t}.\]
    This can be achieved by selecting $c \ge 1, w_t = {\bf 1}, a_t = \frac{1}{c}{\bf 1}, \tilde{k}_t = \hat{\kappa}_t = e_{\tilde t}$ and $v = e_{w_t}$.
    
    In words, the state update replaces the $\tilde t$th column of the wkv state with $e_{w_t}$. Hence, the wkv state stores the last $2n$ tokens.

    We make $n$ such wkv heads, where the $i$th wkv head applies receptance $r = e_i$. This reads out the full state, which contains the last $2n$ tokens. The state and $\tilde t$ are fed into the subsequent MLP layer, which performs the lookup into $\Xi$.

    Note that by Lemma \ref{lem:mod_2n}, the first two layers can compute current position modulo $2n$ and the $2n$ most recent tokens. Hence, a 3-layer RWKV can compute $\Xi$.
\end{proof}

\paragraph{Removing the assumption \texorpdfstring{$c = 2$}{c = 2}}\label{sec:c_2}
Some of our constructions use $c = 2$, while the actual model uses $c = 1$. However, since the transition matrix is $A_t = \mathrm{diag}(w_t) - c\hat{\kappa}^T_t (a_t \odot \hat{\kappa}_t)$, halving both $c$ and $w_t$ simply causes $A_t$ to be halved. This causes the wkv state to halve in magnitude at each token. However, since the wkv heads are immediately followed by group normalizations, the magnitude of the wkv state does not affect subsequent calculations. Additionally, since floating point numbers store a separate exponent, this rescaling only requires log-precision.
The shrinking state could in principle be mismatched with the scales of $v_t$ and $\tilde k_t$, but our constructions always satisfy $v_t = \tilde k_t = {\bf 0}$ whenever $c \ne 1$ is required.

\section{Additional Architectural and Training Details}\label{sec:arch_train}

\paragraph{Architecture Diagram}
\begin{figure*}[htbp]
    \centering
    \includegraphics[width=0.7\linewidth]{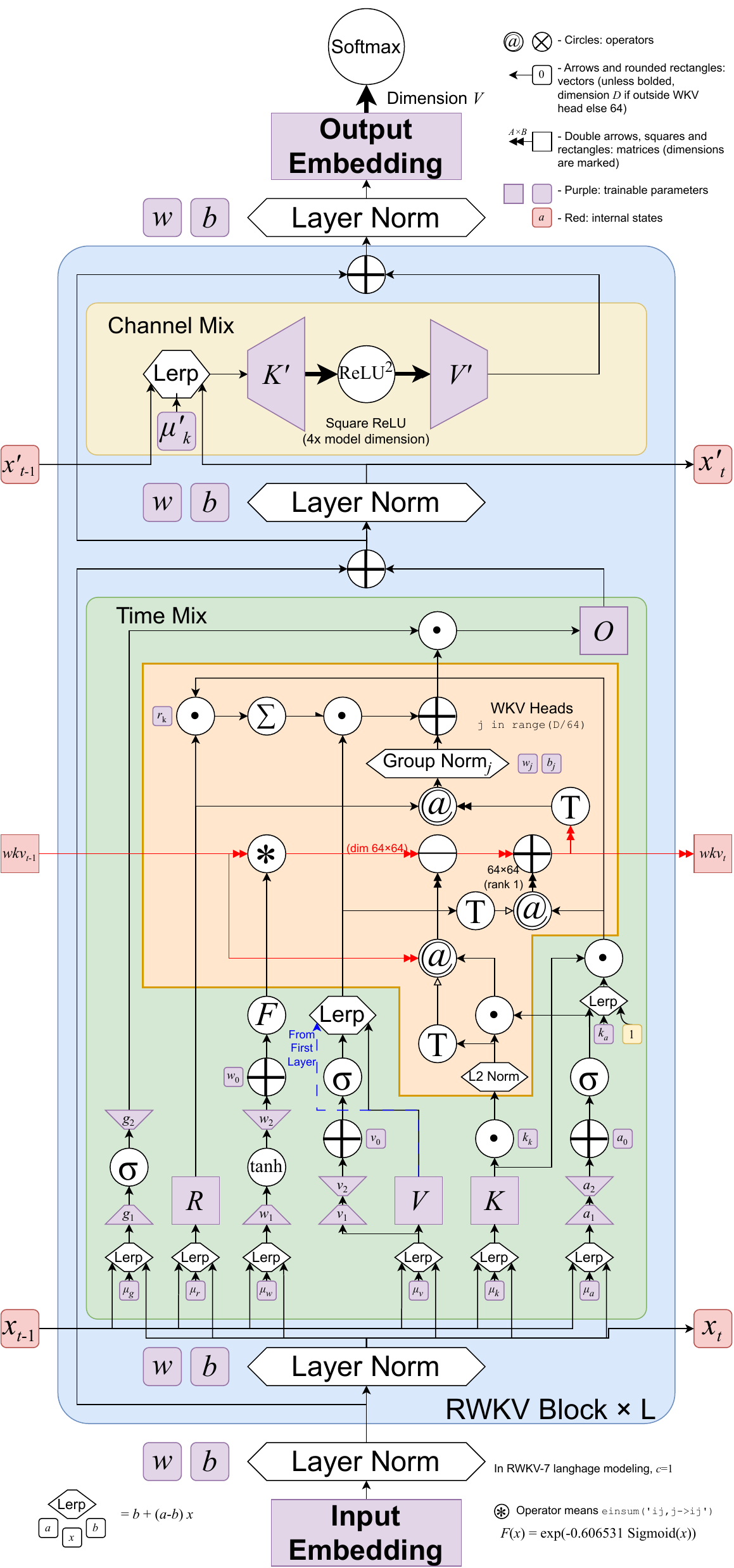}
    \caption{The architecture of RWKV-7, drawn in detail.}
    \label{fig:rwkv7-details}
\end{figure*}

We provide a comprehensive architecture diagram (Figure \ref{fig:rwkv7-details}) in order to help readers thoroughly understand our architecture design.

\paragraph{Parameters and Dimensions}

Throughout this section, we denote by $D$ the model dimension, $L$ the number of layers, $h=D/D_h$ the number of heads, and $V$ the vocabulary size. All models are trained with head size $D_h=64$, i.e., each time-mixing has $h=D/64$ heads with dimension $64 \times 64$.

Pile models are trained with $V=50304$ with the GPT-NeoX 20B tokenizer. World models are trained with $V=65536$ with RWKV World tokenizer.

\begin{table}[htb]
    \centering
    \begin{adjustbox}{max width=\linewidth}
        \begin{tabular}{lcccc} 
 \toprule
Model Name & $L$ & $D$ & State Size (WKV + Shift) & Parameters \\
\midrule
RWKV7-World3-0.1B & 12 & 768 & \num{589824} + \num{18432} & \num{191034624}  \\
RWKV7-World3-0.4B & 24 & 1024 & \num{1572864} + \num{49152} & \num{450767872}  \\
RWKV7-World3-1.5B & 24 & 2048 & \num{3145728} + \num{98304} & \num{1527404544}  \\
RWKV7-World3-2.9B & 32 & 2560 & \num{5242880} + \num{163840} & \num{2947735040}  \\
\bottomrule\\
\end{tabular}
\end{adjustbox}
\caption{Released Goose models with parameters and state size.} \label{tab:model_flop_count}
\end{table}

RWKV-7 uses four low-rank MLPs for decay $w$, value residual $v$, in-context learning rate $a$ and gate $g$ respectively. The intermediate dimensions are listed in Table \ref{tab:lora_dims}. These values are based on our mere speculation of how much information can be passed through.

\begin{table}[ht]
\centering
\begin{adjustbox}{max width=\linewidth}
\begin{tabular}{ccccc} 
\hline
Dimension ($D$) & $d_w$ & $d_a$ & $d_v$ & $d_g$ \\
\hline
768 & 64 & 64 & 32 & 128 \\
1024 & 64 & 64 & 32 & 128 \\
2048 & 96 & 96 & 64 & 256 \\
2560 & 96 & 96 & 64 & 320 \\
4096 & 128 & 128 & 96 & 480 \\
6144 & 128 & 128 & 96 & 640 \\
\hline\\
\end{tabular}
\end{adjustbox}
\caption{Suggested Intermediate Dimensions for the low-rank MLPs for RWKV-7 models}
\label{tab:lora_dims}
\end{table}

The number of parameters for all RWKV-7 models can be computed by the formula: 
\begin{align}
\label{eq:params-7}
\#(\mathrm{Params}) = 2DV + 4D + LD \left(12D + 2\left(d_w + d_a +d_v +d_g \right) + 19 \right) - (2Dd_v + D).
\end{align}

Where:
\begin{itemize}
    \item The weights of the embeddings and head, and the Layernorms beside them, yield $2DV + 4D$ parameters;
    \item The weights of each layer yield $D \left(12D + 2\left(d_w + d_a +d_v +d_g \right) + 19 \right)$ parameters, except for the first layer;
    \item The low-rank MLP for the value residual is not present in the first layer, subtracting $(2Dd_v + D)$ parameters.
\end{itemize}

\paragraph{Parameter Initializations}
\label{sec:initialization}

Proper parameter initialization is crucial for ensuring training stability and achieving optimal performance for language models. RWKV-7 employs a carefully designed initialization strategy tailored to its architecture. The detailed initialization scheme is beyond the scope here but can be found in the official code repository. We emphasize that using the recommended initialization is essential for replicating the results in this paper. Deviations from the prescribed initialization may lead to performance degradation.

\paragraph{Dataset Loading}
\label{sec:dataset_loading}

The dataset used for pretraining consists of $ \num{3119194079123} $ tokens stored on disk, which are memory-mapped using the \texttt{mmap} mechanism. To ensure a diverse and pseudo-random sampling of training sequences, we employ a custom data loading strategy based on a mathematical function with desirable properties. Specifically, we utilize a pseudo-random number generator defined by the function $ f(x) = a x^3 $ over the finite field $ \mathbb{Z}/p\mathbb{Z} $, where $ p $ is a prime number of the form $ 3n+2 $. This function is chosen because it is a bijection (full map) in $ \mathbb{Z}/p\mathbb{Z} $, ensuring that all possible indices are eventually accessed exactly once within one epoch.

For pretraining with a sequence length of $4096$, the relative address of the $k$-th sample is determined as:
\[
\text{start\_address} = 4096 \cdot (a k^3\mod p), \quad \text{end\_address} = \text{start\_address} + 4097.
\]
Here, $p$ is chosen as the largest prime of the form $3n+2$ smaller than $\lfloor \text{dataset\_size} / 4096 \rfloor$, yielding $p = \num{761521949}$. The parameter $a$ is set to an integer close to $0.618 p$ that ensures good mixing properties of the generated addresses.

This approach guarantees both simple calculation and uniform access to the dataset while maintaining pseudo-randomness.
By leveraging the properties of modular arithmetic and cubic mappings, we achieve a balance between computational efficiency and data diversity during pretraining. 

\paragraph{Training Details}
\label{sec:hyperparameters}

All RWKV-7 models were trained under \texttt{bfloat16} format on nodes of $8\times$ Nvidia H800. 
The AdamW optimizer was configured with $\beta_1 = 0.9$, $\beta_2 = 0.99$, $\epsilon = \num{1e-18}$, and a weight decay of $0.1$ applied exclusively to linear layers and embedding weights. The choice of such a small $\epsilon$ value is motivated by the theory proposed by \citet{molybog2023theoryadaminstabilitylargescale}, which suggests that reducing $\epsilon$ can help stabilize training in large-scale models by ensuring that intermediate layers remain in a regime of active updates, thus mitigating sudden loss spikes and promoting smoother convergence.

The context length for pretraining was 4096 tokens. The base decay rate $w_0$ parameters are placed into a special 2x learning rate multiplier grouping.

Besides the traditional cosine learning rate decay schedule, we used a phased dynamic batch size scaling strategy inspired by the concept of critical batch size proposed by \citet{McCandlish2018AnEM} and similar to the approaches in \citet{l.2018dontdecay}. Our strategy involves progressively increasing the batch size during training, accompanied by corresponding adjustments to the learning rate.

The detailed training schedules for different model sizes are listed in Table \ref{tab:training-schedules}.

\begin{table}
\centering
\begin{adjustbox}{max width=\linewidth}
\begin{tabular}{lccccc} 
\toprule
\textbf{Model} & \textbf{Phase} & \textbf{Nodes} & \textbf{Batch size} & \textbf{Proposed Initial LR} & \textbf{Final Loss}\\
\midrule
RWKV7-World3-0.1B & 1 & 1 & $240 \times 4096$ & $\num{6e-4}$ & 2.5290 \\
\midrule
RWKV7-World3-0.4B & 1 & 1 & $240 \times 4096$ & $\num{5e-4}$ \\
  & 2 & 2 & $480 \times 4096$ & $\num{6e-4}$ & 2.2580\\
\midrule
RWKV7-World3-1.5B & 1 & 3 & $480 \times 4096$ & $\num{4e-4}$ \\
  & 2 & 4 & $672 \times 4096$ & $\num{4.5e-4}$ \\
  & 3 & 6 & $1152 \times 4096$ & $\num{6.1e-4}$ & 1.9970 \\
\midrule
RWKV7-World3-2.9B & 1 & 4 & $640 \times 4096$ & $\num{4e-4}$ \\
  & 2 & 6 & $1008 \times 4096$ & $\num{5e-4}$ \\
  & 3 & 7 & $1120 \times 4096$ & $\num{5.4e-4}$ \\
  & 4 & 12 & $2016 \times 4096$ & $\num{8e-4}$  & 1.8745 \\
\bottomrule
\end{tabular}
\end{adjustbox}
\caption{Training Schedules and Batch Sizes}
\label{tab:training-schedules}
\end{table}

The learning rate undergoes a cosine decay schedule from the proposed initial learning rate at the beginning of the entire training run to the expected final learning rate of $ \num{1e-5} $ at the end of the entire run, but the implied initial rate varies across phases.

This approach not only enhances training efficiency but also utilizes GPU resources economically. After smaller models complete their training, additional GPU resources become available for the later stages of training larger models. This cascading resource allocation ensures that computational power is dynamically reallocated, maximizing hardware utilization and reducing idle time.

We observe extremely stable training without any loss spikes in all four runs, indicating that the likelihood of encountering such spikes during the training of a very large 
RWKV-7 model is minimal.

See Figure \ref{fig:trainingloss_all} for the resulting learning rates and observed loss curves.

Despite the general stability of our loss curves, we did sometimes observe NaN loss across a single training step, which we theorize may be due to our use of such an extremely low AdamW $\epsilon$. When this occurs, we rewind the training to the prior checkpoint, clear optimizer states, and continue from that point.

\begin{figure*}[ht!]
    \centering
    \begin{subfigure}[b]{0.45\textwidth}
        \includegraphics[width=\textwidth]{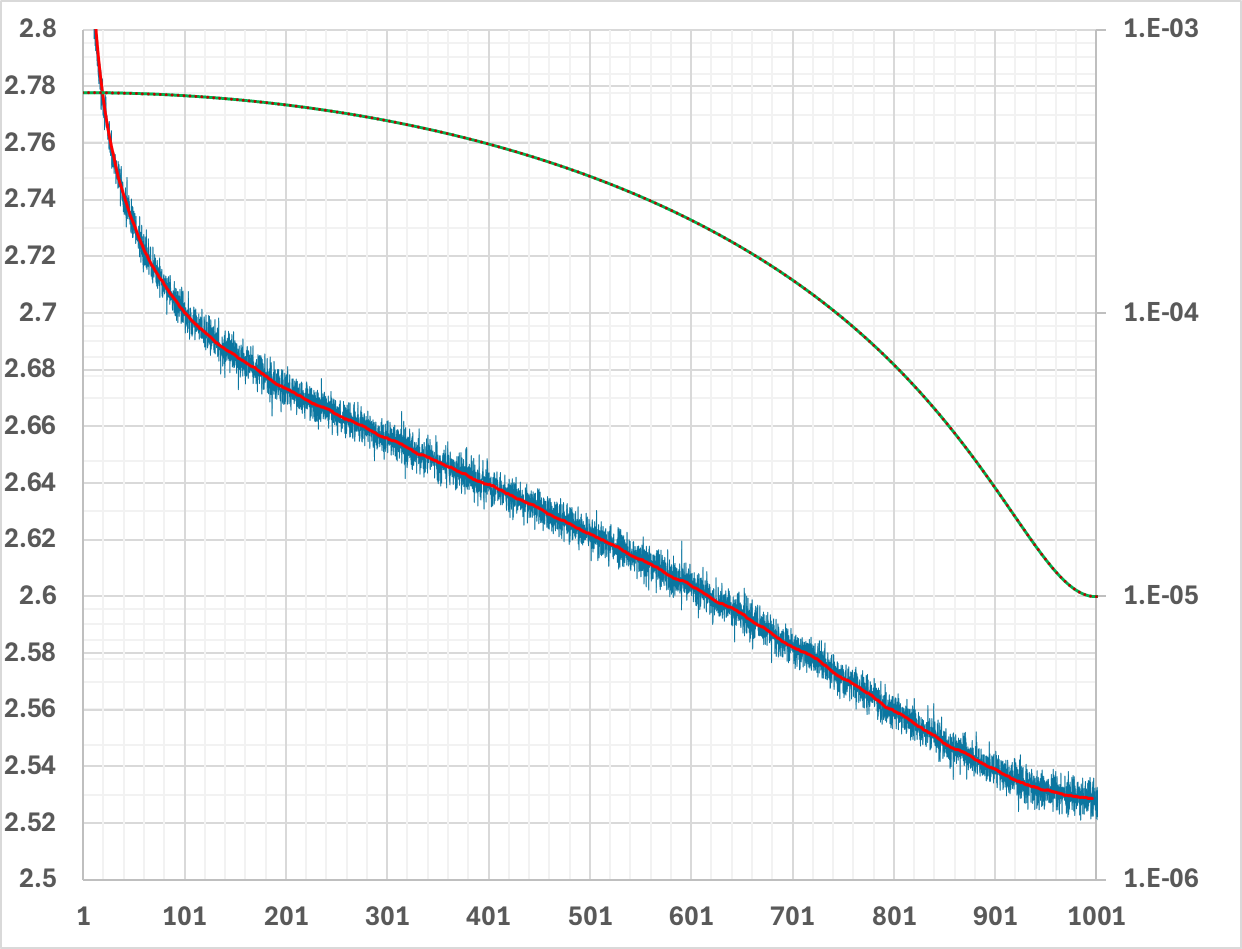}
        \caption{0.1B}
        \label{fig:trainingloss_subfig1}
    \end{subfigure}
    \hfill 
    \begin{subfigure}[b]{0.45\textwidth}
        \includegraphics[width=\textwidth]{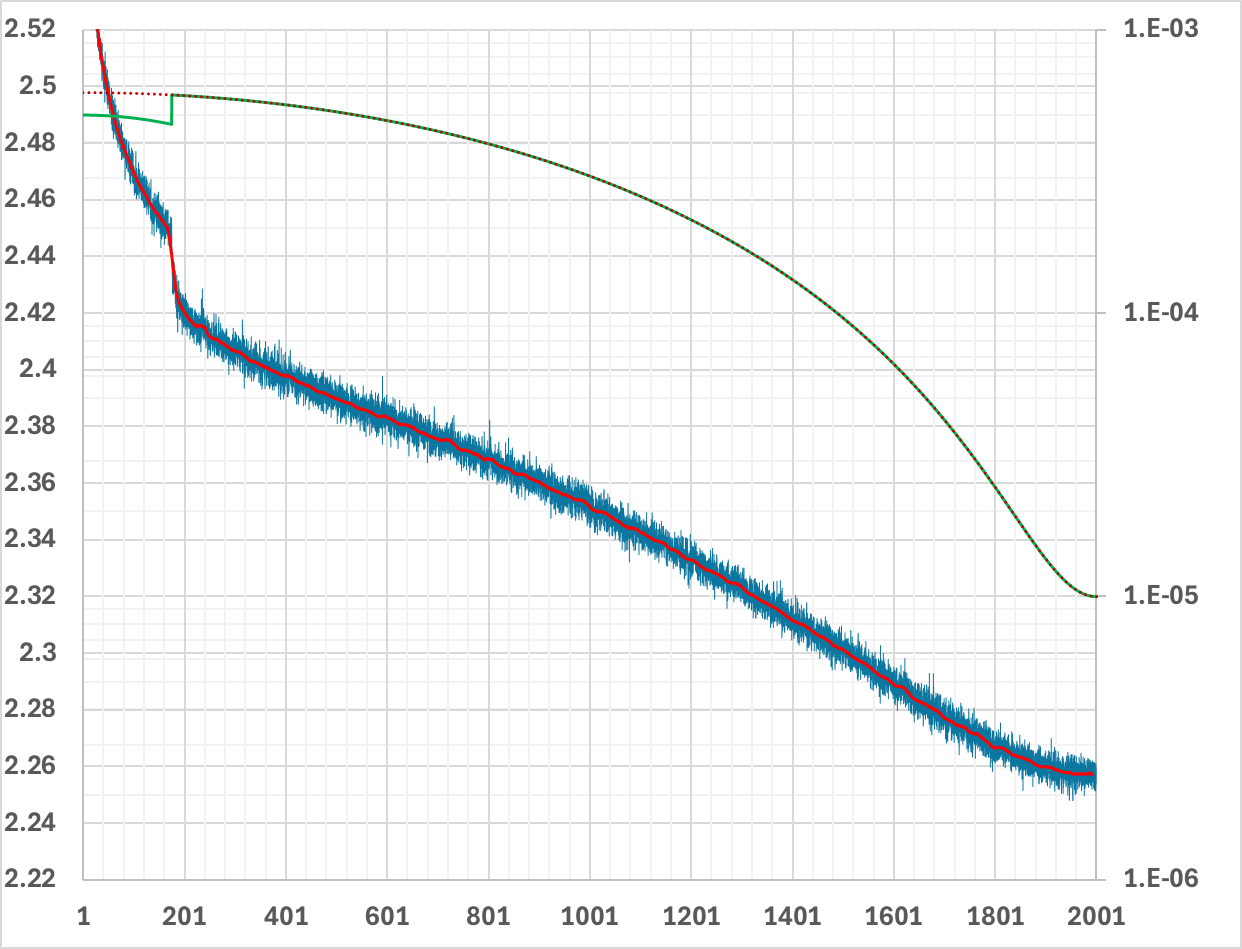}
        \caption{0.4b}
        \label{fig:trainingloss_subfig2}
    \end{subfigure}

    \vspace{1em} 

    \begin{subfigure}[b]{0.45\textwidth}
        \includegraphics[width=\textwidth]{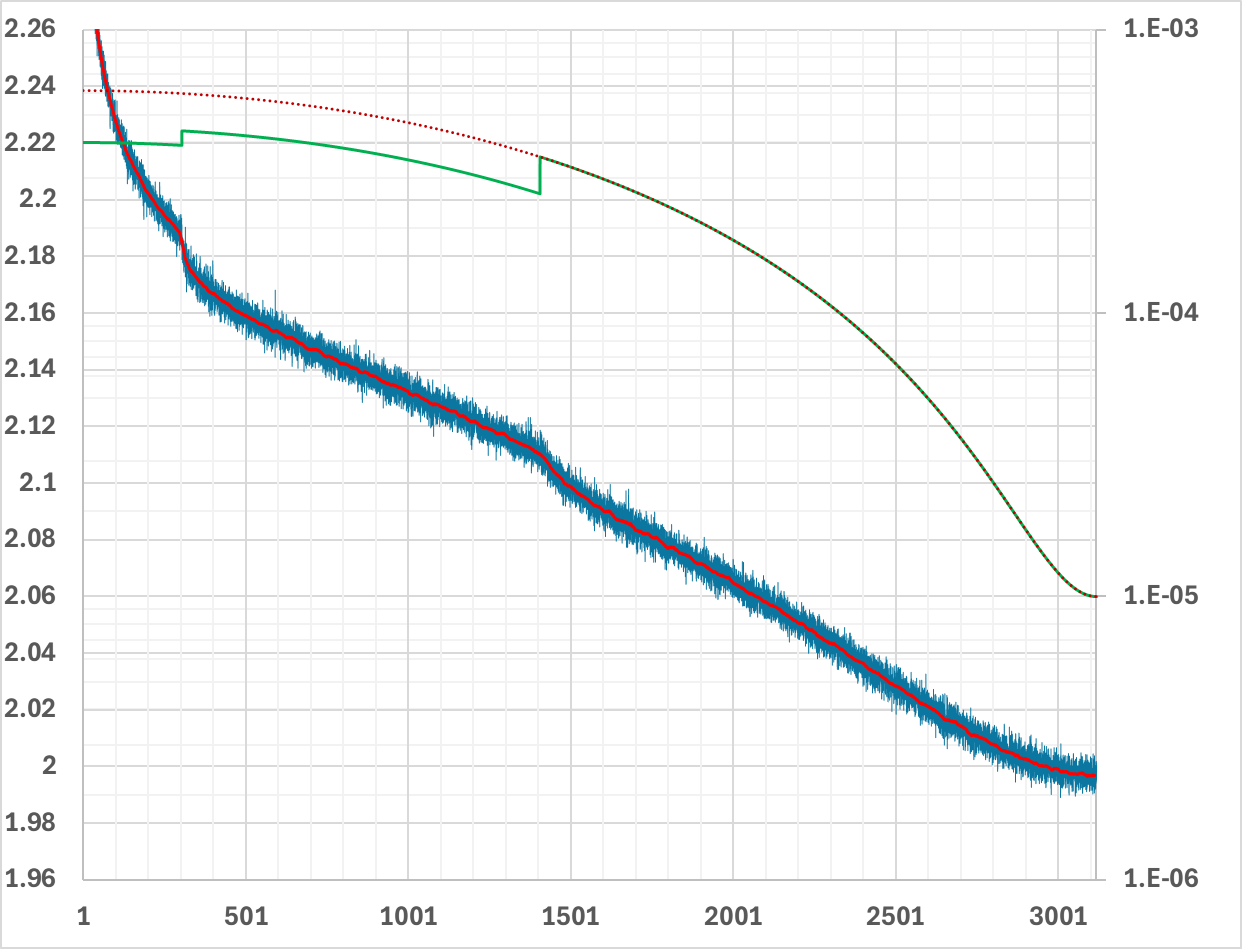}
        \caption{1.5B}
        \label{fig:trainingloss_subfig3}
    \end{subfigure}
    \hfill 
    \begin{subfigure}[b]{0.45\textwidth}
        \includegraphics[width=\textwidth]{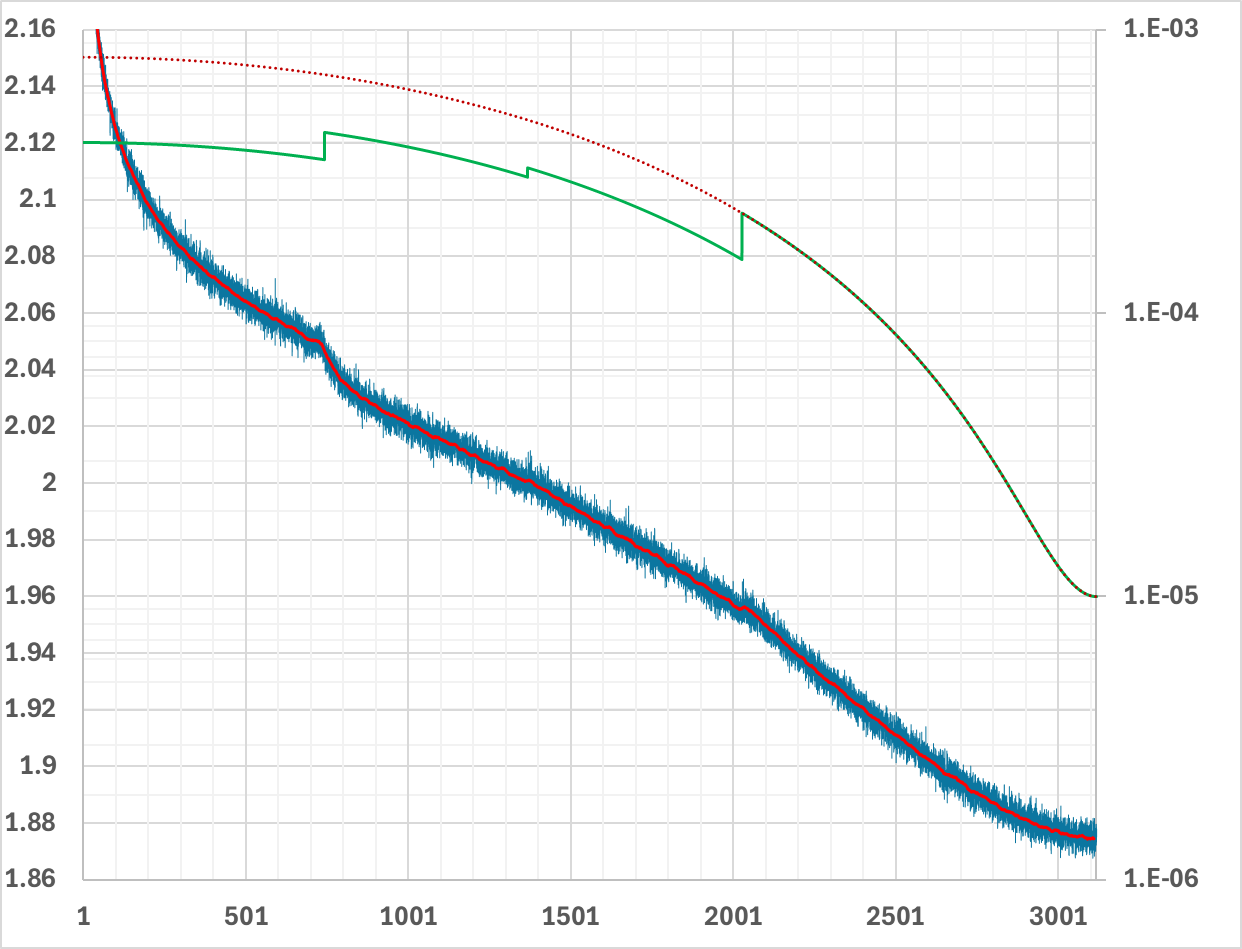}
        \caption{2.9B}
        \label{fig:trainingloss_subfig4}
    \end{subfigure}

    \caption{Training loss curve for RWKV-7 World models. Blue line: loss; Red line: smoothed loss; Green line: actual learning rate; Dotted red line: learning rate schedule.}
    \label{fig:trainingloss_all}
\end{figure*}

\section{Additional Architecture Discussion}
\label{sec:arch_details}

The following is a general overview of the RWKV-7 architecture and selected rationale for its design choices:

The RWKV-7 Time Mixer begins with a feature that has been in RWKV since its inception: token shift. Token shift is a variety of 1D short convolution that is intended to allow the model to create induction heads within a single layer. At its core, token shift is just a linear interpolation between the current and prior tokens on a per channel basis, where the amount per channel is a learned parameter. Many other modern models (e.g. Mamba) have begun including short convolutions before attention replacements (See also \cite{causal-conv1d}). RWKV-6 introduced an advanced new form of token shift that was data dependent. While we found that this was beneficial in terms of loss decrease per step, we made the judgement call that the improvement in training and inference efficiency was not worthwhile. Therefore, RWKV-7 includes only the simple token shift found in RWKV-4 and RWKV-5.

RWKV-7 time mixing follows the overall form of delta rule, applying an SGD-like mechanism to train the state at test time, such that when presented with a key it will respond with an appropriately matching value, like those it has been shown previously. This can be conceptualized as a simple order of operations: 1) decay the state 2) remove part of the old value located at the current key 3) add part of the new value into the current key. Then, a query (we call it receptance) is applied to the state in order to request the value located at a specific key. After that, the remaining operations normalize the returned values to keep numerical sizing consistent, apply gating, and recombine the heads to obtain the output.

In SGD terms, the decay found in models like RWKV-7 and Gated DeltaNet can be thought of as similar to weight decay. RWKV models since RWKV-4 have employed some form of vector-valued decay on the state in place of positional encoding. We continue this tradition with RWKV-7. While vector valued decay is harder to compute efficiently than its scalar valued equivalent, it brings a significant improvement in loss per step. Many other models, like Mamba-2, make the choice to maximize training efficiency and simplify their kernels by using scalar decay, purposely trading quality for speed.

We are forced to limit the maximum decay that can occur in a given time-step both in order to maintain training stability and to assist the creation of fast but numerically stable kernel implementations. We have many varieties of kernel available today, and are still working on new designs with enhanced efficiency and accuracy. Please see Parallelizing Linear Transformers with the Delta Rule over Sequence Length \citep{yang2024_deltanet} and its accompanying blog post series for insightful details on many components of such algorithms. We hope to be able to reduce the decay limit further in future revisions.

This lower bound on the decay multiplier is expressed by the $\mathrm{exp}(-e^{-0.5} \sigma(d_t))$ formula for decay. It is a rearrangement of an original source formula $\mathrm{exp}(-\mathrm{exp}(-0.5 - \mathrm{softplus}(d_t)))$. The outer $\mathrm{exp}(-\mathrm{exp}(x))$ in the original is nearly a flipped version of sigmoid, but with a better gradient. The $-0.5 - \mathrm{softplus}(d_t)$ is a way of limiting the inputs to this sigmoid-like function to be less than -0.5, which results in the final decay being greater than 0.545. This means that decay can remove at most 45.5\% of the pre-existing values in the wkv state from one timestep to the next. More can be then removed by the delta rule mechanism.

The in-context learning rate is usually equivalent to the learning rate in SGD. Ours is a bit less restrictive than the traditional delta rule or SGD would allow. We also make this data dependent and extend it to be vector-valued instead of being merely a scalar. This is allows each key channel to have its own learning rate, dependent on the current token. Note that in TTT, Gated DeltaNet, and Titans although the in-context learning rates are data dependent, they are only scalar valued.

Some of the design of RWKV-7 with regard to the in-context learning rate is theoretically motivated but may not be apparent from the equations. RWKV-6c, a later v6 sub-version with no trained models but which worked well experimentally, kept its state somewhat normalized using a new design. As mentioned early in this paper and consistent with the observations in \cite{yang2024gatedlinearattentiontransformers, chen2024stuffedmambastatecollapse}, there is a fundamental issue with linear attention numerically growing in an unbounded fashion, and the base RWKV-6 revision has this problem. RWKV-6c, however, defeats this issue by ensuring that the amount of key added to the state is never more than the amount removed by decay. It accomplishes this by pre-multiplying the key by one minus the decay before adding it into the wkv state. 

Early versions of RWKV-7 attempted to use similar mathematical formulations to keep everything normalized. But experimentally we found that the model both did best and was most efficient when we allowed it enough mathematical leeway to make these decisions on its own, rather than enforcing them. So instead of pre-multiplying the key, we give the replacement key latitude in its learning rate via the replacement rate booster.

Similarly, the removal key is decoupled from the replacement key. Note that the removal key is normalized. This is important in delta rule variations, because otherwise the outer product of removal key with itself would cause unwanted changes in the amount removed due to the implicit squaring of its length during that process.

Another point that may not be clear upon first examination is the importance of RWKV-7 being outside of the $\mathsf{TC}^0$ complexity class in which transformers, linear attention, and many varieties of SSMs lie. Consider a single layer of the Key Value Cache of a transformer. It is appended to upon each new token processed, and can never be changed. RWKV-7, however, can and does edit its state at each subsequent token. This can include simple operations like replacement, which can be viewed as functioning similarly to a transformer losing a KV Cache entry via some sparsity mechanism e.g. a sliding window. But it can also include complex operations like swapping two entries in the state; operations a transformer cannot do within its immutable KV Cache. These kinds of operations can be used to enact computations on the state itself, which lends greater computational abilities to RWKV-7 when processing a fixed set of inputs. You might think of the RWKV-7 state as being like an internal scratchpad.

There are easy problems that simply cannot be solved by transformers on fixed inputs because they lack this ability. One example is that if you give a transformer an ordered set of items and a list of which ones swap places, it will not be able to tell you which item ends up in which position at the end. Both architectures gain even more power when they are allowed extra outputs (e.g. chain of thought), becoming Turing complete. But this requires costly additional inference-time computation, whereas the RWKV-7 state does not.



A possible way to interpret the designation of RWKV-7 is training a model to learn (how to train an internal model). A WKV head can be viewed as a linear transformation that takes the input (receptance) $r$ and outputs $o$. Every WKV head can be regarded as a linear model that can update its weights as the context progresses. The entries of WKV states are exactly the parameters of these models.

After receptance is applied to the WKV state, we normalize the result on a per head basis. This has become common across many linear attention and post-linear-attention architectures. It is a way of ensuring that change in the numerical size of the state over time does not impact the model's ability to use the state. The original formulations of RWKV-4 used a denominator in place of normalization to achieve a similar effect, but it is costly, harder to code, and uses more memory.

The readout part of RWKV-7 differ from RWKV-6 by the addition of the per-head $\left (r_t (\rho \odot \tilde{k}_t)^T \right)$ scalar gating of $v_t$. The trainable parameter $\rho$ resembles the design of "time-first" $u$ term existing from RWKV-4 to RWKV-6, under the belief that the information of the current token deserves special treatment. The "time-first" term was fused inside the WKV kernel in previous architectures of RWKV, but we decide to extract this term out of the kernel to simplify the implementation of RWKV-7.

\section{Pseudocode For RWKV-7}
\label{sec:pseudocode}

\begin{python}
import math
import torch as th
import torch.nn.functional as F

def rwkv_timemix(params, x, vprime_0, shift_state, wkv_state, layer_id):
    B, T, C = x.shape       # batch_size, sequence_length, d_model
    N = wkv_state.shape[-1] # head size
    H = C // N              # head count
    
    # weight preparation
    x_shifted   = th.cat([shift_state, x[:, :-1, :]], dim=1)
    shift_state = x[:, -1:, :]

    x_receptance = th.lerp(x, x_shifted, params.mu_r)
    x_decay      = th.lerp(x, x_shifted, params.mu_d)
    x_key        = th.lerp(x, x_shifted, params.mu_k)
    x_value      = th.lerp(x, x_shifted, params.mu_v)
    x_iclr       = th.lerp(x, x_shifted, params.mu_a)
    x_gate       = th.lerp(x, x_shifted, params.mu_g)

    r      = x_receptance @ params.W_receptance # BTC
    d      = params.decay_lora(x_decay)         # BTC 
    k      = x_key @ params.W_key               # BTC
    vprime = x_value @ params.W_value           # BTC
    gate   = params.gate_lora(x_gate)           # BTC
    iclr   = params.iclr_lora(x_iclr).sigmoid() # BTC

    # 1st layer: return value to use in later layers, no interpolation
    if layer_id == 0:
        v = vprime_0 = vprime
    else:
        value_residual_gate = th.sigmoid(params.nu_lora(x_value))
        v  = th.lerp(vprime, vprime_0, value_residual_gate)

    decay = th.exp(-math.exp(-0.5) * d.to(th.float).sigmoid())
    removal_k = k * params.removal_key_multiplier
    removal_k = F.normalize(removal_k.view(B,T,H,-1), dim=-1).view(B,T,C)
    replacement_k = th.lerp(k, k * iclr, params.iclr_mix_amt)

    # recurrence relation
    out = th.empty_like(x).view(B,T,C)
    for t in range(T):
        # single step of wkv state transition
        decay_t         = decay[:, t].view(B, H, N, 1)
        iclr_t          = iclr[:, t].view(B, H, N, 1)
        removal_k_t     = removal_k[:, t].view(B, H, N, 1)
        replacement_k_t = replacement_k[:, t].view(B, H, N, 1)
        v_t             = v[:, t].view(B, H, N, 1)
        r_t             = r[:, t].view(B, H, N, 1)
        
        wkv_state = wkv_state * decay_t.mT - wkv_state @ removal_k_t @ (iclr_t * removal_k_t).mT
        wkv_state = wkv_state + v_t @ replacement_k_t.mT
        y = wkv_state @ r_t  # BHVK @ BHK1 = BHV1
        out[:,t] = y.view(B,C) # recombine heads

    # normalization
    y = F.group_norm(y.view(B * T, -1), num_groups=H, params.ln_x.weight, params.ln_x.bias, eps = H * 1e-5).view(B, T, -1)

    # bonus and output
    bonus = ((r*k*params.bonus_multiplier).sum(dim=-1, keepdim=True) * v)
    bonus = bonus.view(B,T,C)   # recombine heads
    out = out + bonus
    out = (out * gate) @ params.W_output # BTC

    return out, v0, shift_state, wkv_state

def rwkv_channelmix(x, shift_state):
    x_shifted = th.cat([shift_state, x[:, :-1, :]], dim=1)
    shift_state = x[:, -1:, :]
    xk = th.lerp(x, x_shifted, params.mu_x)
    k = params.W_k @ xk
    v = params.W_v @ th.relu(k).square()
    return v, shift_state

def rwkv_model(params, input_ids, state):
    x = params.embedding(input_ids)
    x = params.layer_norm_pre(x)

    v0 = None
    layer_id = -1
    for layer in params.layers:
        layer_id = layer_id + 1
        dx, v0, state.timemix_shiftstate, state.timemix_wkvstate = rwkv_timemix(
            layer.time_mix, 
            layer.layer_norm_timemix(x), 
            v0, 
            state.timemix_shiftstate, 
            state.timemix_wkvstate,
            layer_id
        )
        x = x + dx
        dx, state.chanmix_shiftstate = rwkv_chanmix(
            layer.channel_mix, 
            layer.layer_norm_chanmix(x), 
            state.chanmix_shiftstate
        )
        x = x + dx

    x = params.layer_norm_out(x)
    logits = params.head(x)
    return logits, state

\end{python}

\section{PyTorch code For Naive WKV7 Kernel (Forward and Backward)}
\label{sec:pseudocode_wkv7}
\begin{python}

class WKV7_Kernel(nn.Module):
    def __init__(self):
        super().__init__()
    
    def forward(self, r, w, k, v, a, b, state):
        r = r.view(B, T, H, N)
        k = k.view(B, T, H, N)
        v = v.view(B, T, H, N)
        a = a.view(B, T, H, N)
        b = b.view(B, T, H, N)
        self.state_cache = torch.zeros((B, T+1, H, N, N))
        self.state_cache[:, 0, :] = state
        w = torch.exp(-torch.exp(w.view(B, T, H, N)))
        out = torch.zeros((B, T, H, N))
        for t in range(T):
            kk = k[:, t, :]
            rr = r[:, t, :]
            vv = v[:, t, :]
            aa = a[:, t, :]
            bb = b[:, t, :]
            state = (
                state * w[: , t, :, None, :]
                + torch.einsum('bhik,bhk,bhj->bhij', state, aa, bb) 
                + torch.einsum('bhj,bhi->bhij', kk, vv)
            )
            self.state_cache[:, t+1, :] = state
            out[:, t, :] = torch.einsum('bhj,bhij->bhi', rr, state)
        return out, state
    
    def backward(self, r, w0, k, v, a, b, gout, gstate):
        gout = gout.view(B, T, H, N)
        gr = torch.zeros((B, T, H, N))
        gw = torch.zeros((B, T, H, N))
        gk = torch.zeros((B, T, H, N))
        gv = torch.zeros((B, T, H, N))
        ga = torch.zeros((B, T, H, N))
        gb = torch.zeros((B, T, H, N))
        w = torch.exp(-torch.exp(w0.view(B, T, H, N)))
        for t in range(T-1, -1, -1):
            gr[:, t, :] = torch.matmul(
            gout[:,t,:,None,:], self.state_cache[:,t+1,:]).squeeze(-2)
            gstate.add_(torch.matmul(
            gout[:,t,:,:,None], r[:,t,:,None,:]))
            gk[:, t, :] = torch.matmul(
            v[:,t,:,None,:], gstate).squeeze(-2) 
            gv[:, t, :] = torch.matmul(
            gstate, k[:,t,:,:,None]).squeeze(-1)
            ga[:, t, :] = torch.einsum(
            'bhik,bhj,bhij->bhk', 
            self.state_cache[:, t, :], b[:, t, :], gstate)
            gb[:, t, :] = torch.einsum(
            'bhik,bhk,bhij->bhj', 
            self.state_cache[:, t, :], a[:, t, :], gstate)
            gw[:, t, :] = torch.einsum(
            'bhij,bhij->bhj', 
            self.state_cache[:, t, :], gstate)
            gstate      = torch.einsum(
            'bhj,bhij->bhij', w[:, t, :], gstate) \
            + torch.einsum(
            'bhk,bhj,bhij->bhik', a[:, t, :], b[:, t, :], gstate)
        gw = -torch.exp(w0-torch.exp(w0)) * gw
        return gr, gw, gk, gv, ga, gb, gstate
\end{python}

\section{Board Game Modeling}
\label{sec:rwkv_othello}

Previous research (\citet{schultz2024masteringboardgamesexternal, topsakal2024evaluatinglargelanguagemodels}) has shown that board games can serve as potentially effective tools for evaluating a model's capabilities. As an RNN with powerful state tracking abilities, RWKV-7 is highly suitable for board game modeling and conducting extended searches directly within context to find better strategies. As an early exploration, we tested RWKV-7's board game modeling capabilities and its ability to perform in-context search on the game of Othello (Reversi). We found an expanded formulation of RWKV-7 to be useful for this task, which we designate as RWKV-7a: $\bm{S}_t = \bm{S}_{t-1} \mathrm{diag}(w_t)(I - c\hat{\kappa}_t^T (a_t \odot \hat{\kappa}_t))  + v_t^T k_t $ This formula allows the full range of (-1,1) eigenvalues when $c=2$.

\paragraph{Data}
We designed training data that guides the model to predict legal moves, evaluate these moves, and perform Alpha-Beta pruning. Each training sample consists of three components:
\begin{itemize}
    \item \textbf{Input section}: Captures the game state, including the current board position, active player, and search parameters (tree depth and width settings).
    \item \textbf{Reasoning section}: Varies with search settings:
    \begin{itemize}
        \item Without search (depth or width = 1): Lists legal moves and their evaluations
        \item With search (depth and width > 1): Performs Alpha-Beta pruning to find optimal moves
    \end{itemize}
    All move evaluations in the reasoning section were generated using the Egaroucid engine \citep{Yamana_Egaroucid_2025}.
    \item \textbf{Output section}: This contains the final move decision and displays the resulting board position after implementing the move.
\end{itemize}
An example of our training data is shown in sample \ref{lst:othello_data_example}.

\begin{lstlisting}[caption={Training Data Example}, label={lst:othello_data_example}]
<input>
. . o o * * . . 
* . o o o * * . 
* * * o * * o o 
* * * * * o o o 
* o * o o * o o 
* * o o o * o . 
* . o o o o . . 
. . . . o o . . 
NEXT o 
MAX_WIDTH-2
MAX_DEPTH-2
</input>
<reasoning>
Possible moves and score: g1 -19 h1 -01 b2 -08 h2 -23 b7 -12 g7 -09
<stack>
Remaining_Depth:2
Max_Node Alpha: -in Beta: +in Best: -- Current: h1 -01 Unexplored: b2 -08
</stack>
=> Search next node
[Depth limit not reached]
<board>
. . o o * * . o 
* . o o o * o . 
* * * o * o o o 
* * * * o o o o 
* o * o o * o o 
* * o o o * o . 
* . o o o o . . 
. . . . o o . . 
</board>
NEXT * 
Possible moves and score: b1 +02 b2 +05 h2 +10 h6 +03 h7 +08 c8 +06 d8 +01 g8 +09
[Current player has legal moves]
[Internal node - expand]
<stack>
Remaining_Depth:1
Min_Node Alpha: -in Beta: +in Best: -- Current: d8 +01 Unexplored: b1 +02
Max_Node Alpha: -in Beta: +in Best: -- Current: h1 -01 Unexplored: b2 -08
</stack>
=> Search next node
[Depth limit reached - evaluate all leaves]
[Updated stack]
<stack>
Remaining_Depth:1
Min_Node Alpha: -in Beta: +01 Best: d8 Current: -- --- Unexplored:
Max_Node Alpha: +01 Beta: +in Best: h1 Current: b2 -08 Unexplored:
</stack>
=> Search next node
[Depth limit not reached]
<board>
. . o o * * . . 
* o o o o * * . 
* o o o * * o o 
* o * o * o o o 
* o * o o * o o 
* * o o o * o . 
* . o o o o . . 
. . . . o o . . 
</board>
NEXT * 
Possible moves and score: a1 -07 b1 +13 h6 -01 b7 -08 g7 +08 h7 -02 c8 +01 d8 -03 g8 +04
[Current player has legal moves]
[Internal node - expand]
<stack>
Remaining_Depth:1
Min_Node Alpha: +01 Beta: +in Best: -- Current: b7 -08 Unexplored: a1 -07
Max_Node Alpha: +01 Beta: +in Best: h1 Current: b2 -08 Unexplored:
</stack>
=> Search next node
[Depth limit reached - evaluate all leaves]
[Updated stack]
<stack>
Remaining_Depth:1
Min_Node Alpha: +01 Beta: -08 Best: b7 Current: -- --- Unexplored:
Max_Node Alpha: +01 Beta: +in Best: h1 Current: -- --- Unexplored:
</stack>
[End of search]
> Playing h1 
</reasoning>
<output>
 h1 
. . o o * * . o 
* . o o o * o . 
* * * o * o o o 
* * * * o o o o 
* o * o o * o o 
* * o o o * o . 
* . o o o o . . 
. . . . o o . . 
</output>
\end{lstlisting}

\paragraph{Training}
We trained RWKV-7 models with 9M and 26M parameters respectively on 6 million samples. By tracking the loss across different token types during training (figure \ref{fig:othello_loss}), we noticed that the model first mastered output formatting, then developed board state tracking capability, and continuously improved its evaluation accuracy throughout training.
\begin{figure*}[ht!]
    \centering
    \includegraphics[width=\linewidth]{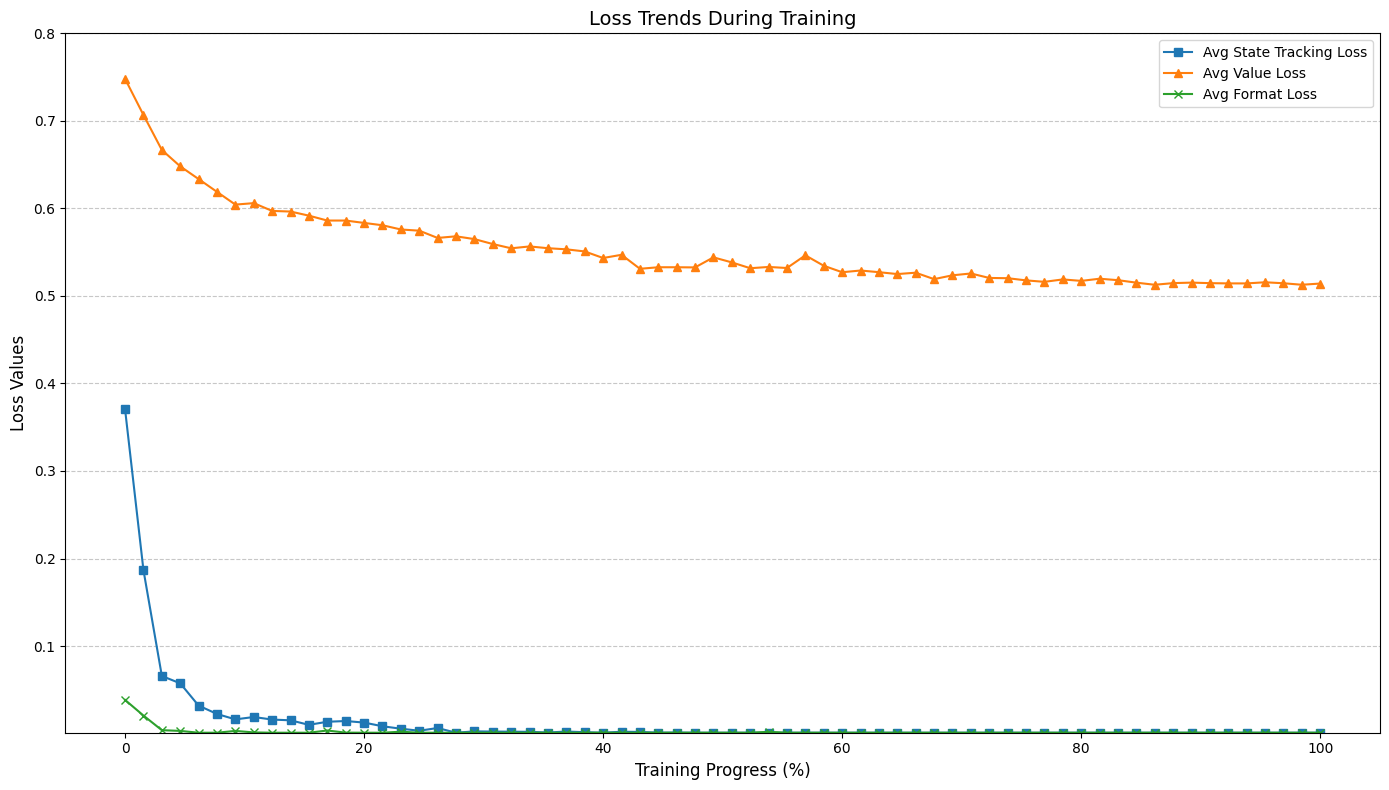}
    \caption{Reversi Training loss for different token types}
    \label{fig:othello_loss}
\end{figure*}

\paragraph{Evaluation}
We control the model's thinking budget by setting the width and depth of Alpha-beta pruning, and test the win rate against baseline a model (depth=1, width=1) under different budgets. As shown in figure \ref{fig:othello_tts}, by increasing the testing budget, RWKV-7 can effectively search for better strategies, demonstrating a positive test-time scaling law on this task.
\begin{figure*}[ht!]
    \centering
    \includegraphics[width=\linewidth]{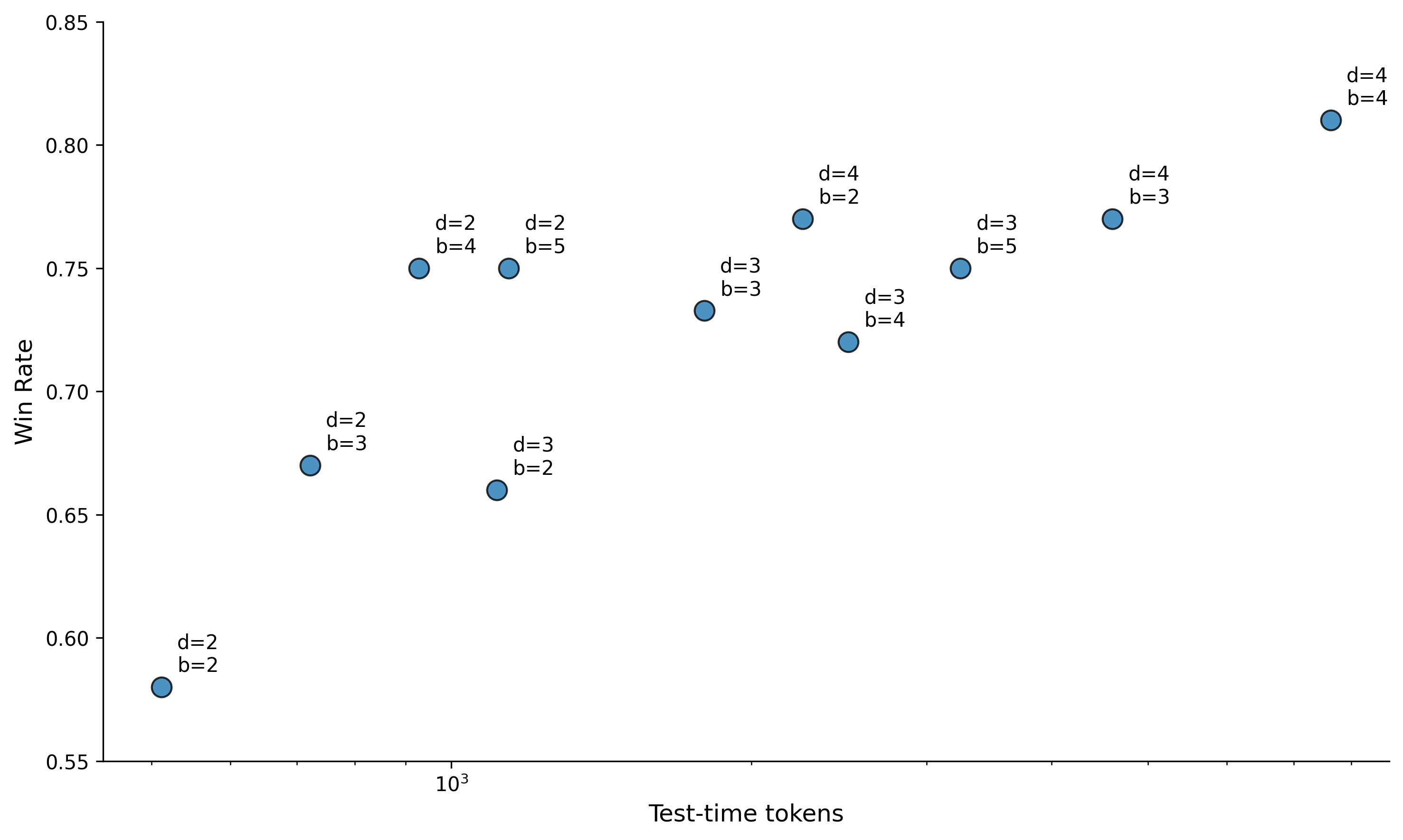}
    \caption{Reversi Token consumption and win rates under different search configurations}
    \label{fig:othello_tts}
\end{figure*}

\section{State Inspections}
\label{sec:state_inspections}

We inspected the WKV matrix states of RWKV-7 and compared them with those of RWKV-5 and RWKV-6. We analyzed the following aspects: 

\begin{enumerate}
    \item Visualization example of WKV state matrices, to better understand the structure and behavior of the matrices.
    \item Root Mean Square (RMS) of matrix elements, to assess the numerical stability of those WKV matrices.
    \item Stable Rank (SR) of the matrices \citep{rudelson2007samplingfromlargematrices}, defined as the square of (the Frobenius norm divided by the spectral norm): 
    $$\mathrm{SR} (A) := \left(\frac{\left\lVert A \right\rVert _F}{\left\lVert A \right\rVert _2} \right) ^2 . $$
    This serves as a rough measure to the effective amount of information of the states.
\end{enumerate}

For this analysis, we selected 10 samples from the validation set of the PG19 dataset \citep{rae2019compressivetransformerslongrangesequence}, ensuring that each sample had a sequence length of at least 8192 tokens. We tested the 1.5B parameter versions of RWKV-5, RWKV-6, and RWKV-7, plotting on the appearance, stable rank and RMS values of their WKV matrices.

\begin{figure*}[ht!]
    \centering
    \begin{subfigure}[b]{0.95\textwidth}
        \includegraphics[width=\textwidth]{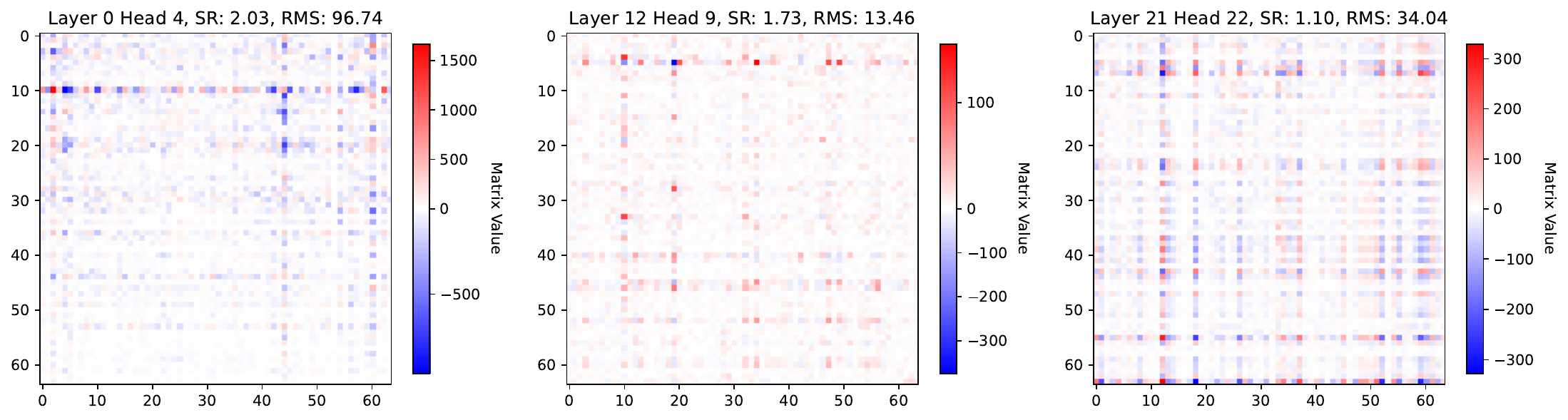}
        \caption{RWKV-5 1.5B}
        \label{fig:statesubfig1}
    \end{subfigure}

    \vspace{1em} 

    \begin{subfigure}[b]{0.95\textwidth}
        \includegraphics[width=\textwidth]{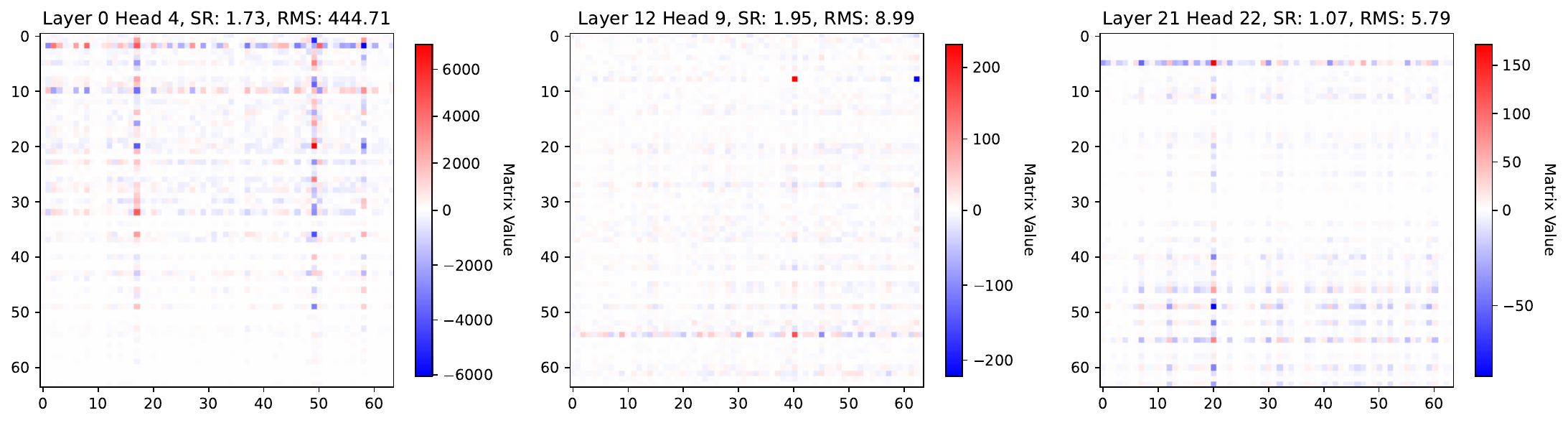}
        \caption{RWKV-6 1.6B}
        \label{fig:statesubfig2}
    \end{subfigure}

    \vspace{1em} 
    
     \begin{subfigure}[b]{0.95\textwidth}
        \includegraphics[width=\textwidth]{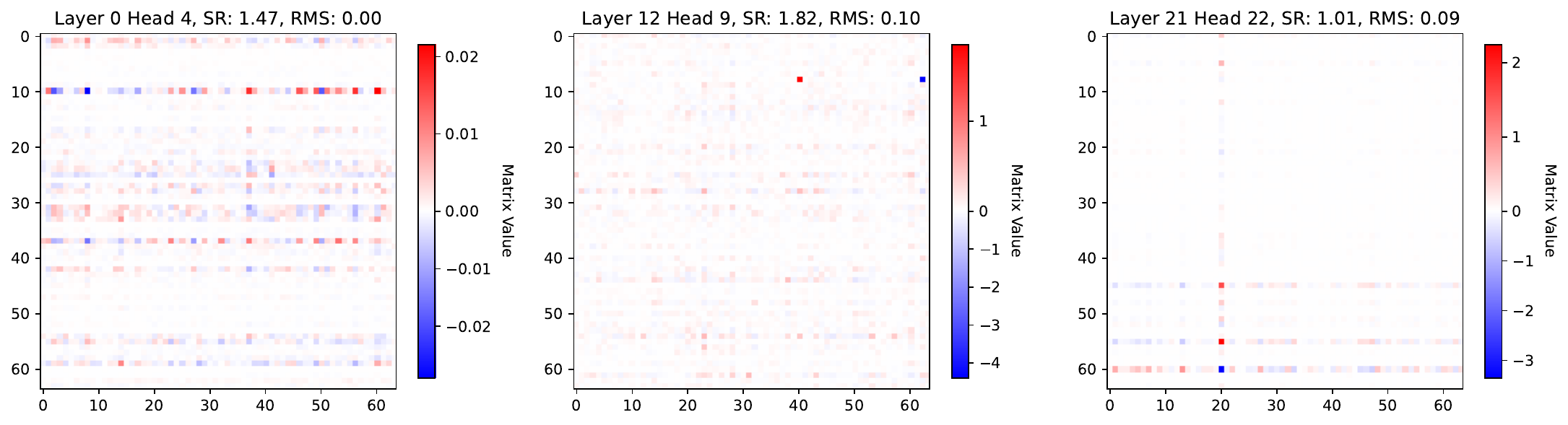}
        \caption{RWKV-7 1.5B}
        \label{fig:statesubfig3}
    \end{subfigure}

    \caption{Visualization example of RWKV's WKV matrices.}
    \label{fig:rwkv_state_inspect_appearance}
\end{figure*}

We observed that the WKV states of RWKV-7 had significantly smaller RMS values compared to RWKV-5 and RWKV-6. The entries of the WKV matrix in RWKV-7 were consistently of order $ O(1) $ (i.e., no outliers, and does not grow over context length), whereas RWKV-5 and RWKV-6 constantly produced outliers on the order of thousands (see Figures \ref{fig:rwkv_state_inspect_appearance} and \ref{fig:rwkv_state_rms_sr}). This indicates that RWKV-7 has better numerical stability during training and inference.

\begin{figure}[ht!]
    \centering
    \includegraphics[width=\textwidth]{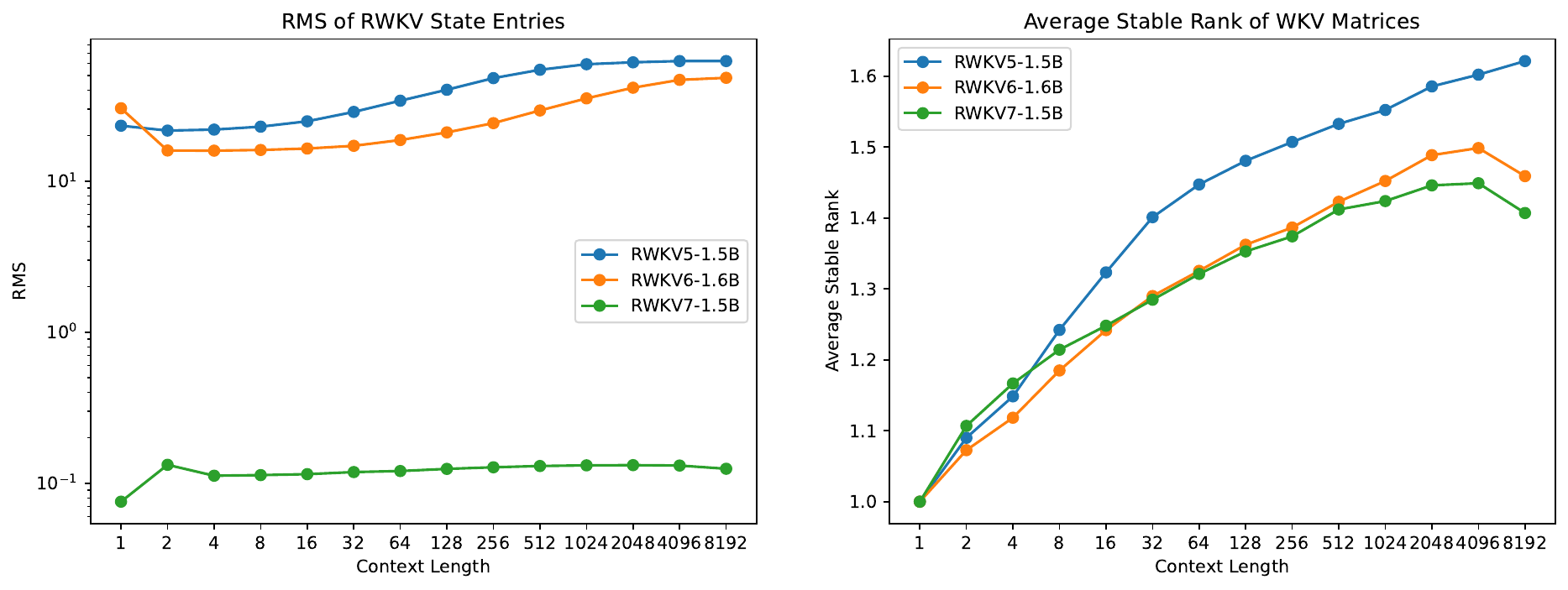}
    \caption{The global RMS and average stable rank of WKV matrices, plotted over sequence length.}
    \label{fig:rwkv_state_rms_sr}
\end{figure}

Interestingly, the stable rank of the WKV matrix in RWKV-7 has shown to be lower than that of RWKV-5 and RWKV-6 for context longer than 32. A lower stable rank typically suggests that the matrix contains less information or has a more compressed representation. However, this observation appears to contradict the experimental results showing that RWKV-7 performs better on tasks requiring long-term memory. We hypothesize that this contradiction can be explained by RWKV-7's enhanced state evolution mechanism, which enables RWKV-7 to achieve stronger information compression and utilization capabilities and allowing it to maintain important information in a more compact form. This dual capability likely contributes to both the reduced stable rank and the improved performance on memory-intensive tasks.

\section{Ablation Experiments}
\paragraph{The Pile}
\label{subsec:ablation-pile}
To demonstrate the architectural advantages of RWKV-7, we conducted ablation experiments by training models of three different sizes—168M, 421M, and 1.47B parameters—on the full Pile dataset \citep{gao2020pile}. 


\begin{table}[ht]
    \setlength\extrarowheight{-5pt}
    \setlength\tabcolsep{5pt}
    \centering
    \begin{adjustbox}{max width=\linewidth}
    \begin{tabular}{lrrrrrrrrrrr}
\toprule
    \bf Model & \bf Tokens & \bf lmb.o & \bf lmb.o & \bf hella & \bf piqa & \bf arcE & \bf arcC & \bf glue & \bf WG & \bf sciq & \bf avg \\
          & (B) & ppl $\downarrow$ & acc $\uparrow$ & acc\_n $\uparrow$ & acc $\uparrow$ & acc $\uparrow$ & acc $\uparrow$ & acc $\uparrow$ & acc $\uparrow$ & acc $\uparrow$ & acc $\uparrow$ \\
\midrule
    RWKV4-169M-Pile & 332 & 29.2 & 33.2 & 32.2 & 64.8 & 47.1 & 19.9 & 47.6 & 51.2 & 77.6 & 46.7 \\
    Pythia-160M & 300 & 37.3 & 35.4 & 30.3 & 62.3 & 43.6 & 19.5 & 46.5 & 51.3 & 75.4 & 45.5 \\
    Mamba-130M & 300 & 16.0 & 44.3 & 35.3 & 64.5 & 48.0 & 19.7 & 48.5 & 52.1 & 78.2 & 48.8 \\
    Mamba2-130M & 300 & 16.8 & 43.9 & 35.3 & 64.9 & 47.4 & \bf 20.9 & 45.8 & \bf 52.6 & 81.0 & 49.0 \\
    RWKV6-173M-Pile & 332 & 16.0 & 44.5 & 34.9 & 64.4 & \bf 48.3 & 19.7 & 48.9 & 51.9 & 80.6 & 49.2 \\
    RWKV7-168M-Pile & 332 & \bf 14.2 & \bf 45.7 & \bf 36.9 & \bf 65.5 & 47.9 & 19.7 & \bf 49.1 & 52.4 & \bf 81.6 & \bf 49.8 \\
    \midrule

    RWKV4-430M-Pile & 332 & 13.1 & 45.1 & 40.8 & 67.7 & 52.8 & 24.1 & 49.4 & 52.0 & 80.7 & 51.6 \\
    Pythia-410M & 300 & 10.8 & 51.6 & 40.6 & 66.7 & 51.9 & 21.4 & 44.1 & 53.3 & 81.5 & 51.4 \\
    Mamba-370M & 300 & 8.1 & 55.6 & 46.5 & 69.5 & 55.0 & 25.0 & 46.8 & 55.5 & 84.5 & 54.8 \\
    Mamba2-370M & 300 & 8.0 & 55.9 & 46.9 & \bf 70.5 & 54.8 & \bf 25.1 & 48.1 & 55.4 & 85.3 & 55.2 \\
    RWKV7-421M-Pile & 332 & \bf 7.2 & \bf 57.9 & \bf 48.0 & 69.3 & \bf 56.3 & 23.5 & \bf 50.3 & \bf 56.4 & \bf 85.9 & \bf 56.0 \\
    
    \midrule

    RWKV4-1.5B-Pile & 332 & 7.1 & 56.4 & 52.8 & 72.2 & 60.7 & 24.9 & \bf 50.5 & 54.3 & 85.8 & 57.2 \\
    Pythia-1.4B & 300 & 6.1 & 61.7 & 52.0 & 70.8 & 60.5 & 26.1 & 47.7 & 57.5 & 86.6 & 57.9 \\
    Mamba-1.4B & 300 & 5.0 & 65.0 & 59.1 & \bf 74.2 & \bf 65.5 & 29.8 & 46.2 & 61.4 & 87.3 & 61.1 \\
    Mamba2-1.3B & 300 & 5.0 & 65.6 & 59.9 & 73.2 & 64.2 & 29.9 & 46.1 & 61.0 & 89.8 & 61.2 \\
    RWKV7-1.47B-Pile & 332 & \bf 4.8 & \bf 67.0 & \bf 61.8 & 73.6 & 64.9 & \bf 30.2 & 48.0 & \bf 64.4 & \bf 91.1 & \bf 62.6 \\
\bottomrule    
    \end{tabular}
\end{adjustbox}
    
  \caption{\centering{English Focused Benchmarks, including LAMBADA (\textbf{lmb.o}) \citep{paperno2016lambada}, Hellaswag (\textbf{hella}) \citep{zellers2019hellaswagmachinereallyfinish}, PIQA \citep{bisk2020piqa}, AI2 ARC (\textbf{arcE}, \textbf{arcC}) \citep{bhakthavatsalam2021think}, GLUE \citep{wang2018glue}, Winogrande (\textbf{WG}) \citep{sakaguchi2021winogrande}, SciQ \citep{welbl2017crowdsourcing}.}}
  \label{tab:pile}%

\end{table}%

These results highlight the consistent improvements brought by the RWKV-7 architecture over earlier RWKV models, even when trained on the same dataset. As all RWKV models shown were trained under identical configurations and dataset, this underscores the inherent architectural advantages of RWKV-7 over its predecessors. Notably, the performance gap sustains as the model size increases, suggesting that RWKV-7 may scale more effectively than its predecessors.


\paragraph{Architecture Choice Ablations}
\label{sec:ablation-minipile}

We ran a series of ablation experiments to validate the various improvements made in RWKV-7 versus some of the more restrictive choices seen in other DeltaNet and post-DeltaNet related work. These improvements were:
\begin{itemize}
    \item using a vector-valued decay $w$ instead of a scalar-valued decay;
    \item using a vector-valued in-context learning rate $a$ instead of a scalar-valued rate;
    \item using different removal $\kappa$ and replacement $\tilde{k}$ keys, instead of the same for both; and
    \item adding the bonus term $u_{t,j}v_{t,j}$ to the output step (Equation \ref{subeq:addbonus}).
\end{itemize}

We trained a small 6-layer, $d_{model}=768$ Goose model on the 1.6B token minipile~\citep{kaddour2023minipile} dataset at context length 512 and obtained the loss results shown in Table~\ref{tab:abla_specific}.

\begin{table}[ht]
    \setlength\extrarowheight{-5pt}
    \setlength\tabcolsep{2pt}
    \centering
    \begin{tabular}{lrr}
\toprule
Model & Training Loss & Validation Loss \\
\midrule
Goose & 2.834 & 2.541 \\
Goose, scalar decay & 2.873 & 2.609 \\
Goose, scalar in-context learning rate & 2.843 & 2.591 \\
Goose, same removal/replacement keys & 2.840 & 2.560 \\
Goose, no bonus term & 2.841 & 2.588 \\
\bottomrule
    \end{tabular}
  \caption{\centering{Ablation Results for 6 layer 768 dimension Goose model}}
  \label{tab:abla_specific}
\end{table}

\section{Parameters Statistics}
\label{sec:param_stat}
In Section \ref{sec:method}, the actual ranges and statistical metrics of the parameters within the trained model are not specified. To facilitate a better understanding of the role of these parameters in the practical models, this appendix provides empirical statistical metrics of selected parameters from the released RWKV-7 model.
\begin{figure}[ht!]
    \centering
    \includegraphics[width=0.75\columnwidth]{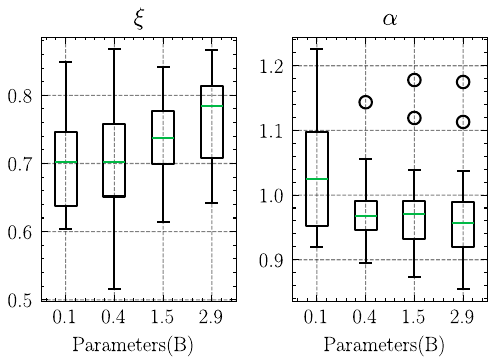}
    \caption{Box Plots of $\xi$ and $\alpha$ Across Models}
    \label{fig:boxplot_xi_alpha}
\end{figure}
\begin{figure}[ht!]
    \centering
    \includegraphics[width=0.95\columnwidth]{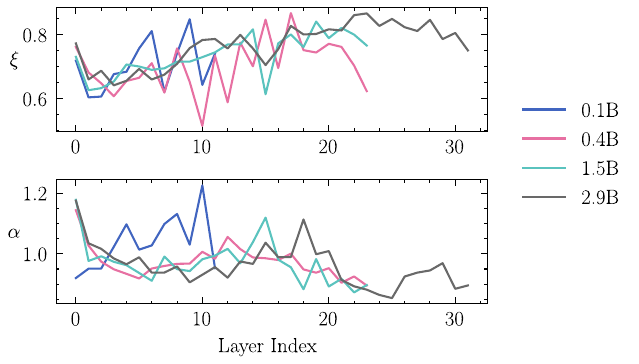}
    \caption{Mean $\xi$ and $\alpha$ Across Layers for Different Models}
    \label{fig:mean_trend_xi_alpha}
\end{figure}
\begin{figure}[ht!]
    \centering
    \includegraphics[width=0.95\columnwidth]{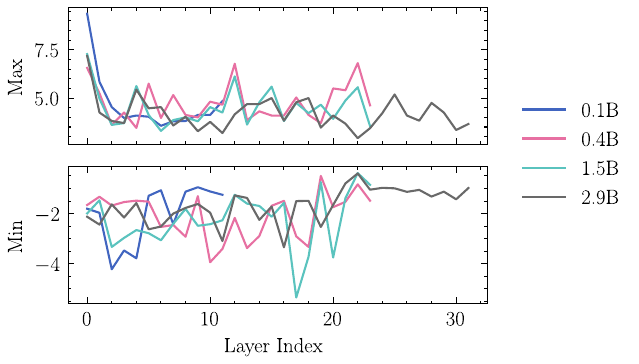}
    \caption{Maximum and Minimum of $\xi$ Across Layers in Different Models}
    \label{fig:minmax_xi}
\end{figure}
\begin{figure}[ht!]
    \centering
    \includegraphics[width=0.95\columnwidth]{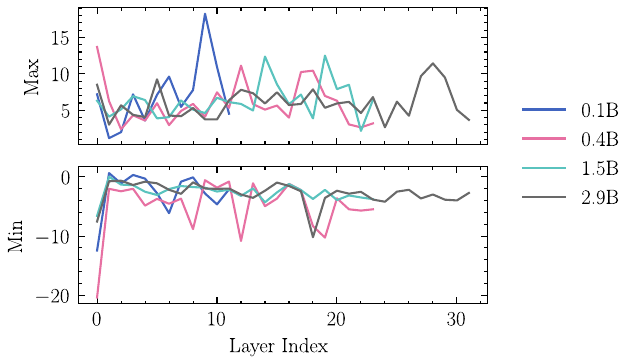}
    \caption{Maximum and Minimum of $\alpha$ Across Layers in Different Models}
    \label{fig:minmax_alpha}
\end{figure}
\begin{figure}[ht!]
    \centering
    \includegraphics[width=0.75\columnwidth]{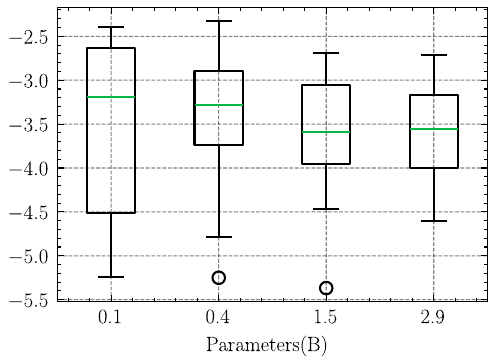}
    \caption{Box Plots of biases of $d_t$ Across Models}
    \label{fig:boxplot_w0}
\end{figure}
\begin{figure}[ht!]
    \centering
    \includegraphics[width=0.95\columnwidth]{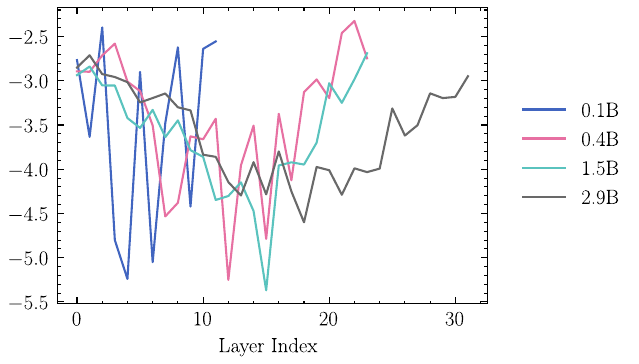}
    \caption{Mean biases of $d_t$ Across Layers for Different Models}
    \label{fig:mean_trend_w0}
\end{figure}
\begin{figure}[ht!]
    \centering
    \includegraphics[width=0.95\columnwidth]{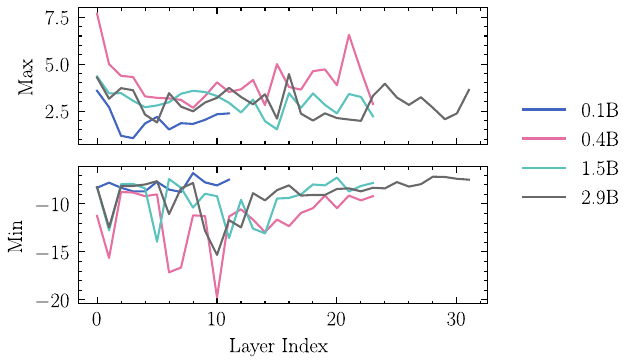}
    \caption{Maximum and Minimum of biases of $d_t$ Across Layers in Different Models}
    \label{fig:minmax_w0}
\end{figure}

\section{Initial Token Sensitivity}
\label{sec:initial_sensitivity}


In our evaluation of LAMBADA \citep{paperno2016lambada}, we found that RWKV-7's performance varies significantly under different settings. After ruling out precision issues, we investigated the consistency of the input and discovered that omitting the special token \texttt{<|endoftext|>} at the beginning of the input caused substantial and statistically significant differences in perplexity (PPL) and accuracy (ACC) for some RWKV-7 models.

Previous research highlights that Transformer models are sensitive to special tokens, and fine-tuning these tokens can yield notable improvements \citep{yang2023parameterefficienttuningspecialtoken}. However, to our knowledge, no quantitative study has examined this effect systematically.

We analyzed the cases with the largest performance discrepancies, and identified a key pattern:
\begin{itemize}
    \item The answer appears as the first word of the paragraph.
    \item This first word does not reappear elsewhere in the text.
\end{itemize}
This suggests that the model may struggle to retain the first token in memory.

In the LAMBADA test set, we identified 142 such examples out of a total of 5153 questions. One example is:
\begin{tcolorbox}[]
\begin{text}
Beth smoothed her wiry half-black, half-gray hair from her makeup-free face. In New Mexico, the natural look was common. Standing next to Cindy Fanucci, she felt like a disaster. She hid her ragged nails under the sleeves of her sweatshirt.

"Hi, I'm Cindy. It's so nice to meet you, \textbf{Beth.}"
\end{text}
\end{tcolorbox}

The following table summarizes the performance of different models with and without the padding token \texttt{<|endoftext|>}:
\begin{table}[ht!]
\centering
\begin{tabular}{lcccc}
\toprule
\bf Model & \bf EOS padding & \bf PPL & \bf ACC (\%) & \bf Significance \\
\midrule
RWKV7 World 0.1B & 0 & 357 & 9.2 & $^{***}$ \\
           & 1 & 16.4 & 36.6 &  \\
\midrule
RWKV7 World 0.4B & 0 & 42.7 & 28.9 & $^{***}$ \\
           & 1 & 7.25 & 48.6 &  \\
\midrule
SmolLM2 360M & 0 & 21.1 & 39.4 & $^{*}$ \\
            & 1 & 9.17 & 49.3 &  \\
\midrule
Qwen2.5 0.5B & 0 & 12.2 & 47.9 & NS \\
             & 1 & 7.97 & 54.9 &  \\
\bottomrule
\end{tabular}
\caption{Performance comparison of different models with and without the \texttt{<|endoftext|>} token in the partial set of 142 samples from LAMBADA. Significance levels: $^{*} p < 0.05$, $^{**} p < 0.01$, $^{***} p < 0.001$, NS = Not Significant.}
\label{tab:lambada_results}
\end{table}

The results indicate that the inclusion of the \texttt{<|endoftext|>} token improves the performance of RWKV-7 very significantly, especially for small models (e.g., RWKV-7 0.1B). This finding highlights the importance of proper state initialization and context setting for RNN-based architectures like RWKV. Unexpectedly, we also found that two consecutive \texttt{<|endoftext|>} tokens at the beginning can further improve the performance of RWKV-7, despite that \texttt{<|endoftext|>} never appears consecutively in the training corpus.

However, for Transformer-based models such as Qwen2.5-0.5B, we observe that the impact of the \texttt{<|endoftext|>} token is less pronounced, suggesting that these models may have better mechanisms for attending to the initial token.

As a result, for optimal performance when prompting RWKV-7, we recommend always including the \texttt{<|endoftext|>} token at the beginning of the prompt. For example, if you plan to use RWKV-7 as a chat assistant, consider the following structured prompt format:

\begin{tcolorbox}[]
\begin{verbatim}
<|endoftext|>User: <Your Question>

Assistant: <Assistant Answer>

User: <Another Question>

Assistant:
\end{verbatim}
\end{tcolorbox}

\end{document}